%% file: siam_main.tex
\documentclass[hidelinks,onefignum,onetabnum]{siamart250211}
\usepackage{amsmath,amssymb,amsfonts,mathrsfs, mathtools, bm, bbm, dsfont, epstopdf,algorithm,algorithmic}
\usepackage[algo2e,ruled]{algorithm2e}
\usepackage[numbers]{natbib}
\usepackage{hyperref}
\usepackage[inline]{enumitem}
\usepackage{microtype}
\usepackage{graphicx,caption,subfigure}
\usepackage{csvsimple,booktabs}
\usepackage{textcomp}
\usepackage{xfrac}
\usepackage{adjustbox}
\usepackage{lmodern}
\usepackage{lipsum}
\usepackage{epstopdf}
\ifpdf
  \DeclareGraphicsExtensions{.eps,.pdf,.png,.jpg}
\else
  \DeclareGraphicsExtensions{.eps}
\fi

\input{header}
\renewcommand{\hat}{\widehat}
\renewcommand{\tilde}{\widetilde}
\newcommand{\obar}{\overline}

\newcommand{\id}{{\textrm{Id}}}
\newcommand{\emd}{{I_\kappa}}
\newcommand{\kov}{{C_{XX}}}
\newcommand{\CMEO}{U_\star}

\newcommand{\Tr}{{\textrm{Tr}}}

\newcommand{\mix}{\text{mix}}

\newcommand{\bcmp}{b_\text{cmp}}
\newcommand{\Bsrc}{B_{\text{src}}}
\allowdisplaybreaks
\newcommand{\Lsq}{{L_2\left(\rho_X\right)}}
\newcommand{\rev}[1]{{\textcolor{black}{{#1}}}}

\newsiamremark{remark}{Remark}
\newsiamremark{hypothesis}{Hypothesis}
\crefname{hypothesis}{Hypothesis}{Hypotheses}
\newsiamthm{claim}{Claim}
\newsiamremark{fact}{Fact}
\crefname{fact}{Fact}{Facts}

\headers{Nonparametric Sparse Online Learning of the Koopman Operator}{B. Hou, S. Sanjari, N. Dahlin, A. Koppel and B. Subhonmesh}

\title{Nonparametric Sparse Online Learning of the Koopman Operator}

\author{Boya Hou\thanks{Coordinated Science Laboratory, University of Illinois Urbana-Champaign  (\email{boyahou2@illinois.edu}).}
\and Sina Sanjari\thanks{Department of Mathematics and Computer Science, Royal Military College of Canada 
  (\email{sanjari@rmc.ca}).}
\and Nathan Dahlin\thanks{Department of Electrical and Computer Engineering, University at Albany, SUNY
(\email{ndahlin@albany.edu}).}
\and Alec Koppel\thanks{Applied Physics Labs, Johns Hopkins University
(\email{alec.koppel@jhuapl.edu}).}
\and Subhonmesh Bose\thanks{Department of Electrical and Computer Engineering, Coordinated Science Laboratory, University of Illinois Urbana-Champaign  (\email{boses@illinois.edu}).}  }

\usepackage{amsopn}

\ifpdf
\hypersetup{
  pdftitle={Nonparametric Sparse Online Learning of the Koopman Operator},
  pdfauthor={B. Hou, S. Sanjari, N. Dahlin, A. Koppel and B. Subhonmesh}
}
\fi

\begin{document}

\maketitle

\begin{abstract}
The Koopman operator provides a powerful framework for representing the dynamics of general nonlinear dynamical systems. However, existing data-driven approaches to learning the Koopman operator rely on batch data.
In this work, we present a \emph{sparse online} learning algorithm that learns the Koopman operator \emph{iteratively} via stochastic approximation, with explicit control over model complexity and provable convergence guarantees. Specifically, we study the Koopman operator via its action on the reproducing kernel Hilbert space (RKHS), and address the mis-specified scenario where the dynamics may escape the chosen RKHS. In the mis-specified settings, we relate the Koopman operator to the conditional mean embeddings (CME) operator. 
We further establish both asymptotic and finite-time convergence guarantees for our learning algorithm in mis-specified settings, with trajectory-based sampling where the data arrive sequentially over time.
Numerical experiments demonstrate the algorithm's capability to learn unknown nonlinear dynamics.
\end{abstract}

\begin{keywords}
Nonlinear dynamical system,
Koopman operator,
Reproducing kernel Hilbert space,
Conditional mean embedding,
Stochastic approximation%
\end{keywords}


\input{1-introduction}

\input{2-rkhs}

\input{3-sysid}

\input{4-algorithm}
\input{5-assumptions}

\input{5-asymptotic}

\input{5-FT}

\input{6-1}

\input{6-2}

\section{Discussions and Conclusions}
\label{sec.conclusion}

In this paper, we present an online learning algorithm that learns a sparse Koopman operator in RKHS with sampling from trajectories. Our method does not require the RKHS to be closed under the Markov kernel. We established the asymptotic and finite-time convergence guarantee of the online sparse SOGD algorithm. The computational framework was illustrated on a simple dynamical system example and model-based RL.

We conclude the paper with a few important remarks. First, discretizations of continuous-time nonlinear dynamical systems naturally lend themselves to Markov models. Nonlinear systems theory provides a gamut of tools to analyze the qualitative behavior of such systems. While the operator-theoretic viewpoint allows analysis of such behavior through the spectrum of the operator, the ability to do so with data in a computational framework vitally depends on the choice of the function space. Our work is focused on an RKHS; the interdependency of the choice of the function space and the qualitative behaviors that can be studied remains an interesting direction for research. Second, notice that our mis-specified setting is based on the assumption that $K$ maps functions from $\Hcal$ to another that is equivalent to a function in an intermediate space $[\Hcal]^\beta$ in the $L_2(\rho_X)$ sense. For trajectory-based sampling, the unique invariant measure provided an easy candidate for $\rho_X$. Our numerical example suggests that there may be other viable candidates for analysis even when such unique invariant measures do not exist. Moreover, our analysis builds upon intermediate spaces defined via $\rho_X$. One can view $\rho_X$ as a base measure whose interaction with the kernel determines what properties of the operator we can learn. The mis-specified setting has the downside that it limits the ability to meaningfully perform multi-step predictions. While one can always build a predictor using $U_\lambda$, the bias from lack of regularity will preclude meaningful bounds on the errors. Understanding such effects more fully defines an interesting research question. Finally, understanding the tightness of our derived bounds and a fair comparison with the sample-average approximation under similar settings will provide insights into the pros and cons of data collection methods in data-driven model learning.

\appendix
\crefalias{section}{appendix}
\input{appdx_HS}

\input{proof/interpolation}

\input{proof/thm2}
\input{CME_pf}
\input{appdx_CMP}
\input{5-proof}

\input{proof/lemma5-1}
\input{proof/lemma5-3}

\input{proof/lemma5-4}

\input{CME_FT_pf}

\input{proof/thm10}

\input{proof/corollary11}
\input{proof/corollary12}\\

\bibliographystyle{siamplain}
\bibliography{ref-rkhs,ref-NLDynm,ref-learning,ref-operator}

\end{document}

%% file: header.tex
\newcommand{\TV}{{\textrm{TV}}}

\newcommand{\HS}{{\textrm{HS}}}
\newcommand{\op}{{\textrm{op}}}

\newcommand{\beq}{\begin{equation}}
\newcommand{\eeq}{\end{equation}}
\newcommand{\beqa}{\begin{eqnarray}}
\newcommand{\eeqa}{\end{eqnarray}}
\newcommand{\beqan}{\begin{eqnarray*}}
\newcommand{\eeqan}{\end{eqnarray*}}

\newcommand{\vnorm}[1]{\left\|#1\right\|}

\newcommand{\tmin}{\text{min}}

\newcommand{\dd}{\mathrm{d}}

\newcommand{\E}{\mathds{E} }

\newcommand{\prob}{\mathbb{P}}

\newcommand{\Aset}{\mathbb{A}}

\newcommand{\Iset}{\mathbb{I}}
\newcommand{\Jset}{\mathbb{J}}

\newcommand{\Nset}{\mathbb{N}}

\newcommand{\Rset}{\mathbb{R}}

\newcommand{\Tset}{\mathbb{T}}
\newcommand{\Uset}{\mathbb{U}}

\newcommand{\Xset}{\mathbb{X}}
\newcommand{\Yset}{\mathbb{Y}}
\newcommand{\Zset}{\mathbb{Z}}

\newcommand{\Bcal}{{\cal B}}

\newcommand{\Dcal}{{\cal D}}

\newcommand{\Fcal}{{\cal F}}

\newcommand{\Hcal}{{\cal H}}
\newcommand{\Ical}{{\cal I}}

\newcommand{\Lcal}{{\cal L}}

\newcommand{\argmin}{\mathop{\rm argmin}}

\newcommand{\bone}{\mathds{1}}

\newcommand{\ve}{\varepsilon}

\newcounter{l1}
\newcounter{l2}
\newcounter{l3}
\newcommand{\bdotlist}{\begin{list}{$\bullet$}{}}
\newcommand{\bboxlist}{\begin{list}{$\Box$}{}}
\newcommand{\bbboxlist}{\begin{list}{\raisebox{.005in}{{\tiny
$\blacksquare$ \ \ }}}{}}
\newcommand{\bdashlist}{\begin{list}{$-$}{} }
\newcommand{\blist}{\begin{list}{}{} }
\newcommand{\barablist}{\begin{list}{\arabic{l1}}{\usecounter{l1}}}
\newcommand{\balphlist}{\begin{list}{(\alph{l2})}{\usecounter{l2}}}
\newcommand{\bAlphlist}{\begin{list}{\Alph{l2}.}{\usecounter{l2}}}
\newcommand{\bdiamlist}{\begin{list}{$\diamond$}{}}
\newcommand{\bromalist}{\begin{list}{(\roman{l3})}{\usecounter{l3}}}

\newtheorem{assumption}{Assumption}

%% file: 1-introduction.tex
\section{Introduction}
\label{sec.intro}
Poincar{\'e}'s geometric state-space approach in \cite{poincare1899methodes} analyzes dynamical systems by studying the evolution of system states over time. 
Koopman operator theory, with its origins in \cite{koopman1932dynamical}, provides an alternate way to characterize nonlinear systems through the lens of linear operators by studying how the system evolves functions of states through time. 
For a discrete-time deterministic dynamical system on finite-dimensional state space $\Xset \subseteq \Rset^n$ described by $x_{t+1} = T(x_t)$ for $t \in \Nset$ where $T:\Xset \to \Xset$, the Koopman operator is defined via composition on function $g: \Xset \to \Rset$ as
\begin{align}
    K g(x_t) = \left(g \circ T \right) (x_t) = g \left( T \left(x_t\right)\right) = g \left(x_{t+1}\right),\quad t  \in \Nset.
\end{align}
For a discrete-time time-homogeneous Markov process described by the transition kernel $P$, the (stochastic) Koopman operator \cite{lasota2013chaos,mezic2005spectral} generalizes the above to
\begin{align}
 \left(K g\right)(x_t)=\int P(\dd x_{t+1} | x_t) g(x_{t+1}) \dd x_{t+1},\quad t  \in \Nset. 
\label{eq.def.K.Markov}
\end{align}

The Koopman operator lifts the nonlinear dynamical system description over a finite-dimensional state space to a linear but infinite-dimensional operator description over a space of functions. As a linear operator, its spectra contain valuable information for understanding system dynamics, such as the state space geometry  \cite{mezic2005spectral,mezic2020spectrum,mezic2021koopman}.
Numerical methods such as the dynamic mode decomposition (DMD) \cite{schmid2010dynamic,rowley2009spectral}, and its variants, e.g., \cite{jovanovic2014sparsity,williams2015data,klus2020eigendecompositions}, can approximate the Koopman operator and its spectra from empirical data. As a result, this operator has come to define the gateway for data-driven analysis of nonlinear dynamical systems with unknown models, e.g., see  \cite{budivsic2012applied,brunton2016discovering,otto2021koopman,hou2024propagating,hou2024TSK,bold2024kernel,zagabe2025uniform,rosenfeld2024occupation}. 
In this paper, we aim to learn the Koopman operator \emph{iteratively} with streaming data collected from trajectories.

The Koopman operator is studied through its interaction with a function space, and the choice of that space dictates how well the system dynamics encoded in the operator can be analyzed.
Of the existing \emph{parametric} techniques that learn the Koopman operator, extended dynamic mode decomposition (EDMD) \cite{williams2015data} is perhaps the most widely used, where the function space is the finite-dimensional span of a pre-selected basis of functions.
If this subspace is not rich enough to capture the system dynamics, the learned operator fails to capture crucial properties of the dynamical system. Given the difficulty of selecting a set of basis functions, we study a nonparametric approximation method that aims to learn the Koopman operator through its interaction with a reproducing kernel Hilbert space (RKHS), along the lines of \cite{williams2014kernel,kawahara2016dynamic,klus2020eigendecompositions,hou2023sparse,kostic2022learning,bevanda2023koopman,kohne2025error}. Such a nonparametric computational framework automatically produces a set of basis functions from data, thus avoiding subscriptions to specific parametric choices a priori. While it is natural to consider the setting in which the function space is closed under the action of the system dynamics, it is well known that such a closedness assumption is restrictive and challenging to verify  \cite{mezic2020spectrum,colbrook2024limits,kohne2025error}. In our analysis, we allow for this ``mis-specification'' in the operator learning setting, where the Koopman operator $K$ maps a function in an RKHS to some intermediate space between the RKHS and square-integrable functions (see Section \ref{sec.RKHS}), thus relaxing the closedness assumption. Specifically, we characterize how fast the Koopman operator can be approximated in an online fashion in this mis-specified setting with trajectory-based sampling, where the data becomes available sequentially.

First presented in \cite{song2009hilbert},  the conditional mean embedding (CME) operator in its action over an RKHS encodes how the distribution over one random variable relates to another (see Section \ref{sec.embedding}). If the random variables correspond to successive states of a discrete-time Markov process, CMEs naturally encapsulate the transition dynamics without resorting to explicit modeling of system dynamics such as those via difference equations. 
If the RKHS is closed under the action of the Koopman operator, then the CME and the Koopman operators define adjoints of each other, and have been studied in \cite{klus2020eigendecompositions,mollenhauer2020nonparametric}. Without closedness, a rigorous characterization of the relationship between these operators is missing in the literature--a gap we bridge in Section \ref{sec.Koopman}, utilizing the framework of \cite{park2020measure} that views CME as a vector-valued Bochner-integrable random variable and subsequently leveraged in \cite{li2022optimal,kostic2023sharp} to derive learning rates.

Learning rates for Koopman/CME operators that interact with RKHS have been extensively studied, e.g., see \cite{klus2015towards, klus2020eigendecompositions, klus2020kernel,hou2023sparse,li2022optimal,kostic2023sharp,kohne2025error}. Such sample complexity analysis relies on writing the Koopman/CME operators in terms of the so-called covariance operators, which are expectations of certain rank-one tensors (see Section \ref{sec.embedding}) and then replacing said expectations with empirical means. As a result, these prior works approach the operator learning problem via \emph{sample average approximation} that requires one to collect and process independent and identically distributed samples from interactions with a dynamical system or a Markov process in a batched fashion. In many settings, one can only access a \emph{single trajectory} of a sequence of data points from a Markov process. Neither are these samples independent, nor identically distributed. In this paper, we develop a \emph{stochastic approximation}-based algorithm that processes \emph{streaming data} from a Markov process to iteratively update an estimate of the Koopman/CME operators. Our algorithm in Section \ref{sec.algorithm} essentially performs a stochastic operator gradient descent (SOGD) with consecutive samples from a single trajectory observed from the process. We employ SOGD with Markovian data to solve a regression problem in linear operators to maintain said estimate, leveraging the framework in \cite{grunewalder2012conditional,li2022optimal}.

The description of the learned nonparametric Koopman/CME estimate grows with the length of the trajectory. To combat the growth in representation, we develop an \emph{online sparsification} to SOGD along the lines of kernel matching pursuit in \cite{vincent2002kernel,koppel2017parsimonious}, but for operator learning. 
Our theoretical analysis in Section  \ref{sec.5.analysis} generalizes the rich literature on stochastic approximation in Euclidean spaces such as in \cite{borkar2008stochastic,srikant2019finite,chen2022finite} to sparse SOGD in the space of Hilbert-Schmidt operators over an RKHS to provide both asymptotic and finite-time convergence guarantees for online estimation of Koopman/CME operators. Such an analysis stands in stark contrast to \cite{zhang2019online,giannakis2023learning}, which have proposed online versions of DMD and kernel-DMD without developing sample complexity guarantees. Our analysis handles the dependencies across samples generated from a trajectory, as well as the compounding bias due to sparsification to provide last-iterate approximation guarantees on the learned operator. In Section \ref{sec.example}, we visualize the dominant eigenfunction of the learned Koopman operator for a nonlinear dynamical system to illustrate the efficacy of our algorithm to settings where our assumptions for the theoretical analysis may not apply. Due to space limitations, some proofs and additional experimental results are omitted. A complete version, including these details, can be found in \cite{hou2024nonparametric}.

\subsection{Our Contributions}
\begin{itemize}[leftmargin=*]
    \item For a discrete-time Markov process, we relate the Koopman operator acting on an RKHS with a CME operator for the transition kernel in the mis-specified setting, where the RKHS may not be closed under push-forward of the Markov kernel. 
    \item We propose an \emph{online} learning algorithm based on sparse SOGD to estimate the Koopman/CME operator with data collected from a single trajectory from the Markov process. 
    \item We present both almost-sure asymptotic and finite-time convergence guarantees in the mean-square sense for the Koopman/CME estimate from our sparse SOGD algorithm.
\end{itemize}

%% file: 2-rkhs.tex
\section{RKHS Preliminaries}
\label{sec.RKHS}

We start by describing a set of square-integrable functions. Let $(\Xset, \Bcal_X, \rho_X)$ be a probability space, where $\Xset$ is a subset of an Euclidean space,  $\Bcal_X$ is the Borel $\sigma$-algebra on $\Xset$, and $\rho_X$ a probability measure on $\Xset$. 
Denote by $\Lcal_2(\rho_X)$, the vector space of real-valued square-integrable functions with respect to $\rho_X$ with the norm $\vnorm{\cdot}_{\rho}$ that satisfies $\vnorm{f}_{\rho}^2:= \int_{\Xset}\left|f\left(x\right)\right|^2 \dd \rho_X $ for any $f\in \Lcal_2(\rho_X)$. For any $f \in \Lcal_2(\rho_X)$, its $\rho_X$-equivalent class $[f]$ comprises all functions $g \in \Lcal_2(\rho_X)$ that $\rho_X\left( \left\{f \neq g \right\}  \right) = 0$.
Let $\Lsq:= \Lcal_2(\rho_X)_{/ \sim} $ be the corresponding quotient space equipped with the norm $\vnorm{\left[f\right]}_{\Lsq}  = \vnorm{f}_{\rho}$ for any $f\in \Lcal_2(\rho_X)$.

Next, we describe a real-valued RKHS; see \cite{berlinet2011reproducing} for a more comprehensive review. 
A Hilbert space of real-valued functions on $\Xset$ with its inner product $\left({\Hcal}, \left \langle \cdot, \cdot \right \rangle_{{\Hcal}} \right)$ is an RKHS if the evaluation functional defined by $\delta_x f = f(x)$ is bounded (continuous) for all $x \in \Xset, f\in {\Hcal}$.
The Riesz representation theorem implies that for all $f \in {\Hcal}$, there exists an element $\phi(x) \in {\Hcal}$ such that
$ \delta_x f = \left \langle f, \phi(x)  \right \rangle_{{\Hcal}}$.
Define $\kappa_{X}: \Xset \times \Xset \rightarrow \mathbb{R}$ by $ \kappa_{X}(x,x') := \left \langle \phi \left(x\right), \phi \left(x'\right)\right \rangle$. Then, $\kappa_{X}$ is a positive definite kernel that satisfies 
$\kappa_X(\cdot,x) \in {\Hcal}$, and
$ \left \langle f, \kappa_X(\cdot,x) \right \rangle_{{\Hcal}} = f(x)$, $\forall x \in \Xset$,  $\forall \rev{f} \in {\Hcal}$, where
 $\kappa_X$ is called a reproducing kernel and $\phi(x):= \kappa_X(\cdot,x)$ is the canonical feature map.
We make the following assumption about $\kappa_X$ and relate it to $L_2(\rho_X)$.
\begin{assumption} $\kappa_X$ is $\rho_X$-measurable and  $\sup\limits_{x\in \Xset} 
{\kappa_X(x,x)} \leq 
{B_\infty}$ for some $0<{B_\infty} < \infty$.
\label{assumption.polk}
\end{assumption}

From \citep[Lemmas 2.2 and 2.3]{steinwart2012mercer}), the uniform boundedness of the  kernel guarantees that ${\Hcal}$ admits a compact embedding within $\Lsq$, We denote its image by $\left[{\Hcal}\right]:=\left\{ \left[f\right]: f \in {\Hcal}\right\}$. Define the integral operator ${T}: \Lsq \to \Lsq$ as
\begin{align}
    {T} \left[f\right]  := \left[\int_{\Xset} \kappa_X(\cdot,x) g(x) \dd \rho_X (x)\right], \quad \forall g\in \left[f\right],
\label{eq.L_kappa}
\end{align}
for any $\left[f\right]\in \Lsq$. Under Assumption \ref{assumption.polk}, ${T}$ is continuous, self-adjoint, positive trace-class, and compact. The spectral theorem for self-adjoint compact operators \cite[Theorem V.2.10]{kato2013perturbation} indicates that there exists a countable index set $\Iset$, a non-increasing, summable sequence $(\sigma_i)_{i \in \Iset} \in (0,\infty)$ converging to $0$ and a family $(e_i)_{i \in \Iset} \subset {\Hcal}$ such that $\left(\left[e_i\right] \right)_{i \in \Iset}\subset \Lsq$ is an orthonormal system (ONS) of $\Lsq$, and ${T}$ admits the decomposition,
\begin{align}
    {T} [f] = \sum_{i \in \Iset} \sigma_i \left \langle [f] , \left[e_i\right] \right \rangle_{\Lsq} \left[e_i\right], \quad
    [f] \in \Lsq.
    \label{eq.L_kappa_spectral}
\end{align}
Thus, $\left(\sigma_i \right)_{i \in \Iset}$ defines the family of non-zero eigenvalues of ${T}$ and $\left(\left[e_i\right] \right)_{i \in \Iset}$ are the corresponding eigenvectors of ${T}$. It is not difficult to show that the embedding of ${\Hcal}$ within $\Lsq$ is then given by,
\begin{align}
 [{\Hcal}] = \left\{ \sum_{i\in\Iset} b_i [e_i] : \sum_{i \in \Iset} \frac{b_i^2}{\sigma_i} < \infty \right\}.   
\end{align}

Using these eigenpairs, following \cite{steinwart2012mercer}, define the inner product space for $\beta \in (0,1)$,
\begin{align}
\left[\Hcal\right]^\beta
:=  
\left\{ \sum_{i \in \Iset} b_i  \left[e_i\right]: \sum_{i \in \Iset} \frac{b_i^2}{\sigma_i^\beta} < \infty \right\}, 
\; 
\left \langle \sum_{i \in \Iset} b_i \left[e_i\right],
\sum_{i \in \Iset} b'_i \left[e_i\right]
\right \rangle_{\left[\Hcal\right]^\beta} 
:= \sum_{i \in \Iset} \sigma_i^{-\beta} {b_i b'_i}
\label{def.power}
\end{align}

\begin{align}
    \left[\Hcal\right]^1 = [{\Hcal}] \subseteq \left[\Hcal\right]^\beta \subseteq \left[\Hcal\right]^{\beta'} \subseteq \left[\Hcal\right]^0 \subseteq \Lsq, \quad 0<\beta' <\beta <1.
\end{align}

Generally, $\Lsq$ is larger than $[\Hcal]^0$. In fact, $\Lsq = \ker T \oplus \left[\Hcal\right]^0$, where $\oplus$ stands for direct sum.
In this paper, ${\Hcal}$, $\Lsq$, and the intermediate spaces $\left[\Hcal\right]^\beta$'s, play important roles in defining the Koopman operator.

Consider another separable real-valued Hilbert space $\Hcal_Y$ on a subset $\Yset$ of Euclidean spaces. A bounded linear operator $A$ from ${\Hcal}$ to $\Hcal_Y$ is Hilbert-Schmidt (HS), if 
the Hilbert-Schmidt norm of $A$, denoted by $\vnorm{A}_{\HS}$ is finite, where $\vnorm{A}_{\HS}^2 = \sum_{i \in \Nset} \vnorm{A e_i}_{\Hcal_Y}^2$ and $\{e_i\}_{i \in \Nset}$ defines an orthonormal basis (ONB) of ${\Hcal}$. The summation can be shown to be independent of the choice of the ONB. 
We denote $\HS({\Hcal},\Hcal_Y)$ as the Hilbert space of HS operators from ${\Hcal}$ to $\Hcal_Y$, endowed with the norm $\vnorm{\cdot}_\HS$. See the extended version of this paper \cite[Appendix A ]{hou2024nonparametric} for an introduction to  HS operators. 
For $f_1 \in {\Hcal}$ and $f_2 \in \Hcal_Y$, the tensor product $f_1 \otimes f_2$ is defined as a rank-one operator from $\Hcal_Y$ to ${\Hcal}$ via 
$\left(f_1 \otimes f_2 \right)g \mapsto \left \langle g,f_2\right\rangle_{\Hcal_Y}f_1$ for all $g \in \Hcal_Y$.
This rank-one operator is HS. Given a second operator $f'_1 \otimes f'_2$ for $f'_1 \in {\Hcal}$, $f'_2 \in \Hcal_Y$, their inner product is 
$\left \langle f_1 \otimes f_2,f'_1 \otimes f'_2\right\rangle_{\HS}   
= \left \langle f_1,f'_1\right\rangle_{{\Hcal}}
\left \langle f_2,f'_2\right\rangle_{\Hcal_Y}$.
Denote by ${\Hcal} \otimes \Hcal_Y$, the tensor product of two Hilbert spaces ${\Hcal}$ and $\Hcal_Y$ which is the completion of the algebraic tensor product with respect to the norm induced by the aforementioned inner product. 
Moreover, $\HS({\Hcal},\Hcal_Y)$ is isometrically isomorphic to $\Hcal_Y \otimes {\Hcal}$, denoted $\HS({\Hcal},\Hcal_Y) \cong \Hcal_Y \otimes {\Hcal}$, per \cite[Lemma C.1]{park2020measure}.

Denote $L_2(\rho_X,\Hcal_Y)$ as the \rev{equivalent classes} of $\Hcal_Y$-valued Bochner square-integrable functions $y: x \mapsto y(x)$ with values in $\Hcal_Y$ such that
$\vnorm{y}^2_{L_2(\rho_X,\Hcal_Y)}= \left(\int_\Xset \vnorm{y\left(x\right)}_{\Hcal_Y}^2 \dd \rho_X \right) < \infty$.
An $\Hcal_Y$-valued Hilbert space $\left(\Hcal_V, \left \langle \cdot, \cdot \right \rangle_{\Hcal_V} \right)$ of functions $v: \Xset \to \Hcal_Y $ is an $\Hcal_Y$-valued RKHS if for each $x \in \Xset$, $y \in \Hcal_Y$, the linear functional $v \mapsto \left \langle y, v \left(x\right) \right \rangle_{\Hcal_Y}$ is bounded. 
$\Hcal_V$ admits an operator-valued reproducing kernel of positive type $\Gamma:\Xset\times \Xset \to \Lcal(\Hcal_Y)$ which satisfies $ \left\langle v\left(x\right) , y \right\rangle_{\Hcal_Y} = \left\langle v , \Gamma(\cdot,x)y \right\rangle_{\Hcal_V}$ and $ \left\langle y,\Gamma(x,x') y'\right\rangle_{\Hcal_Y} = \left\langle \Gamma(\cdot,x)y,\Gamma(\cdot,x')y'  \right\rangle_{\Hcal_V}$ for all $x,x'\in \Xset$, $y,y' \in \Hcal_Y$ and $v \in \Hcal_V$. Here,  $\Lcal(\Hcal_Y)$ denotes the Banach space of bounded linear operators from $\Hcal_Y$ to itself. Restrict attention to the vector-valued RKHS associated with the operator-valued kernel $\kappa_X \left(x,x'\right) \id_{Y}$ where $\id_{Y}$ is the identity map on $\Hcal_Y$ and denote it by $\Hcal_V$. 

\begin{lemma} 
Let $\rho$ be a joint probability measure over $\Xset \times \Yset$ with marginals $\rho_X$ and $\rho_Y$, respectively. Suppose $\kappa_X$ and $\kappa_Y$ are positive definite, $\rho_X$ and $\rho_Y$ measurable kernels over $\Xset$ and $\Yset$, respectively, that admit a uniform upper bound. 
Then,
$ \Hcal_V \cong \Hcal_Y \otimes {\Hcal}$ and $ L_2(\rho_X,\Hcal_Y) \cong \Hcal_Y \otimes \Lsq$, where $\Hcal_V$ is the vector-valued RKHS induced by $\kappa_X \left(x,x'\right) \id_{Y}$. In addition, $\Hcal_V$ can be compactly embedded in $L_2(\rho_X,\Hcal_Y)$.
\label{lemma.isomorphism}
\end{lemma}

We do not formally prove this result, but make two remarks. The first isomorphism \rev{between $\Hcal_Y \otimes {\Hcal} \to \Hcal_V$, denoted as}
$\iota_\kappa : \Hcal_Y \otimes {\Hcal} \to \Hcal_V$,
relies on \cite[Lemma 15]{ciliberto2016consistent} and \cite[Theorem 1]{li2022optimal}. The second claim is a direct consequence of \cite[Theorem 12.6.1]{aubin2011applied}, where the isometric isomorphism $\iota:\Hcal_Y \otimes \Lsq \to L_2(\rho_X,\Hcal_Y)$ is   realized by 
\begin{align}
    \iota(f \otimes g) = \left( x \mapsto f g\left(x\right)\right), \quad f \in \Hcal_Y, \quad g \in \Lsq.
\label{def.iota}
\end{align}
The embedding of $\Hcal_V$ within $L_2(\rho_X, \Hcal_Y)$ is compact if and only if $\Hcal$ can be compactly embedded within $L_2(\rho)$, which in turn is true, given the uniform boundedness of $\kappa_X$.
The authors of \cite{li2022optimal} establish that for each $v \in \Hcal_V$, there exists a unique $V\in \Hcal_Y \otimes {\Hcal}$ given by $V= \iota_\kappa^{-1} (v)$ such that $\vnorm{v}_{\Hcal_V}=\vnorm{V}_\HS$, and 
$v(x) = V \phi_X(x) \in \Hcal_Y$, $\forall x \in \Xset$.

Given the well-established isomorphism between the space of HS operators and tensor product spaces from \cite[Lemma C.1]{park2020measure}, the last result indicates that $\Hcal_V \cong \HS(\Hcal, \Hcal_Y)$ can be compactly embedded within $L_2(\rho_X, \Hcal_Y) \cong \HS(L_2(\rho_X), \Hcal_Y)$. One can define intermediate spaces of vector-valued functions (or analogously linear operators) for $\beta \in (0,1)$ as
\begin{align}
\begin{aligned}
\left[H_V\right]^\beta:= 
\left\{\iota\left(U \right) : U \in \HS\left(\left[\Hcal\right]^\beta, \Hcal_Y  \right)\right\},
\end{aligned}
\label{def.vRKHS.beta}
\end{align}
equipped with the norm
$\vnorm{v}_\beta:=\vnorm{U}_{\HS\left(\left[\Hcal\right]^\beta, \Hcal_Y  \right)}$, per \cite[Definition 3]{li2022optimal}. Here, $\iota$ is the isomorphism between $\HS\left(\left[\Hcal\right]^\beta, \Hcal_Y  \right)$ and $\left[H_V\right]^\beta$ in Lemma \ref{lemma.isomorphism}.

\section{Embedding of Probability Distributions}
\label{sec.embedding}
Let $X$ be an $\Xset$-valued random variable with measure $\rho_X$. Consider an RKHS $\Hcal$ over $\Xset$ with kernel $\kappa_X$ that satisfies Assumption \ref{assumption.polk}. Then, the \emph{kernel mean embedding} (KME) of $\rho_X$ in ${\Hcal}$ is the Bochner integral
$\E\left[\kappa_X(X, \cdot)\right]$,
where the expectation is computed with respect to $\rho_X$.

Now, let $(X,Y)$ be an $\Xset \times \Yset$-valued random variable with joint measure $\rho$. Let $\Hcal_Y$ be an RKHS over $\Yset$ with kernel $\kappa_Y$. Together, let $\Hcal, \Hcal_Y$ satisfy the conditions of Lemma \ref{lemma.isomorphism}.
Then, $(X,Y)$ can be embedded into ${\Hcal} \otimes \Hcal_Y$, per \cite{berlinet2011reproducing}, as
\begin{align}
\begin{aligned}
C_{XY}:=  \E[\phi_X(X) \otimes \phi_Y(Y)],
\end{aligned}
\label{def:cross-covariance}
\end{align} 
where the expectation is computed with respect to $\rho$. We call $C_{XY}$ (uncentered) cross-covariance operator. 
Likewise, the (uncentered) covariance operator is defined as
$C_{XX}:=\E_{X}[\phi_X(X) \otimes \phi_X(X)]$,
which can be viewed as the embedding of the marginal distribution $\rho_X$ into ${\Hcal} \otimes {\Hcal}$. 

Next, we define the conditional mean embedding (CME) which captures the \emph{dependency} between random variables $X$ and $Y$. The conditional mean embedding (CME) of $Y$ given $X$ is defined as
\begin{align}
\mu_{Y|X} :=\E[\kappa_Y(\cdot,Y) | X], 
\label{eq:def.muYx}
\end{align} 
Thus, $\mu_{Y|X}$ is an $X$-measurable $\Hcal_Y$-valued random variable.
For all $f_Y \in \Hcal_Y$, 
\begin{align}
\E[f_Y(Y) | X] = \left \langle f_Y, \mu_{Y|X} \right\rangle_{\Hcal_Y}.
\label{eq.CME.innerproduct}
\end{align}
According to \cite{park2020measure}, we can write 
$\mu_{Y|X} = \mu(X)$,    
where  $\mu: \Xset \to \Hcal_Y$ is a $X$-measurable $\Hcal_Y$-valued deterministic function in $L_2(\rho_X,\Hcal_Y)$ that defines the unique minimizer of a least squares regression problem in $L_2(\rho_X,\Hcal_Y)$ as
\begin{align}
\mu:= \argmin_{u \in L_2(\rho_X,\Hcal_Y)}\int_{\Xset \times \Yset}\vnorm{u\left(x\right) - \phi_Y \left(y\right)}_{\Hcal_Y}^2 \dd \rho.
\label{eq.v-regression-relax}
\end{align}
In light of Lemma \ref{lemma.isomorphism}, there exists a unique $U = \iota ^{-1} \left(\mu\right) \in \HS(\Lsq, \Hcal_Y)$ that we henceforth call the CME operator.
As will become clear in the sequel, the fact that the CME operator is a solution to a regression problem will ultimately allow us to develop a stochastic approximation algorithm to solve for it and provide approximation guarantees. That algorithm will then also yield an approximation to the Koopman operator for a Markov process via its relation to the CME operator.

%% file: 3-sysid.tex
\section{Studying Koopman Operator via CME}
\label{sec.Koopman}
Consider a time-homogeneous Markov process on $(\Xset, \Bcal_X)$ that is defined by the kernel $P$. Precisely, $P(\cdot|x)$ is a probability measure on $(\Xset, \Bcal_X)$ for each $x \in \Xset$, and $P(\Aset|\cdot) \in \Bcal_X$ for every $\Aset \in \Bcal_X$. For a $\Bcal_X$-measurable function $g:\Xset \to \Rset$, the action of the Koopman operator $K$ associated with the Markov process can be defined as,
\begin{align}
    (K g) (x) = \int g(x^+) P(d x^+ | x).
\end{align}
\label{eq:K.def}

While the stochastic Koopman operator has its origins in \cite{lasota2013chaos,mezic2005spectral}, it has appeared in the literature as a Markov operator acting on observables, such as in \cite{meyn2012markov,nummelin2004general,kontoyiannis2012geometric}. The definition can also be adapted to continuous-time Markov processes through a Markov semigroup as in \cite{lasota2013chaos,hou2023sparse}.

In this paper, we are interested in the interaction of $K$ with an RKHS $\Hcal$ characterized by a kernel $\kappa_X$ that satisfies Assumption \ref{assumption.polk}, and in doing so, establish a connection between $K$ and the CME operator $U$. To that end, let $\rho_X$ again denote a measure on $(\Xset, \Bcal_X)$ and consider the joint measure $\rho$ over $\Xset \times \Xset$  as the push-forward of the marginal $\rho_X$ through a kernel $P'$, where
\begin{align}
{P'}(\Aset|x)
:= \int_{\Xset} \bone_{(x, x^+) \in \Aset} P(dx^+ | x) ,
\label{eq:P'.def}
\end{align}
for all sets $\Aset$ in the $\sigma$-algebra generated by sets in $\Bcal_X \times \Bcal_X$. The joint measure $\rho$ can be expressed using $\rho_X$ and $P'$ as
\begin{align}
\rho(\Aset) = \int_{\Xset} {P'}(\Aset|x)\,\rho_X(dx).
\end{align}

With this notation, let $(X, X^+)$ be an $\Xset \times \Xset$ random variable with joint measure $\rho$. Said plainly, if $X$ is sampled according to $\rho_X$, then $P$ pushes it forward to $X^+$. The CME of $X^+$ given $X$ can be written as $\mu_{X^+|X} = \E[\kappa_X(., X^+) | X]$ from \eqref{eq:def.muYx}, where we use $\Yset = \Xset$, $\kappa_Y = \kappa_X$ and $\Hcal_Y = \Hcal$ in the notation of Section \ref{sec.embedding}.
The relation in \eqref{eq:K.def} can then be written as
\begin{align}
\begin{aligned}
(K g) (X)
= \E\left[g(X^+) | X\right]
= \left\langle g, \mu_{X^+|X}\right\rangle,
\quad g \in \Hcal,
\end{aligned}
\label{eq.Koopman.CME}
\end{align}
via \eqref{eq.CME.innerproduct}.
Hence, $\mu_{X^+|X}$ defines the Riesz representation of the function evaluation of the Koopman operator restricted to $\Hcal$. This provides a connection between the CME and the Koopman operator. We formalize this relationship in terms of an adjoint and impose regularity conditions on the operators that will facilitate algorithm design to approximate the operator.

The relation in \eqref{eq:K.def} provides a point-wise definition of the action of $K$ on a function $g$. For all $g \in \Hcal$, if $K g \in \Hcal$, then $\Hcal$ is said to be \emph{invariant} under the action of $K$. In such cases, the link between $K$ and the CME operator has been studied by \cite{klus2020eigendecompositions,hou2023sparse}. In general, one cannot guarantee such invariance. \rev{To provide a simple example, consider the scalar discrete-time deterministic dynamical system over $\Rset$ described by $x^+ = x^2$ and the RKHS $\Hcal$ defined by the linear kernel, $\kappa_X(x,x') = x x'$. In this case, $\Hcal$ contains all linear functions. For $g(x) = x$, notice that $K g = x^2$, a quadratic that does not belong to the class of all linear functions of $x$.}

In general, closure under dynamics is a restrictive assumption and is difficult to certify, except in a few cases, e.g., see \cite{kohne2025error}. To tackle this challenge, we consider the mis-specified setting and assume that $K$ is an HS operator that maps a function in $\Hcal$ to a function whose equivalent class (in the $L_2(\rho_X)$ sense) is an intermediate space from \eqref{def.power} between $\left[\Hcal\right]$ and $\Lsq$. 
Let $[K]:\Hcal \to \Lcal_2(\rho_X)$ be the composition of $K$ with the canonical quotient map $f \mapsto [f]$ for $f \in \Lcal_2(\rho_X)$. For notational simplicity, when the intended codomain is clear, we use $K$ and $[K]$ interchangeably.
The following theorem formally establishes the connection of the Koopman operator and the CME in this setting; see Appendix \ref{pf.thm2} for its proof.

\begin{theorem}
\rev{Let $K \in \HS(\Hcal, \left[\Hcal\right]^\beta)$ for some $\beta\in(0,1)$. Then, $U = K^{\star} \in \HS(\left[\Hcal\right]^\beta,\Hcal)$ where $U = \iota^{-1}(\mu)$ is the CME operator.}
\label{thm.Koopma-CME}
\end{theorem}

We remark that this adjoint relationship is specific to the study of the action of $K$ on an RKHS $\Hcal$ that allows the definition of the CME operator, and that too in the mis-specified setting. Within other function spaces, the adjoint relationship of the Koopman operator and the so-called Perron-Frobenius (or forward Kolmogorov) operator has been known, e.g., see \cite{lasota2013chaos}.

%% file: 4-algorithm.tex
\section{Sparse Online Learning Algorithm}
\label{sec.algorithm}
Having established that $K$ is the adjoint of $U$ \rev{in Theorem \ref{thm.Koopma-CME}}, we next present an online algorithm to construct $K$ \emph{iteratively}. Our algorithm builds on stochastic operator gradient descent (SOGD) for $U$ to solve the regression problem in \eqref{eq.v-regression-relax}. Our framework defines a sharp deviation from prior art that uses sample average approximation, e.g., see \cite{li2022optimal,kostic2023sharp}.

\subsection{The Regularized CME Learning Problem}
\label{sec.3.1}
Consider again a time-homogeneous Markov process defined by the kernel $P$, such that $\rho_X$ is a measure on $\Xset$ and $\rho$ is a measure on $\Xset \times \Xset$ obtained by the push-forward of $\rho_X$ through $P'$.
Define the regularized variant of \eqref{eq.v-regression-relax} as
\begin{align}
\mu_{\lambda}:= \argmin_{u \in \Hcal_V}\frac{1}{2}\int_{\Xset \times \Xset}\vnorm{u \left(x\right) - \phi \left(x^+\right)}_{\Hcal}^2 \dd \rho(x,x^+)
+ \frac{\lambda}{2} \vnorm{u}_{\Hcal_V}^2,
\quad \lambda>0.
\label{eq.v-regression-lambda}
\end{align}
Again, with $\mu_\lambda \in \Hcal_V$, we associate a unique HS-operator $U_{\lambda}\in \HS(\Hcal,\Hcal)$ such that 
$
\mu_\lambda(x) = \iota_\kappa \left(U_\lambda \right)(x) = U_\lambda \phi_X(x)
$,
where $\iota_\kappa$ is the isometric isomorphism between $\HS(\Hcal,\Hcal)$ and $\Hcal_V$ defined in Lemma \ref{lemma.isomorphism}. We call $U_{\lambda}$ as the {regularized} CME operator that solves  $U_\lambda = \argmin_{U \in \HS(\Hcal,\Hcal)}R_\lambda(U)$, where the regularized risk $R_\lambda: \HS(\Hcal,\Hcal) \to \Rset$ is defined by
\begin{align}
R_\lambda(U):=
\frac{1}{2} \E \left[\vnorm{\phi(x^+) - U \phi(x)}_{\Hcal}^2 \right] + \frac{\lambda}{2} \vnorm{U}^2_{\HS(\Hcal,\Hcal)}.
\label{eq.regression}
\end{align}

Our next result establishes the existence and uniqueness of the minimizer, in effect making $U_\lambda$ well-defined. The proof is presented in Section \ref{lemma.CME.gradient.pf}. In the sequel, we use the notation $\vnorm{\cdot}_\HS$ to denote $\vnorm{\cdot}_{\HS(\Hcal, \Hcal)}$.

\begin{lemma} For $\lambda \geq 0$,
$R_\lambda$ is strong lower semi-continuous (l.s.c) and its gradient is given by $\nabla R_\lambda(U)= U C_{XX} -C_{X^+X} + \lambda U$ for any  $U \in \HS(\Hcal,\Hcal)$. 
When $\lambda>0$, $R_\lambda$ is strongly convex and $U_\lambda = C_{X^+X} (C_{XX} + \lambda \id )^{-1}$ is the unique minimizer of $R_\lambda$ over $\HS(\Hcal,\Hcal)$. 
\label{lemma.CME.gradient}
\end{lemma}

$U_\lambda$ is the regularized CME operator first proposed in \cite{song2009hilbert}. When $\lambda = 0$, $R_\lambda$ remains convex and the expression $\nabla R_\lambda(U)= U C_{XX} -C_{X^+X}$ continues to hold, but the existence of a solution requires extra conditions. Here, $C_{XX}$ and $C_{X^+ X}$ are covariance operators obtained using the definitions in Section \ref{sec.embedding}.

\subsection{Stochastic Operator Gradient Descent}
We now present an algorithm that suitably collects data from a time-homogeneous Markov process to iteratively maintain and update an estimate of $U_\lambda$. To understand the crux of the algorithm, notice that $R_\lambda$ is written in terms of $C_{X^+ X}$ and $C_{XX}$, both of which are expectations of rank-1 tensors as in Section \ref{sec.embedding}, where $(X, X^+)$ is sampled according to $\rho = \rho_X P'$.

Now, consider a collection of samples in $\left\{\left(x_i,x^{+}_{i}\right)\right\}_{i=0}^t$, where $x_i \in \Xset$ and $x_i^+ \sim P(\cdot|x_i) \in \Xset$. Notice that this collection leaves open the description of how $x_t$'s are chosen, but $x_t^+$ is always chosen by sampling the Markov kernel $P$ from $x_t$.
Given a sample $(x_t,x^{+}_t) \in \Xset \times \Xset$ for $t\in \Nset$, a sample-centered stochastic estimators of $C_{XX}$ and $C_{X^+ X}$ at $(x_t,x^{+}_t)$ are given by
\begin{align}
    \tilde{C}_{XX}(t)= \phi\left(x_t\right) \otimes \phi\left(x_t\right),
    \;
    \tilde{C}_{X^+X}(t)= \phi\left(x^{+}_t\right) \otimes \phi\left(x_t\right),
\end{align}
with which a sample-centered operator gradient is
\begin{align}
\tilde{\nabla} R_\lambda (x_t,x^{+}_t;U) =  U \tilde{C}_{XX}(t) - \tilde{C}_{X^+X}(t)+  \lambda U \in \HS(\Hcal,\Hcal),
\label{eq.SFG}
\end{align}
for all $U\in \HS(\Hcal,\Hcal)$ and  $t \in \Nset$.

Assume $U_0 = 0$.
For a step-size sequence $\left\{\eta_t\right\}_{t\in\Nset}$, consider $\left\{U_t\right\}_{t\in\Nset}$ taking values in $\HS(\Hcal,\Hcal)$ given by
\begin{align}
\begin{aligned}
U_{t+1} = U_t - \eta_t \tilde{\nabla} R_\lambda (x_t,x^{+}_t;U_t)
\end{aligned}
\label{eq.CME.SFGD}
\end{align}
for all $t \in \Nset$. This algorithm is \emph{online} in that it processes the samples $(x_t, x_t^+)$ one at a time and updates the estimate of $U_\lambda$ from $U_t$ to $U_{t+1}$.
We characterize the iterates of \eqref{eq.CME.SFGD} in terms of elements in $\Hcal \otimes \Hcal$. See Appendix \ref{lemma.CME.U.pf} for a proof. 
\begin{lemma}
    Let $\{U_t\}_{t\in \Nset}$ be the sequence generated by \eqref{eq.CME.SFGD}. Define matrices
    \begin{align}
     \Phi_{X,t}:=\left[\phi(x_1), \ldots, \phi(x_{t})\right],
     \quad
     \Psi_{X^+,t}:=\left[\phi\left(x^{+}_{1}\right), \ldots, \phi(x^{+}_{t})\right].
    \end{align}
    Then $\{U_t\}_{t\in \Nset}$ admits the representation,
    \begin{align}
    U_{t+1} 
        = \sum_{i=1}^t \sum_{j=1}^t W^{ij}_{t} \left(\phi(x^{+}_{i}) \otimes \phi(x_j)\right)
        = \Psi_{X^+,t} W_t \Phi_{X,t}^\top,\quad  \forall t\in \Nset,
    \label{eq.U.HS}
    \end{align}
    with the coefficient matrix $W_t$ given by
    \begin{gather}
    \begin{gathered}
        W^{ij}_{t} = (1 -\lambda \eta_t) 
        W^{ij}_{t-1}, \  1 \leq i,j \leq t-1;
        \\
        W^{it}_{t} =  -\eta_t \sum_{j=1}^{t-1} 
        W^{ij}_{t-1} \kappa_X(x_j,x_t), \ 1\leq i \leq t-1;
        \\
        W_t^{tj} = 0, \ 1 \leq j \leq t-1;
        \quad
        W^{tt}_{t} = \eta_t, \ t \in \Nset 
        \setminus \{0\};
        \quad W_0 = 0.
        \end{gathered}
    \label{eq.U.w}
    \end{gather}
\label{lemma.CME.U}
\vspace{-0.15in}
\end{lemma}

The above result states that the estimates of $U_\lambda$ generated by SOGD can be described by a linear combination of rank-1 tensors with kernel functions centered at the samples observed up to that time. 


\subsection{Online Sparsification}

Our last result indicates that the estimate $U_{t}$ of $U_\lambda$ from SOGD has a description that grows with $t$.
Next, we aim to control the growth of that description by judiciously admitting a new sample only when the new sample brings sufficiently ``new'' information, leading to a \emph{sparse} SOGD algorithm.

\begin{figure}[ht]
    \centering
    \includegraphics[width=0.6\linewidth]{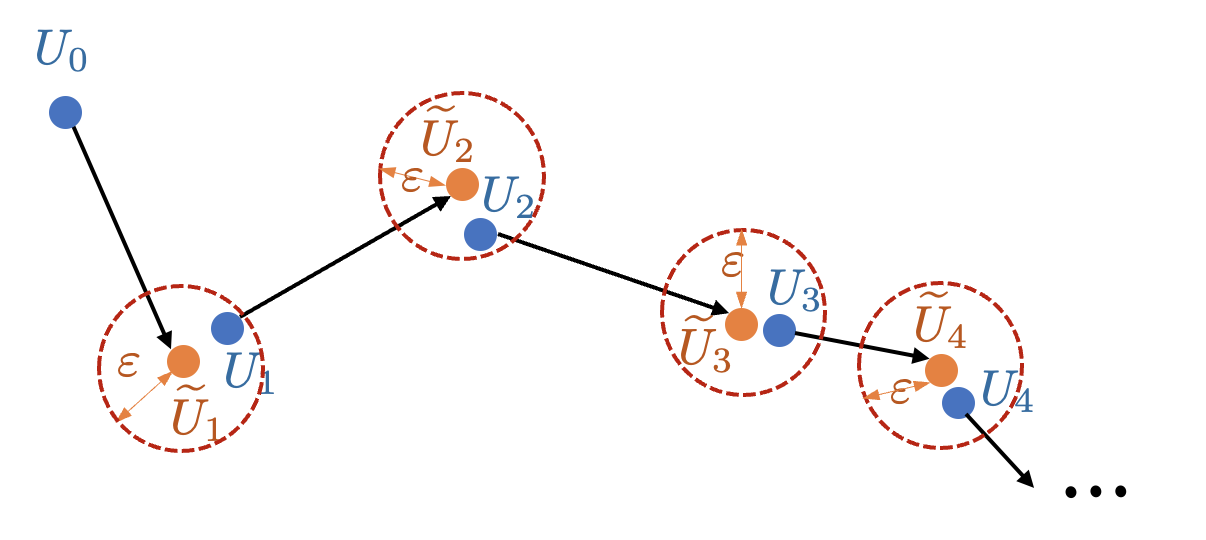}
    \caption{An illustration of the sparse SOGD algorithm: 
    $\left\{U_t\right\}_{t\in\Nset}$ (blue) are the iterates generated by sparse SOGD, and $\{\tilde{U}_t\}_{t\in\Nset}$ (orange) is the auxiliary sequence computed based on $\tilde{\Dcal}_t$ via basic SOGD \eqref{eq.CME.SFGD}. 
    Condition \eqref{eq.CME.stop} ensures that at each step $t\in\Nset$, the sparse estimate $U_t$ lies within the $\ve$-ball around $\tilde{U}_t$.
    }
    \label{fig1}
\end{figure}

Let $U_t$ be the estimate from sparse SOGD at time $t$, whose description uses the samples in $\Dcal_{t-1} \subseteq \{(x_i, x_i^+)\}_{i=1}^{t-1}$. Let $\Ical_{t-1} \subseteq \{1, \ldots, t-1\}$ be  the indices among $1,\cdots,t$ for which $(x_i,x^{+}_i)$ are in $\Dcal_{t-1}$. Let $W_{t-1}$ encode the weights in the description of $U_t$.
After receiving a new sample pair $(x_{t},x^{+}_{t})$, we decide whether to add it to the current dictionary $\Dcal_{t-1}$ or discard it based on its contribution to the description of the estimator. More precisely,  use SOGD to update, 
\begin{align}
    \tilde{U}_{t+1} = U_t -  \eta_t \tilde{\nabla} R_\lambda \left(x_t,x^{+}_t;U_t\right), \quad t \in \Nset,
    \label{eq.vanilla-OSGD}
\end{align}
with the operator description using $\Dcal_{t-1} \cup (x_{t},x^{+}_{t})$, where the weights $\tilde{W}_{t}$ are obtained from $W_{t-1}$ using the right-hand-sides of \eqref{eq.U.w}.
We now test whether $\tilde{U}_{t+1}$ can be well approximated within a desired accuracy level by a combination of kernel functions centered at elements in the old dictionary, $\Dcal_{t-1}$. That is, consider the orthogonal projection of $\tilde{U}_{t+1}$ onto the closed subspace, $\text{span}\left\{\phi(x^{+}_{i}) \otimes \phi(x_j):i,j \in \Ical_{t-1} \right\}$, i.e., 
$\hat{U}_{t+1}:= \Pi_{\Dcal_{t-1}}\left[ \tilde{U}_{t+1}\right] $, described by the coefficient matrix $\hat{W}_t$, given by
\begin{align}
{\small
\begin{aligned}
\hat{W}_{t} =& \argmin_{Z \in \Rset^{|\Ical_t|\times |\Ical_t|}} \vnorm{\sum_{i\in \Ical_t} \sum_{j \in \Ical_t} Z^{ij} \phi(x^{+}_{i}) \otimes \phi(x_j) -\sum_{i \in {\Ical}_t \cup \{t\}} \sum_{j \in {\Ical}_t \cup \{t\}} \tilde{W}^{ij}_{t} \phi(x^{+}_{i}) \otimes \phi(x_j)}_{\HS}^2.
\end{aligned}}
\label{eq.CME.CMP.W}
\end{align}
If the error due to sparsification is within a pre-selected  sparsification budget $\ve_t$, i.e., 
\begin{align}
    \vnorm{\hat{U}_{t+1} - \tilde{U}_{t+1}}_{\HS} \leq \ve_{t}.
    \label{eq.CME.stop}
\end{align}
Then, we discard the new sample $(x_{t},x^{+}_{t})$ and maintain the same dictionary as before, i.e.,
$\Dcal_t \leftarrow \Dcal_{t-1}$, $\Ical_t \leftarrow \Ical_{t-1}, W_t \leftarrow \hat{W}_t$.
If the sparsification error exceeds $\varepsilon_t$, we keep the new sample $(x_{t},x^{+}_{t})$ and update $\Dcal_t \leftarrow \Dcal_{t-1} \cup (x_{t},x^{+}_{t}), \Ical_t = \Ical_{t-1} \cup t$. The weighting matrix becomes $W_{t} \leftarrow \tilde{W}_{t}$ from \eqref{eq.U.w}.
In both cases, the estimate at time $t+1$ can be computed based on $\Dcal_t$ and $W_{t}$ as
\begin{align}
    \displaystyle U_{t+1} = \sum_{i \in \Ical_t} \sum_{j \in \Ical_t} W^{ij}_{t} \phi(x^{+}_{i}) \otimes \phi(x_j). 
\label{eq.CME.U.update}
\end{align}

In summary, our approach attains a sparse representation of $U_{t+1}$, and the complexity of the representation only depends on the cardinality of $\Dcal_t$ at each $t \in \Nset$. 
Recall from Theorem \ref{thm.Koopma-CME} that the Koopman operator can be defined as the adjoint of $U$. As such, we construct approximates of the Koopman operator $\{K_t\}_{t \in \Nset}$ as $K_t:=U_t^*$ for each $t \in \Nset$. 
The procedure is summarized in Figure \ref{fig1} and Algorithm \ref{algorithm.main}. 

\input{4-algorithm-box}

A keen reader might recognize our algorithm as \emph{kernel matching pursuit} in \cite{vincent2002kernel,koppel2017parsimonious} function learning. We generalize that methodology to vector-valued RKHS, making it applicable to the operator learning problem \eqref{eq.regression}. In the next section, we provide asymptotic and last-iterate convergence guarantees with samples from trajectories of a Markov process, whose analysis is substantially different than scalar-valued function learning as studied by \cite{bach2013non,tarres2014online,smale2009online}.

\subsection{An Illustrative Example}
\label{sec.4.2}
Before diving into the convergence analysis of our algorithm, we provide an illustrative example of its use. Consider the Langevin dynamics described by
$\mathrm{d} X_t = -\nabla V (X_t) \mathrm{d} t + \sqrt{2 \beta^{-1} } \mathrm{d} B_t$, 
with $x=[x_1,x_2]$, $V (x)=(x_1^2-1)^2+(x_2^2-1)^2$ and $\beta = 4$. As plotted in Figure \ref{fig.4well.phase}, a trajectory stays within one of the four potential wells, while rare transitions happen as ``jumps'' between four metastable sets.
Since the spectrum of $K$ encodes state space connectivity information, in this experiment, we apply the algorithm to identify said metastable sets. Figure \ref{fig.4well.2000}, \ref{fig.4well.20000}, \ref{fig.4well} plot leading eigenfunctions of $K_t$ at various iterations, and Figure \ref{fig.4well} reveals the distinct metastable sets. By leveraging the sparsification mechanism, we are able to control the growth of model complexity such that $|\Dcal_t| \ll t$ to alleviate computational and storage issues. The details are included in Appendix \ref{appdx.CMP}

\begin{figure}[ht]
    \centering
     \subfigure[]{\includegraphics[width=0.2\linewidth]{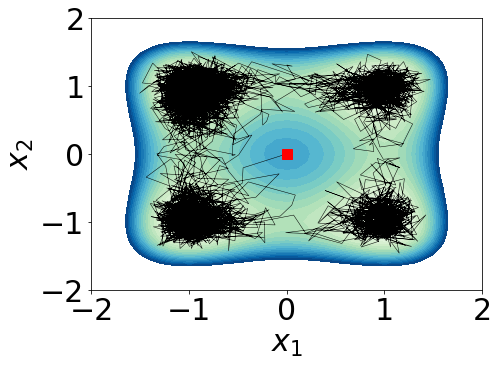}
     \label{fig.4well.phase}}
     \hspace{0.1in}
     \subfigure[{\small Iterate $2000$}]{\includegraphics[width=0.2\linewidth]{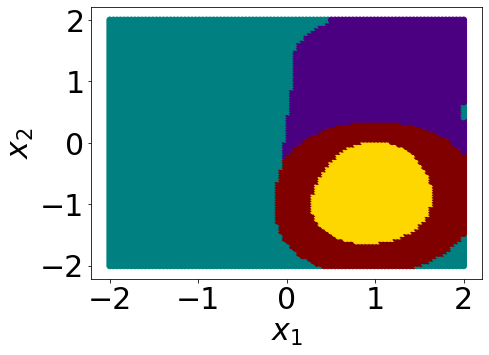}\label{fig.4well.2000}}
     \hspace{0.1in}
     \subfigure[{\small Iterate $20000$}]{\includegraphics[width=0.2\linewidth]{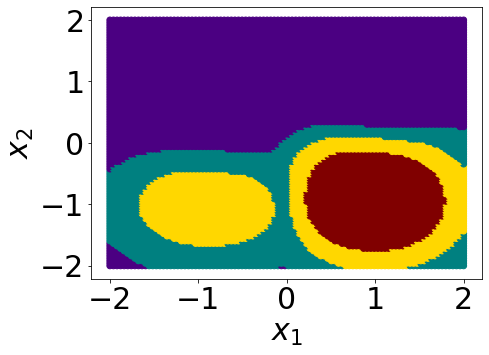}\label{fig.4well.20000}}
     \hspace{0.1in}
     \subfigure[{\small Iterate $40000$}]{\includegraphics[width=0.2\linewidth]{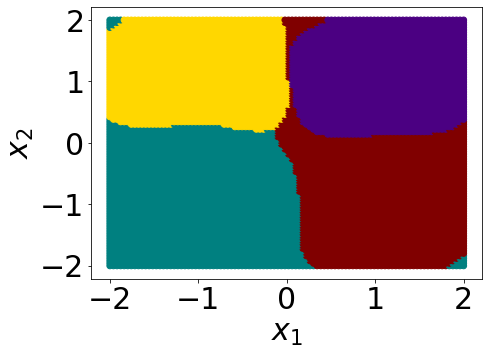}\label{fig.4well}}
    \vspace{-0.1in}
    \caption{(a) Potential landscape and one trajectory of the Langevin dynamics; (b),(c),(d) four metastable sets obtained from leading eigenfunctions of $K_t$ at various iterates, where 
    (b) $t=2000$, $|\Dcal_t| = 101$,
    (c) $t = 20000$,$|\Dcal_t| = 134$, and
    (d) $t = 40000$,$|\Dcal_t| = 145$.}
\end{figure}

%% file: 4-algorithm-box.tex
\begin{algorithm2e}[H]
\SetAlgoLined
\SetKwInOut{Input}{input}\SetKwInOut{Output}{output}
\caption{Sparse Online Learning of the Koopman operator}
\Input{$\{(x_t,x^{+}_t)\}_{t \in \Tset}$, $\kappa$,$\left\{\eta_t\right\}_{t \in \Tset}$, $\{\ve_t\}_{t \in \Tset}$}
\BlankLine
Initialize $U_0=0$
\BlankLine
\For{$t \in \Tset$}
    {Receive a sample pair $(x_t,x^{+}_t)$
    \\
    $\tilde{\Dcal}_t \leftarrow \Dcal_{t-1} \cup (x_t,x^{+}_t)$
    \\
    Compute $\tilde{W}_{t}$ based on $\tilde{\Dcal}_t$ via \eqref{eq.U.w}
    \\
    Compute $\displaystyle \Delta_{t} \leftarrow \min_{Z} \vnorm{\sum_{\substack{i,j\in \Ical_{t-1}}} Z^{ij} \phi(x^{+}_{i}) \otimes \phi(x_j) -\sum_{i,j \in \tilde{\Ical}_{t}} \tilde{W}^{ij}_{t} \phi(x^{+}_{i}) \otimes \phi(x_j) }_\HS^2$
    \\
    \eIf{$\Delta_{t} < \ve_t$}{
    $\Dcal_t \leftarrow \Dcal_{t-1}$,$W_{t} \leftarrow Z_\star$}
    {$\Dcal_t \leftarrow \tilde{\Dcal}_{t}$, 
    $W_{t} \leftarrow \tilde{W}_{t}$
    }
    Compute $U_{t}$ according to \eqref{eq.CME.U.update}.\\
\Output{The Koopman estimate $K_t \gets U_t^* $}
}
\label{algorithm.main}
\end{algorithm2e}

%% file: 5-assumptions.tex
\section{Convergence Analysis with Trajectory-Based Sampling}
\label{sec.5.analysis}

We now present our theoretical results on the convergence behavior of the sparse SOGD algorithm proposed in Section \ref{sec.algorithm}. 
Following Section \ref{sec.Koopman}, we make the following assumption on the regularity of $K$, which encodes the regularity of the Markov kernel. 
\begin{assumption}
$K \in \HS(\Hcal,\left[\Hcal\right]^\beta)$ for some $\beta \in (0, 1)$ and $\vnorm{K}_{\HS\left(\Hcal,\left[ H\right]^\beta\right)}
\leq \Bsrc <\infty$.
\label{assumption.src}
\end{assumption}

Our ultimate goal is to understand how closely $K_t$ approximates $K$ with respect to the norm
$\vnorm{\cdot}_{\HS(\Hcal , \left[\Hcal\right]^\beta)}$. 
To obtain error estimates, using the triangle inequality, we have
\begin{align}
\vnorm{\left[K_t\right] - K}_{\HS(\Hcal , \left[\Hcal\right]^\beta)}
\leq \vnorm{\left[K_t - K_\lambda\right]}_{\HS(\Hcal , \left[\Hcal\right]^\beta)}
+\vnorm{\left[K_\lambda\right] - K}_{\HS(\Hcal , \left[\Hcal\right]^\beta)}.
\label{eq.K.bias-variance}
\end{align}
The first term on the right-hand side depends on the stochastic sample path. It captures sampling error with respect to the norm of the intermediate space defined in Section \ref{sec.RKHS}.
The second term equals the bias in approximating an operator in the mis-specified case. The next lemma provides an upper bound for \eqref{eq.K.bias-variance} by bounding the sampling error and bias separately; for proof, see Appendix \ref{pf.lemma.Kt}.

\begin{lemma}
Under Assumptions \ref{assumption.polk} and \ref{assumption.src}, 
\begin{align}
\begin{aligned}
\vnorm{\left[K_t\right] - K}_{\HS(\Hcal , \left[\Hcal\right]^\beta)}^2
\leq & 2 B_\infty^{1-\beta} \vnorm{U_t - U_\lambda}_{\HS}^2
+  2 \Bsrc^2.
\end{aligned}
\label{eq.mu_t.gamma}
\end{align}
\label{lemma.Kt}
\end{lemma}
The above result suggests that we must focus on the study of the convergence of the sequence of HS operators $\left\{U_t\right\}$ to $U_\lambda$ in the HS-norm. 
This simplification bears a resemblance to the existing work by \cite{li2022optimal}. Yet our analysis is substantially distinct from theirs in the sense that we consider online learning with trajectory-based sampling rather than batch learning with IID samples. \rev{That is, our analysis is stochastic approximation-based, with online streaming samples, rather than batch sample–average–based methods that process an entire dataset all at once.}
Furthermore, we construct a sparse representation for each iterate to control model complexity.
Since each iteration induces an extra error, we carefully handle a compounding bias that arises from sparsification by controlling the step-sizes.
To assist the analysis, define an $\left\{\Fcal_t\right\}_{t\in\Nset}$-adapted sequence $\left\{E_t\right\}_{t\in\Nset}$ where $E_{t}:= U_{t+1} - \tilde{U}_{t+1}$ encodes the error due to sparsification to write the output of our algorithm as 
\begin{align}
\begin{aligned}
U_{t+1} = U_t + \eta_t \left(- \tilde{\nabla} R_\lambda(x_t,x^+_t;U_t) + \frac{E_{t}}{\eta_t} \right), \quad U_0=0.
\end{aligned}   
\label{eq.CME.fixedpt.SA}
\end{align}
Here, $\vnorm{E_t}_\HS\leq \ve_t$ from \eqref{eq.CME.stop}.
We make the following assumption.
\begin{assumption}
    (a) The step-size sequence $\left\{\eta_t\right\}_{t \in \Nset}$ satisfies: $0<\eta_{t+1} \leq \eta_t< 1/\lambda$, and (b) $\ve_t \leq \bcmp \eta_t^2$ for some $\bcmp>0$ for all $t \in \Nset$. 
    \label{assumption.3}
\end{assumption}

We next delineate precise requirements on the Markov process described by the kernel $P(\cdot|x)$. Recall that we aim to sample from a single trajectory of a Markov process starting from an arbitrary initial state. While we have not formally presented the data collection process, it is not difficult to conceive that the measure of consecutive states along such a trajectory will evolve through time. For the definition of CME, we require $\rho_X$ over $\Xset$ to be a static measure from which we hope to draw samples. A natural candidate for such a measure is a stationary invariant measure of the Markov process, which is guaranteed to exist under the next assumption. We use the notation TV to denote total variation distance and $P^t$ to be the $t$-times product of $P$.
\begin{assumption}
$\left\{X_t\right\}_{t \in \Nset}$ described by the kernel $P$ is uniformly geometrically ergodic with a unique invariant measure $\pi$ over $(\Xset, \Bcal_X)$, i.e.,
\begin{align}
\sup_{x\in\Xset} \big\| P^t(\cdot|x) - \pi \big\|_{\mathrm{TV}}
\leq R\gamma^{\,t}, \label{eq:geom_mixing}
\end{align}
for all $t \in \Nset$ for some $R>0$ and $\gamma\in(0,1)$. Let $\rho_X = \pi$ and $\rho = \pi P'$.
\label{assumption.mixing}
\end{assumption}

Next, we describe our data collection process. Consider a time-homogeneous Markov process $\{X_t\}_{t\in\Nset}$ evolving over $(\Xset, \Bcal_X)$, starting from an arbitrary $X_1$ sampled according to the measure $\prob_1$ over $\Xset$, and then pushed forward via the kernel $P$ at each time, i.e.,
\begin{align}
    \prob\{X_1 \in \Aset_1, \ldots, X_t \in \Aset_t\} = \int_{\Aset_1} \ldots \int_{\Aset_t} \prob_1(\dd x_1) \prod_{k=1}^{t-1} P(d x_{k+1} | x_k).
\end{align}
Then, we take a single trajectory by sampling $X_1 \sim \prob_1$ and recursively propagating $X_t$ through $P$ to produce a trajectory. The samples are then given by $(X_1, X_2), (X_2, X_3), \ldots, (X_t, X_{t+1})$ with which we run the sparse SOGD algorithm.

%% file: 5-asymptotic.tex
\subsection{Asymptotic Convergence}
  
\begin{theorem}
And let Assumptions \ref{assumption.polk}, \ref{assumption.src}, \ref{assumption.3} and \ref{assumption.mixing}  hold. Assume that the stepsize sequence $\{\eta_t\}_{t \in \Nset}$ satisfies $\sum_{ t \in \Nset} \eta_t = \infty$, and $\sum_{t \in \Nset} \eta_t^2 < \infty$, and that there exist two deterministic real-valued sequences $\{a_t\}_{t\in \Nset}$ and $\{b_t\}_{t \in \Nset}$ such that,
    \begin{gather}  
    \begin{gathered}
    \vnorm{\E \left[\tilde{\nabla} R_\lambda \left(x_t,x_t^+,U\right)|\Fcal_t\right] - \nabla R_\lambda\left(U\right)}_{\HS} \leq a_t,
    \
    \E \left[ \vnorm{\tilde{\nabla} R_\lambda \left(x_t,x_t^+,U\right)}_{\HS}^2 |\Fcal_t\right]  \leq b_t^2, 
    \end{gathered}
    \label{assumption.asy.a.b.c}
    \end{gather}
for all $t \in \Nset$, and they also satisfy $ \sum_{t \in \Nset} \eta_t a_t <\infty$, and $ \sum_{t \in \Nset} \eta_t^2 b_t^2 <\infty$.
Then, 
\begin{align}
\lim_{t \to \infty} \vnorm{\left[K_t\right] - K}^2_{\HS(\Hcal , \left[\Hcal\right]^\beta)}
\leq 2 \Bsrc^2 , \quad \rho-\text{a.s.}.
\end{align}
\label{corollary.Koopman.asy}
\end{theorem}

This result reveals that the iterates converge \emph{almost surely} to a neighborhood of $K$, the size of which depends on the regularization parameter $\lambda$ and the regularity of the true Koopman operator, measured by $\beta$. Moreover, a diminishing stepsize sequence forces the same on the sparsification budget, i.e., $\ve_t$ approaches $0$ as $t \to \infty$. Asymptotically, under Assumption \ref{assumption.3}, sparsification does \emph{not} degrade the quality of the learned operator. While this may seem counterintuitive, the explanation lies in the interplay between step-size and the sparsification budget: as the sparsification budget keeps shrinking concomitantly with the step-size, it becomes harder to ignore any data point from the dictionary. Although some of the points may be discarded early in training, the \rev{uniformly geometrically ergodic} process generates data that corrects for any early errors, making the effect of sparsification vanish eventually.
Our proof utilizes the almost supermartingale convergence theorem \cite{robbins1971convergence} and is presented in Appendix \ref{appendix.CME.asy}.
The proof shows that the iterates remain bounded, making them Lyapunov stable.

%% file: 5-FT.tex
\subsection{Finite-Time Convergence Analysis}
\label{sec.FT}

Next, we study the finite-time behavior of our sparse SOGD algorithm. 
Under Assumption \ref{assumption.mixing}, the process has sufficiently mixed after $\tau(\delta)$ steps for $\tau(\delta)$ defined as follows. 
For $\delta>0$, define $\tau(\delta) := \tmin \left\{ s \in \Nset:  R \gamma^s \leq \delta \right\}$, implying that after $\tau(\delta)$ time, \rev{$\sup_{x\in\Xset} \big\| P^t(\cdot|x) - \pi \big\|_{\mathrm{TV}} \leq \delta$.}
Then $\tau(\delta)$ satisfies $R \gamma^{\tau(\delta)} \leq \delta$ and $R \gamma^{\tau(\delta)-1} \geq \delta$, and the latter implies
\begin{align}
&\tau(\delta) 
\leq \frac{\log (R/\gamma) + \log (1/\delta)}{\log(1/\gamma)} 
\leq B_\mix \left(\log \frac{1}{\delta} +1 \right),
\label{eq.mixing0}
\end{align}
where $B_\mix = \max{\left\{\frac{1}{\log(1/\gamma)}, \frac{\log (R/\gamma)}{\log(1/\gamma)}\right\}}$.
In contrast to IID sampling, trajectory-based sampling yields biased gradient estimates. We control this bias to generate the following result. \rev{The proof is presented in the longer version of the paper \cite[Appendix G.3]{hou2024nonparametric}. }

\begin{lemma}
Under Assumptions \ref{assumption.polk}, \ref{assumption.3} and \ref{assumption.mixing}, for any $\delta>0$, $s \in \Nset$, and $t \geq \tau(\delta)$,
\begin{align}
\begin{aligned}
\vnorm{\E \left[\tilde{\nabla} R_\lambda \left(x_{t+s},x^{+}_{t+s};U\right)|\Fcal_s\right] - \nabla R_\lambda\left(U\right)}_{\HS} 
 \leq 2 \left(B_\infty + \lambda\right) \delta \left(\vnorm{U}_{\HS}+1\right).
\end{aligned}
\end{align}
\label{lemma.CME.mixing}
\end{lemma}

We adopt a Lyapunov-drift argument \cite{srikant2019finite,chen2022finite}, which is originally designed for stochastic approximation in Euclidean spaces, to study the stochastic \emph{operator} gradient descent with {sparsification}. The argument closely resembles the (informal) analysis of the continuous-time dynamics $\dot{U}(t)  = -\nabla R_\lambda\left(U\left(t\right)\right)$ for $U \in {\HS(\Hcal,\Hcal)}$ for which one can show that $d\vnorm{U(t) - U_\lambda}_{{\HS}}^2/dt \leq -2\lambda \vnorm{U(t) - U_\lambda}_{{\HS}}^2$, and then view \eqref{eq.CME.fixedpt.SA} as its discrete, biased, and stochastic counterpart.

Under Assumption \ref{assumption.3}(b), there exists some $B_\ve>0$ such that $\eta_t \leq B_\ve / \bcmp$ for all $t \in \Nset$.
Denote $B = B_\infty + \lambda + B_\ve$, $\check{B} := 98 B^2  + 32 B$, and $\eta_{t - \tau_t,t-1} := \sum_{k=t - \tau_t}^{t-1} \eta_k$, 
and $\tau_t:=\tau(\eta_t)$.
The following result provides the one-step drift in expectation; see Appendix \ref{thm.CME.FT.one-step.pf} for a proof. 
\begin{lemma}
(One-Step Stochastic Descent Lemma)
Let Assumptions \ref{assumption.polk}, \ref{assumption.3}, and define \ref{assumption.mixing} hold.
Then, for $t\geq \tau_t$ and step-sizes such that $\eta_{t-\tau_t,t-1} \leq 1/4B$, 
\begin{align}
\begin{aligned}
\E\left[ \vnorm{U_{t+1} - U_\lambda}_{\HS}^2\right]
\leq&
\left( 1 - 2\eta_t \lambda  + \check{B} \eta_t \eta_{t-\tau_t,t-1} \right)
\E \left[ \vnorm{U_t - U_\lambda}_{\HS}^2 \right]\\
&+ \check{B} \eta_t \eta_{t-\tau_t,t-1} \left(\vnorm{U_\lambda}_{\HS}+1\right)^2 
+ 4 \ve_t B_\infty/\lambda .
\end{aligned}
\label{eq.CME.FT.one-step.1}
\end{align}
In addition, if for all $t \geq \tau_t$, the stepsizes satisfy $\eta_{t-\tau_t,t-1} \leq \lambda/ \check{B}$, then, 
\begin{align}
\begin{aligned}
\E\left[ \vnorm{U_{t+1} - U_\lambda}_{\HS}^2\right]
\leq & \left(1 -\lambda \eta_t \right)\E\left[\vnorm{U_t - U_\lambda}_{\HS}^2\right]\\
&+ \check{B} \eta_t \eta_{t-\tau_t,t-1} \left(\vnorm{U_\lambda}_{\HS}+1\right)^2 + 4 \ve_t B_\infty/\lambda.
\end{aligned}
\label{eq.CME.FT.one-step.2}
\end{align}
\label{thm.CME.FT.one-step}
\vspace{-0.2in}
\end{lemma}

We remark that our choice of the sparsification budget $\{\ve_t\}_{t\in \Nset}$ stated in Assumption \ref{assumption.3}(b) guarantees that the first summand on the right-hand side of the inequality is the dominant term. Hence, \eqref{eq.CME.FT.one-step.2} becomes a one-step contraction. 
Using Lemma \ref{lemma.Kt} and Lemma \ref{thm.CME.FT.one-step}, we present our main result below. Its proof is deferred to Appendix \ref{thm.CME.general.pf}. 
\begin{theorem}
Let Assumptions \ref{assumption.polk}, \ref{assumption.src} , \ref{assumption.3}, and \ref{assumption.mixing} hold. Also, assume $\eta_{t-\tau_t,t-1}\leq \min\{1/(4B), \lambda/\check{B}\}$ for all $t \geq \tau_t$.
For $ r>s>\tau_t$, define $\Psi(r,s) := \Pi_{j=s}^{r} \left(1 -\lambda \eta_j \right)$. Let $\tau_*=\min\{t:t \geq \tau_t\}$.
Then for all $t \geq \tau_*$, 
\begin{align}
\begin{aligned}
\E \left[\vnorm{\left[K_{t}\right] - K }_{\HS(\Hcal,\left[\Hcal\right]^\beta}^2) \right] 
\leq
&2 B_\infty^{1-\beta}
\left(  4\frac{B_\infty^2}{\lambda^2} 
\Psi(t-1,\tau_*) \right.\\
&\left.+ \sum_{i=\tau_*}^{t-1} \Psi(t-1,i+1) \Theta_1\left(i,\bcmp,\lambda\right)\right) + 2 \Bsrc^2.
\end{aligned}
\end{align}
where $\Theta_1\left(t,\bcmp,\lambda\right)
:= \check{B} \eta_t \eta_{t-\tau_t,t-1}   
+ 4 \bcmp \eta_t^2 B_\infty/\lambda$.
\label{thm.Koopman.FT.general}
\end{theorem}

The preceding result only requires $K$ to be Hilbert-Schmidt from $\Hcal$ to $\left[\Hcal\right]^\beta$ where the constant $\beta$ reflects the degree of mis-specification in operator learning. 
The number of required samples is independent of the dimension of the state space of the underlying data. This observation is useful for solving problems where the state space is high-dimensional. 
Finally, we remark that by \eqref{eq.mixing0}, the condition $t \geq \tau_*$ can be satisfied as long as $\eta_t$ does not decay faster than $e^{-(t/B_\mix-1)}$.

To better illustrate Theorem \ref{thm.Koopman.FT.general}, we now specialize them under two types of stepsize choices. 
The proof of these results is included in Appendix \ref{thm.FT.const.pf} and Appendix \ref{thm.FT.diminish.pf}. 
\begin{assumption}
The stepsize is constant, i.e., $\eta_t = \eta$, $\forall t \in \Nset$.
\label{assumption.SA-constant}
\end{assumption}
\begin{corollary}
Let Assumptions \ref{assumption.polk}, \ref{assumption.src}, \ref{assumption.3}, \ref{assumption.mixing} and \ref{assumption.SA-constant} hold. If $\eta \tau_\eta \leq \lambda/\check{B}$, then for all $t \geq \tau_\eta$, 
\begin{align}
\E \left[\vnorm{\left[K_t\right] - K }_{\HS(\Hcal,\left[\Hcal\right]^\beta)}^2 \right]
\leq  \Theta_2  \left(1 -\lambda \eta \right)^{t-\tau_\eta} + 
 \Theta_3 \eta
+2 \Bsrc^2,
\label{eq.thm.FT.const}
\end{align}
where 
$\Theta_2:=\frac{8 B_\infty^{3-\beta}}{\lambda^2}$,
$\Theta_3:= 2 B_\infty^{1-\beta} \left(\check{B}  \tau_\eta \left(\vnorm{U_\lambda}_{\HS}+1\right)^2  
+ 4 \bcmp B_\infty/\lambda \right)$.
\label{corollary.FT.const}
\end{corollary}

Since $\delta \tau(\delta) \leq B(\delta \log(1/\delta) + \delta) \to 0$, the condition on stepsize can be satisfied. 
In the above result, $\Theta_3$ captures the effect of sparsification through $\bcmp$ defined in Assumption \ref{assumption.3}. After an initial transient period, the error decays exponentially fast in the mean square sense and the iterates converge to a ball centered at $K$, with a radius depending on the stepsize $\eta$, sparsification budget $\ve$ (through $\bcmp$), the regularization parameter $\lambda$, and the degree of mis-specification encoded in $\Bsrc$. 
The dependency of the quality of the learned parameter on the sparsification budget in finite time lies in sharp contrast to the asymptotic independence of the same.

\begin{assumption} 
Assume $\eta_t = \frac{\eta}{(t+r)^a}$ for all $t \in \Nset$ for some $a \in (0,1)$, $r >0$. 
\label{assumption.SA-diminishing}
\end{assumption}

\begin{corollary}
Let Assumptions \ref{assumption.polk}, \ref{assumption.src}, \ref{assumption.3}, \ref{assumption.mixing} and \ref{assumption.SA-diminishing} hold. In addition, assume $r$ in Assumption \ref{assumption.SA-diminishing} is chosen such that $\eta_{t-\tau_t,t-1}\leq \lambda/\check{B}$ for all $t \geq \tau_t$, and $\tau_t \geq (\frac{2a}{\lambda \eta})^{\frac{1}{1-a}}$.
Define $\Theta_4\left(t+r\right)= 2  \left( B_\mix \check{B} \left(\log\left(t+r\right) - \log\left(\eta\right) +1 \right)
\left(\vnorm{U_\lambda}_{\HS}+1\right)^2 + 4 \bcmp B_\infty/\lambda \right)$.
Then for all $t \geq \tau_*$,
\begin{align}
\begin{aligned}
\E \left[ \vnorm{\left[K_{t}\right] - K }_{\HS(\Hcal,\left[\Hcal\right]^\beta)}^2 \right] 
\leq & \Theta_2 \exp\left(-\frac{\lambda \eta}{1-a} \left(\left(t+r\right)^{1-a} -\left(\tau_t+r\right)^{1-a}\right) \right) \\
&+ \frac{4 \eta B_{\infty}^{1-\beta} }{(t+r)^a}
 \Theta_4\left(t+r\right)
+2 \Bsrc^2.
\end{aligned}
\end{align}
\label{thm.FT.diminish}
\end{corollary}

Due to Assumption \ref{assumption.3}(b), the sparsification budget decays faster than the stepsize, and the asymptotic error only depends on the regularization parameter and $\Bsrc$, where the latter encodes the degrees of mis-specification. In other words, we attain accuracy at the price of model complexity in this result. \rev{We remark that with the choice of $\eta_t = \eta/(t+r)^{a}$ with $\tau_t \sim \log(t+r)$, one can always satisfy the requirements in the statement of the corollary by tuning $\eta$. Also, the upper bound for large $t$ scales as $t^{-a}\log(t)$ added to a bias that arises due to the regularity of the Markov kernel.}

%% file: 6-1.tex
\section{Applications}
\label{sec.example}

\subsection{Analyzing Unknown Nonlinear Dynamics}
Consider the unforced Duffing oscillator, described by 
$\ddot{z}=-\delta \dot{z}-z\left(\beta+\alpha z^{2}\right)$,
with $\delta=0.5$, $\beta=-1$, and $\alpha=1$,
where $z\in\Rset$ and $\dot{z}\in\Rset$ are the scalar position and velocity. 
Let $x=(z,\dot{z})$, as shown in Figure \ref{fig.duffing.phase}, the Duffing dynamics exhibits two ROAs, corresponding to stable equilibrium points at $x=(-1,0)$ and $x=(1, 0)$. In this experiment, we leverage the eigenfunction of the learned Koopman operator to characterize the regions of attraction. 
Our data consists of $3550$ streaming sample pairs collected 
over region $[-2,2]\times[-2,2]$ with sampling interval $\tau = 0.25$s.
Figure \ref{fig.duffing}-\ref{fig.duffing.bad} portrays heat maps of the leading eigenfunctions of $K$ after $3550$ iterations with various values of budget $\ve$.
Upon increasing $\ve$, the dictionary becomes sparser with fewer elements. 
As shown in Figure \ref{fig.duffing.cmp}, the resulting eigenfunctions accurately reveal the distinct ROAs, even with merely \emph{$8\%$} of total data points. And the characterization becomes less sound with higher $\ve$ as the algorithm discards too many points. More details, such as the choice of kernel function and step sizes, are provided in the extended version of this paper \cite[Section 6.1]{hou2024nonparametric}. Notice that while our theoretical analysis is restricted to a single trajectory on a uniformly exponentially ergodic system, the Duffing oscillator with two ROAs violates it. Yet, the method, when applied and averaged across multiple long trajectories, yields the ROAs, demonstrating its efficacy.

\begin{figure}[ht]
    \centering
     \subfigure[]{\includegraphics[width=0.19\linewidth]{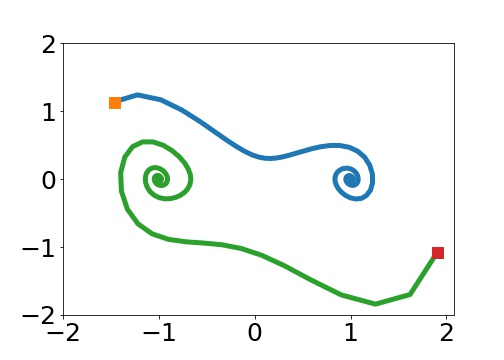}
     \label{fig.duffing.phase}}
     \hspace{0.1in}
     \subfigure[{\small $\ve=0$}]{\includegraphics[width=0.2\linewidth]{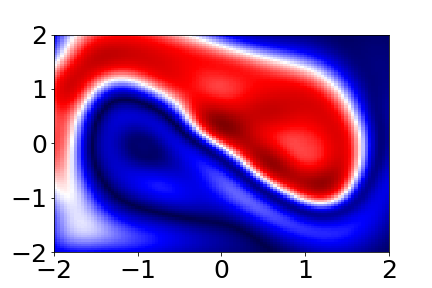}\label{fig.duffing}}
     \hspace{0.1in}
     \subfigure[{\small $\ve = 2 \eta^3$}]{\includegraphics[width=0.2\linewidth]{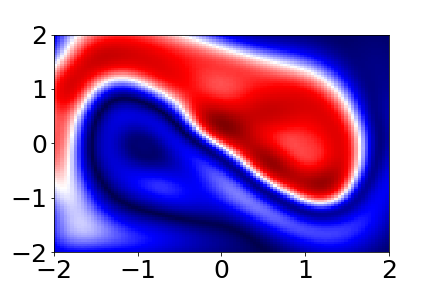}\label{fig.duffing.cmp}}
     \hspace{0.1in}
     \subfigure[{\small $\ve = 1.5 \eta^2$}]{\includegraphics[width=0.2\linewidth]{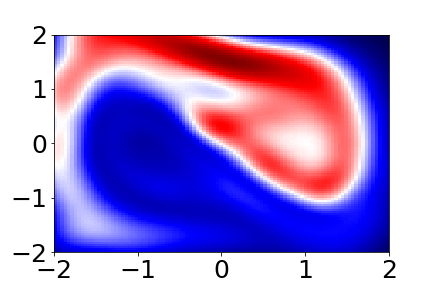}\label{fig.duffing.bad}}
    \vspace{-0.1in}
    \caption{(a) Two trajectories of the Duffing oscillator that converge to two different equilibrium points. (b)-(d) Leading eigenfunction of $K$ with eigenvalue $1$ at $t=3550$ under various compression budget with
    (b)$\ve = 0$,$|\Dcal_t| = 3550$,
    (c) $\ve = 2 \eta^3$,$|\Dcal_t| = 300$, and
    (d) $\ve = 1.5 \eta^2$,$|\Dcal_t| = 190$.}
    \vspace{-0.35in}
\end{figure}

%% file: 6-2.tex
\subsection{Model-Based Reinforcement Learning}
\label{sec.RL}
While previous sections focused on uncontrolled dynamical systems, the proposed sparse online learning framework can be extended to Markov decision processes (MDPs) by using CME--the adjoint of the Koopman operator in RKHS. Specifically, consider an MDP with compact state and action spaces $\Xset$ and $\Uset$ which are subsets of finite-dimensional Euclidean subspaces. The state dynamics are described by a transition kernel function $x_{t+1}\sim p(\cdot|x_t,u_t)$, where $x_{t}\in \Xset$, $u_{t}\in \Uset$, and $x_{t+1}\in \Xset$. The value function at $x\in \Xset$, i.e., the expected cost starting from state $x$, satisfies 
\begin{equation}\label{eq.bellman}
    (\Bcal V)(x): = \min_{u\in \Uset}\left\{c(x,u) + \gamma \E[V(X^+)|(x,u)]\right\},
    \end{equation}
where $c\,:\,\Xset\times \Uset\to \Rset$ is the instantaneous cost function, and $\gamma\in (0,1)$ is a discount factor. Starting from an arbitrary $V_0$, the sequence $\{V_k\}$ defined via value iteration steps $V_{k+1} = \Bcal V_k$ converges in sup-norm to an optimal value function \cite{szepesvari2022algorithms}.
Let $\Zset = \Xset\times \Uset$ and $Z$ be a $\Zset$-valued random variable. For $f\in \Hcal_X$, the mapping $f\mapsto\E[f(X^+)|Z]$ can be implemented using the CME defined in \eqref{eq:def.muYx} as 
$\E_{X^+|z}[f(X^+)|Z] = \langle f,\mu_{X^+|Z}\rangle$, per \cite{grunewalder2012modelling},
where $\mu_{X^+|Z}$ is the CME of $X^+$ given current state-action pair $z=(x,u)$. With an estimate of $\hat{\mu}$ given by Algorithm \ref{algorithm.main} as $\mu_t = U_t \phi(\cdot)$, we can approximate this mapping along with the value function estimate $\hat{V}$. A corresponding greedy policy $\pi_{\hat{\mu}}$ can be executed at any state $x\in \Xset$ via 
\begin{equation}\label{eq.greedy_policy}
\pi_{\hat{\mu}}(x) = \underset{u\in \Uset}{\arg\min}\left\{r(x,u) + \gamma\left\langle \hat{\mu}_{X^+|(x,u)},\hat{V}\right\rangle\right\}.
\end{equation}

We now consider an online, sparse variant of the value iteration process. 
Given dataset $\{(x_i,u_i,x^+_i)\}_{i=1}^m$ and an associated weighting matrix $W$ calculated via Algorithm \ref{algorithm.main}, an estimate of $\mu_{X^+|Z}$ for a given $z=(x,u)$ is computed as
\begin{equation}\label{eq:CME_update}
    \hat{\mu}_{X^+|(x,u)} = \sum_{i=1}^m\alpha_i(x,u)\kappa_X(x^+_i,\cdot), \quad
    \alpha_i(x,u) = \sum_{j=1}^mW^{ij}\kappa_Z((x_j,u_j),(x,u))
\end{equation}
per \cite{grunewalder2012modelling}. 
Assuming that the desired value function $V\in \Hcal_X$, we have 
\begin{equation}
    \E_{X^+|(x,u)}[V(X^+)] \approx \langle \hat{\mu}_{X^+|(x,u)},V \rangle =\sum_{i=1}^m\alpha_i(x,u)V(x^+_i). 
\end{equation}
Thus, for policy iteration, it suffices to estimate the value function at each $x^+_i$ in the given dataset. This further implies that we need only compute weights $\alpha_i(x,u)$ for each $i$ at $m$ points and $u$ drawn from a finite subset of $\Uset$, e.g., a uniformly spaced grid. 

We applied the sparse online value iteration mechanism to the pendulum dynamics implemented in the OpenAI Gym package \cite{brockman2016openai}. The approximated continuous system is governed by $\ddot{\theta}(t) =
(3g/2l) \sin \theta(t) + (3/ml^2)u(t)$, where $\theta$ is the pendulum angle, $g$ is the gravitational constant, $l =$ 1m is the pendulum length and $m=1$kg is the pendulum mass. The state space $\Xset$ is a subset of $\Rset^3$, with entries of the form $(\sin\theta,\cos\theta,\dot{\theta})$, where the angular velocity $\dot{\theta}$ is restricted to $[-8,8]$ and the action space (applied torque) $\Uset$ is the interval $[-2,2]$. Starting from an arbitrary initial state, the goal is to swing up and balance the pendulum in the inverted position. For discrete time-step $k$, the instantaneous cost function is $r(\theta[k],\dot{\theta}[k],u[k]) = -\left(\theta[k]^2 + 0.1\dot{\theta}[k]^2 + 0.001u[k]^2\right)$, where $\theta[k]$ is wrapped between $[-\pi,\pi]$. Episodes terminate after 200 steps. While the highest possible cumulative episode reward is 0, there is no particular performance-based threshold for us to declare that the pendulum balancing task is solved. A score of approximately $-400$ or higher usually indicates that the pendulum was brought upright near the goal position for a significant portion of the episode. As a baseline, high-resolution dynamic programming solutions using full knowledge of the system dynamics achieve average episode scores of roughly $-130$, per \cite{hou2023decentralized}.

In our experiments, we segmented our value iteration approach into stages as follows. Let $\Dcal_{\ell-1}$ denote the dictionary after completion of stage $\ell-1$ with set of indices $\Ical_{\ell-1}$. During stage $\ell$, $n_{\text{new}}$ data points $\Dcal_{\text{new}} = \{(x_i,u_i,x^+_{i+1})\}_{i=1}^{n_{\text{new}}}$ are generated by rolling out trajectories according to behavioral policy $\pi_\ell$. Algorithm \ref{algorithm.main} is executed on this new batch of data points, starting with initial dictionary $\Dcal_{\ell-1}$, yielding the updated dictionary $\Dcal_{\ell}\subset \Dcal_{\ell-1}\cup\Dcal_{\text{new}}$ with index set $\Ical_\ell$, and weight matrix $W_\ell$. A greedy policy with respect to dataset $\Dcal_{\ell}$ may then be derived using \eqref{eq.greedy_policy} and \eqref{eq:CME_update}. 

We implemented this approach, choosing $n_{\text{new}}=400$, so that $\Dcal_{\text{new}}$ consists of two new episode length trajectories, giving 400 new points prior to compression via Algorithm \ref{algorithm.main} with constant step size $\eta = 10^{-4}$ and 
$\varepsilon = 8.91\times 10^{-5}$  per iteration stage. We use the Gaussian kernel with a bandwidth parameter of $0.167$.   The behavioral policy $\pi_k$ in each iteration $k$ selected actions uniformly from $\Uset$ at each step. Other choices for $\pi_{\ell}$ include a greedy or $\epsilon$-greedy policy derived from the last value function estimate $V_{\ell}$. 
The upper plot in Figure \ref{fig.RL.bar} compares the performance of our CME value iteration (CME VI)-based controllers to the reference dynamic programming solution as the number of trajectories incorporated increases. As plotted, the median CME VI policy performance score approaches the reference, while the empirical score distribution concentrates toward the maximum cumulative reward.
At the same time, the lower plot in Figure \ref{fig.RL.bar} shows that our algorithm can achieve the task with control over model complexity via sparsification. In other words, our method presents a means by which dataset size and associated computational complexity can be balanced with performance. 
For example, the CME VI-based controller at stage 19 uses 6000 points, a 25\% reduction compared to the full dataset size of 8000. Finally, Figure \ref{fig.RL.pendulum} illustrates the value function convergence accompanying the performance increase seen in Figure \ref{fig.RL.bar}. As the dataset size increases, the estimated value functions capture important features of the reference such as the high-value diagonal passing through the stationary, upright pendulum position. 


\begin{figure}[ht]
    \centering
     \includegraphics[width=0.8\linewidth]{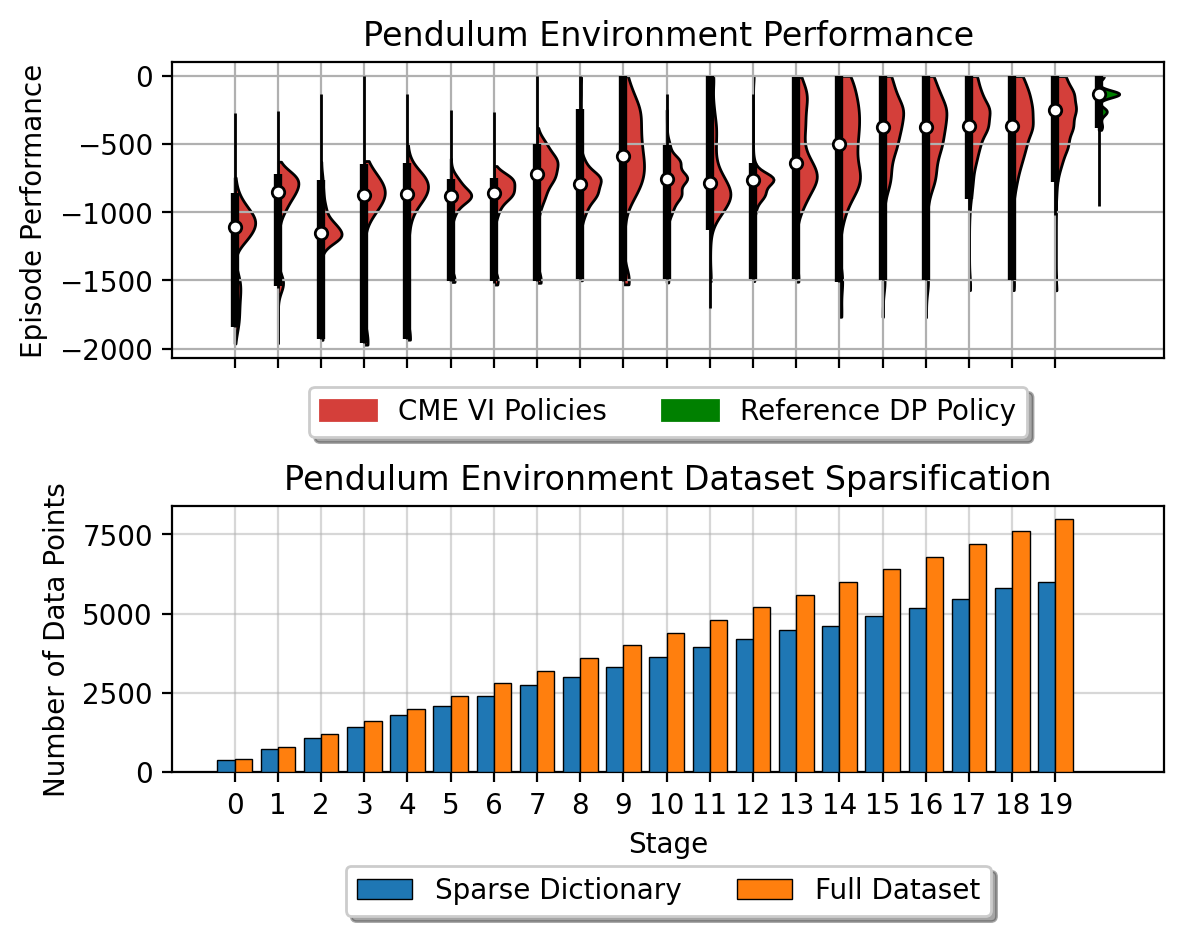}
     \label{fig.pendulum_data}
    \caption{(Top) White dots, bold bars, and whiskers give median, 95\% confidence intervals, and extreme values, respectively, over 1000 episodes. (Bottom) Growth of sparsified and full dataset with iteration stage.}
    \vspace{-0.1in}
    \label{fig.RL.bar}
\end{figure}

\begin{figure}[ht]
    \centering
     \subfigure[Stage 0]{\includegraphics[width=0.19\linewidth]{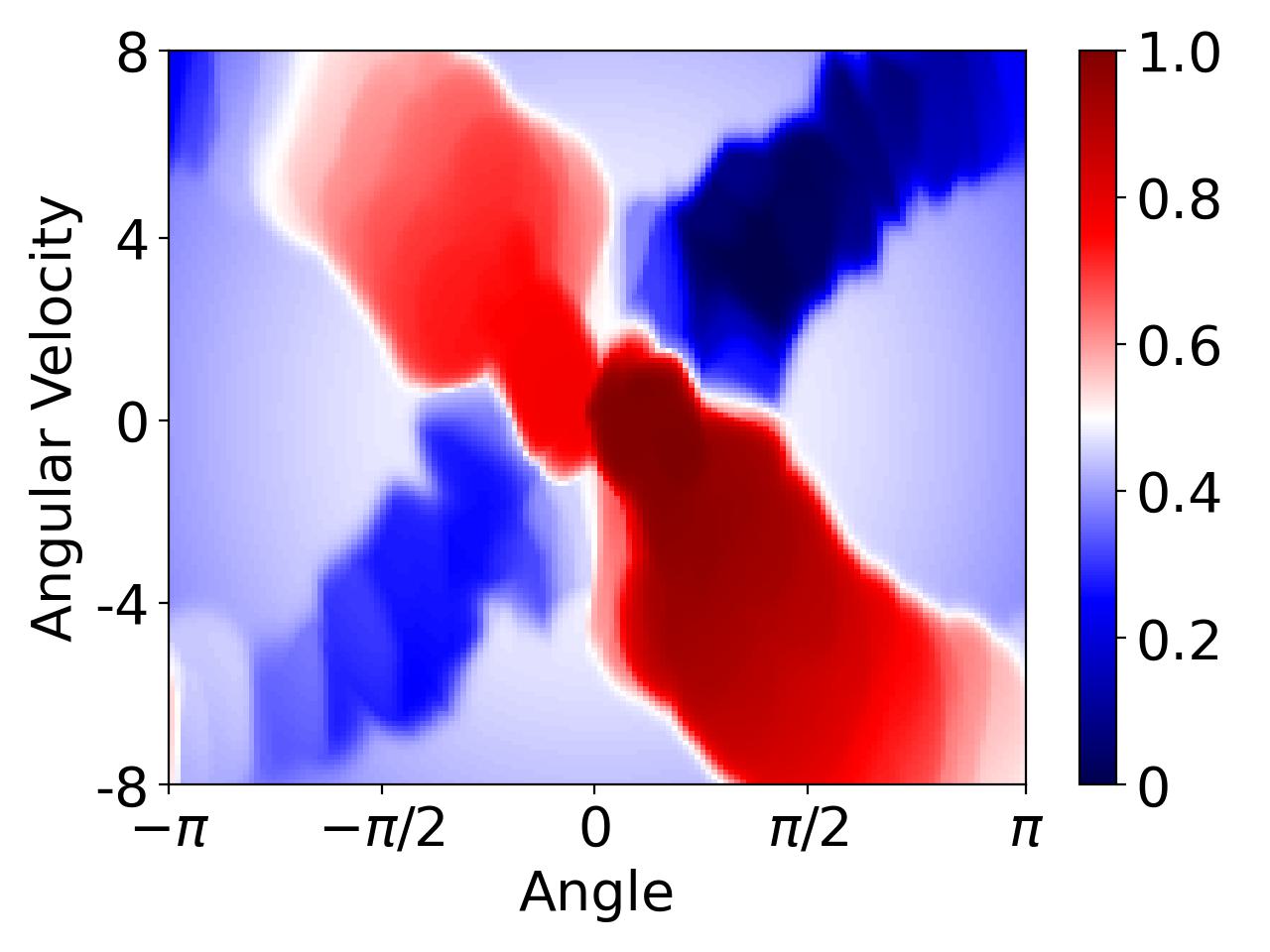}}
     \label{fig.val_fcns}
\subfigure[Stage 7]{\includegraphics[width=0.19\linewidth]{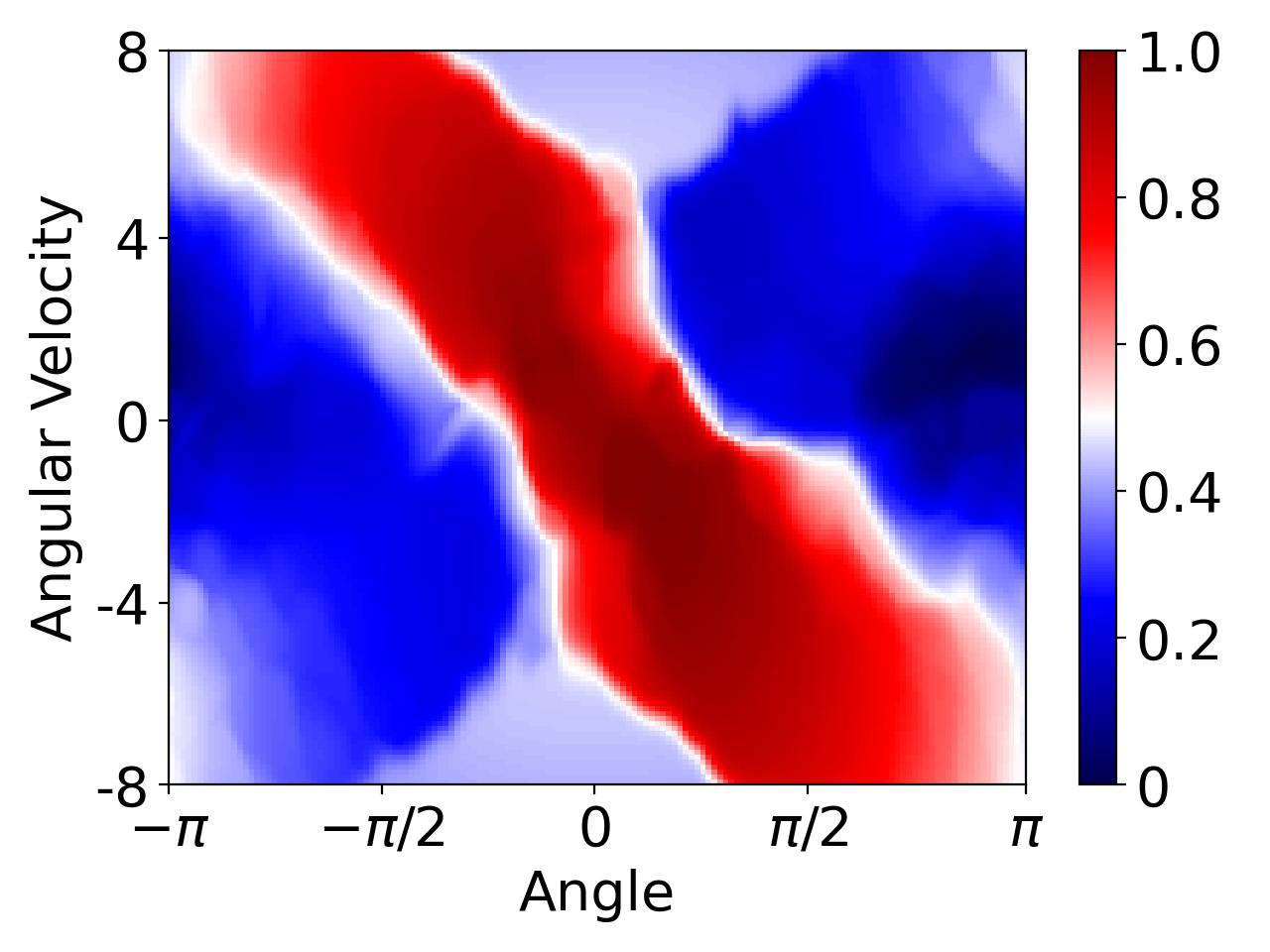}}
\subfigure[Stage 14]{\includegraphics[width=0.19\linewidth]{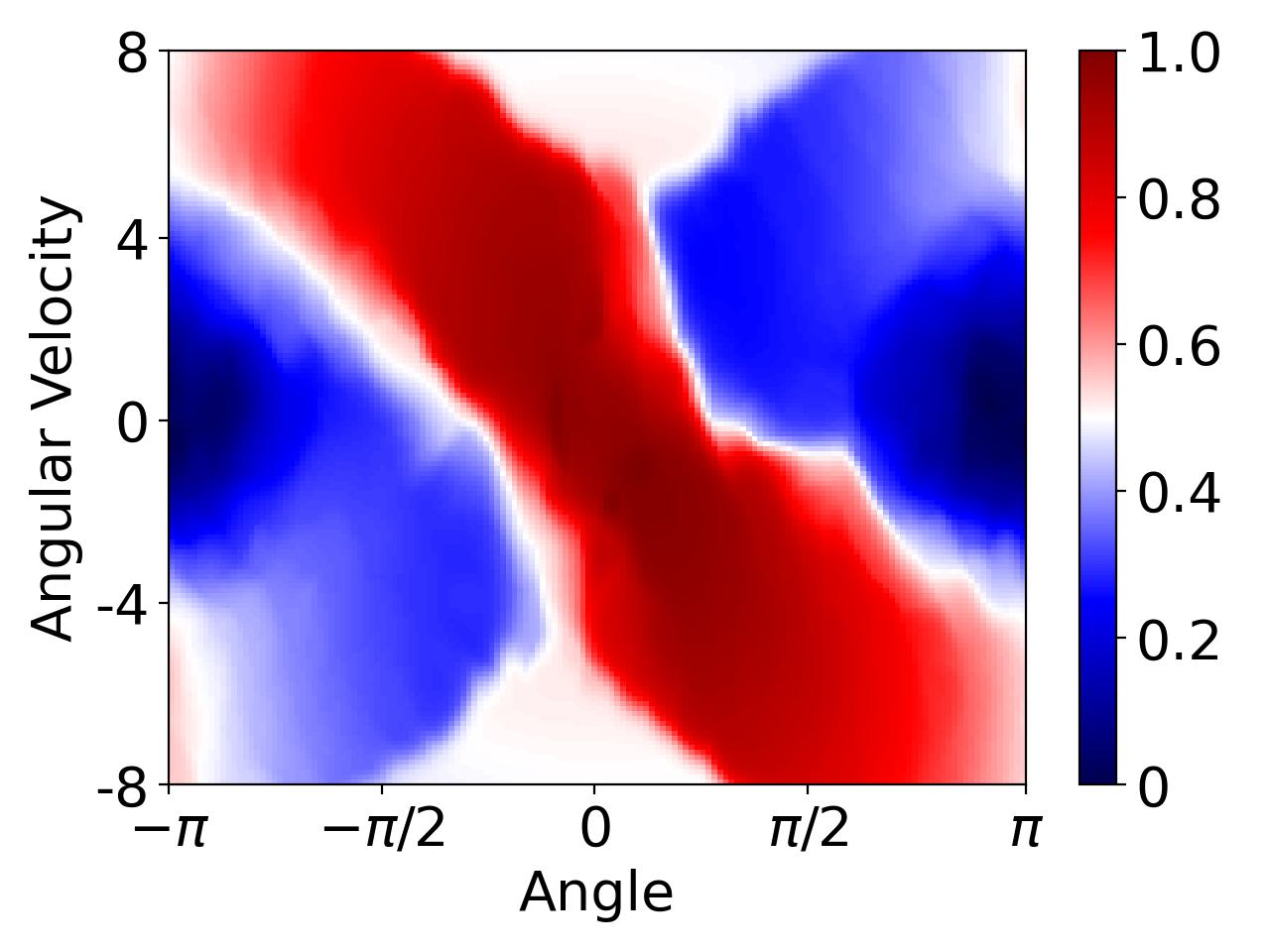}}
\subfigure[Stage 19]{\includegraphics[width=0.19\linewidth]{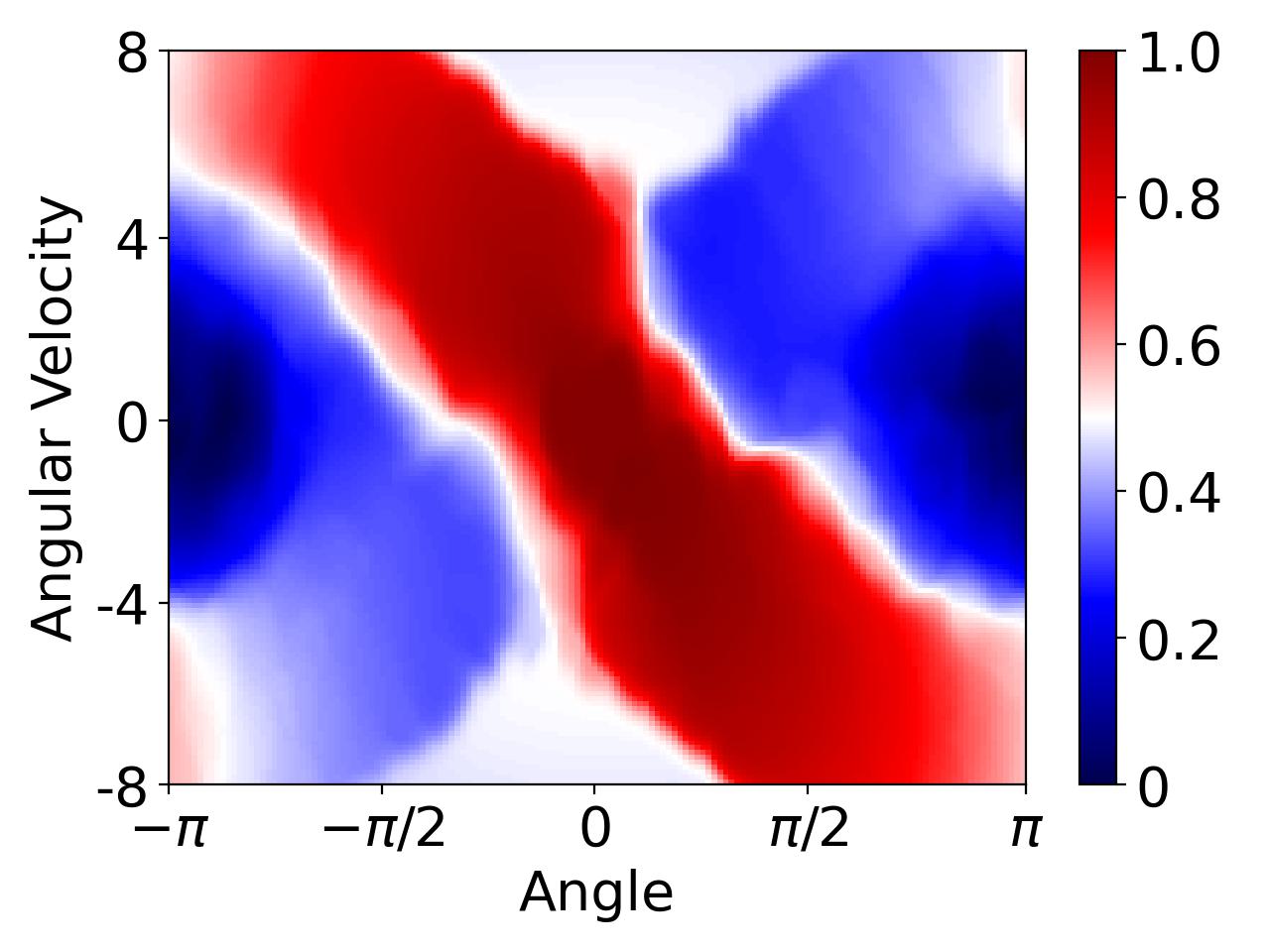}}
\subfigure[DP Reference]{\includegraphics[width=0.19\linewidth]{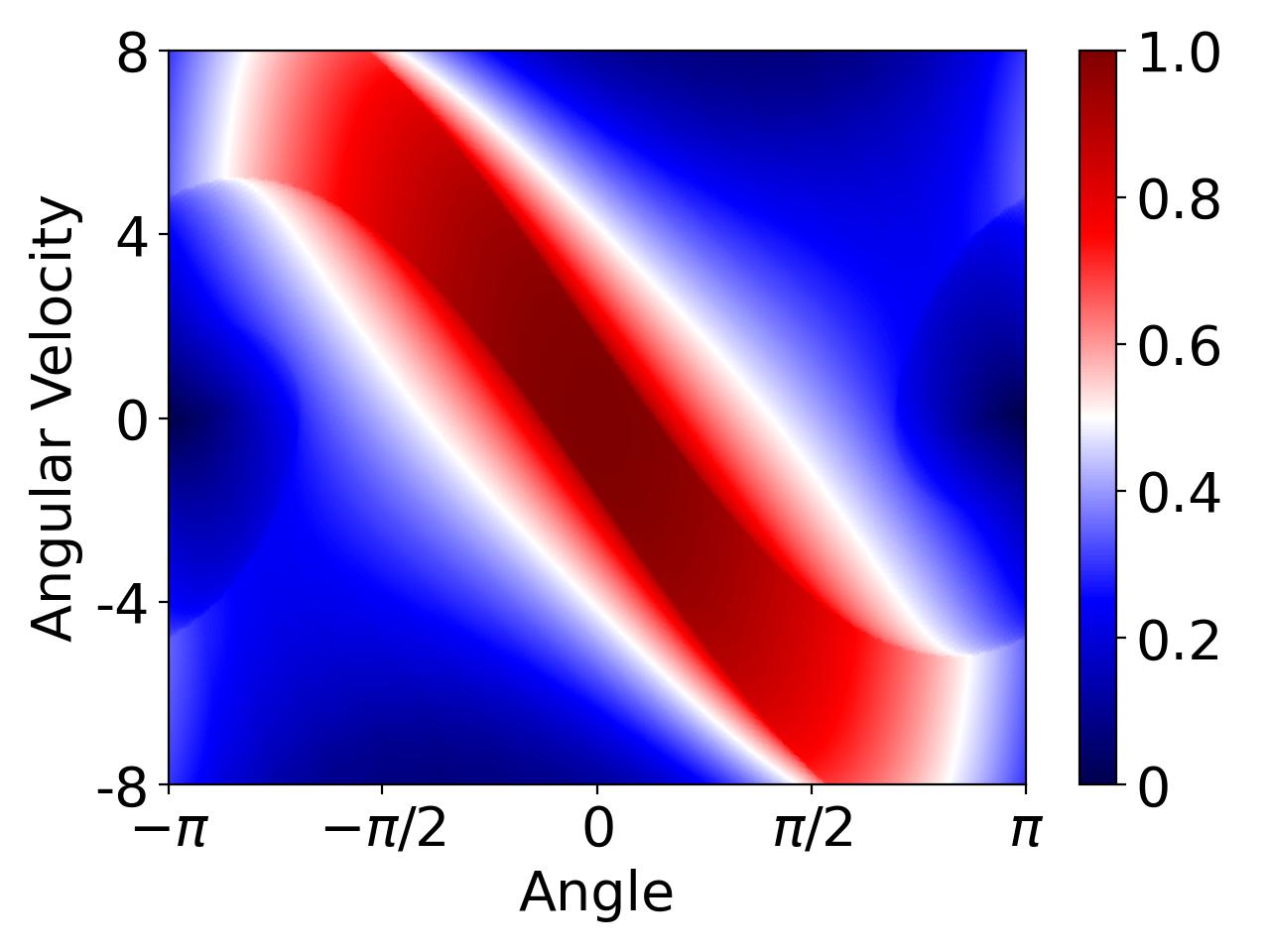}}
\caption{Normalized value functions with an increasing number of iterations. }
    \vspace{-0.1in}
    \label{fig.RL.pendulum}
\end{figure}

%% file: appdx_HS.tex
\section{Tensor Product Hilbert Spaces and Hilbert-Schmidt Operators}
\label{sec.appendix.HS}

This appendix serves as a primer on tensor product Hilbert spaces and Hilbert-Schmidt operators; see \cite[Chapter 12]{aubin2011applied} for a detailed exposition.
Consider two separable real-valued Hilbert spaces $\Hcal_1$ and $\Hcal_2$ defined on separable measurable spaces $\Xset$ and $\Yset$, respectively. Let $\{e_i\}_{i \in \Nset}$ be an orthonormal basis (ONB) of $\Hcal_1$. A bounded linear operator $A: \Hcal_1 \to \Hcal_2$ is a Hilbert-Schmidt (HS) operator if $\sum_{i \in \Nset} \vnorm{Ae_i}_{\Hcal_2}^2 < \infty$. The quantity $\vnorm{A}_{\HS} = \left(\sum_{i \in \Nset} \vnorm{A e_i}_{\Hcal_2}^2\right)^{1/2}$ is the Hilbert-Schmidt norm of $A$ and is independent of the choice of the ONB. 
For two HS operators $A$ and $B$ from $\Hcal_1$ to $\Hcal_2$, their Hilbert–Schmidt inner product is
\begin{align}
 \left\langle A ,B \right \rangle_{\HS(\Hcal_1,\Hcal_2)}  
  = \Tr{(A^*B)}
  = \sum_{i \in \Nset} \left\langle Ae_i ,B e_i \right \rangle_{\Hcal_2}.
  \label{eq.appendix.HS.innerproduct}
\end{align}

For a Hilbert-Schmidt operator $A$ and a bounded linear operator $B$, we have
\begin{align}
&\vnorm{A}_\HS = \Tr{(A^*A)}^{1/2}, \quad
\vnorm{A}_\HS = \vnorm{A^*}_\HS, \quad 
\vnorm{A}_\op \leq \vnorm{A}_\HS,
\label{eq.60}\\
&\vnorm{BA}_\HS \leq \vnorm{B}_\op \vnorm{A}_\HS,
\quad \vnorm{AB}_\HS \leq \vnorm{A}_\HS \vnorm{B}_\op,
\label{eq.61}
\end{align}
where $A^*$ is the adjoint of $A$ and $\vnorm{A}_\op$ is the operator norm of $A$.
Let $f \in \Hcal_1, g \in \Hcal_2$, the tensor product $f \otimes g: \Hcal_2 \to \Hcal_1$ can be viewed as the linear rank-one operator defined by $(f \otimes g)h = \left\langle h,g\right \rangle_{\Hcal_2} f$ for all $h \in \Hcal_2$. 
Thus, for any bounded linear operator $A$ from $\Hcal_1$ to itself, we have,
\begin{align}
A \left(\left(f \otimes g\right) h \right)
= A \left(\left\langle h,g\right \rangle_{\Hcal_2} f \right)
= \left\langle h,g\right \rangle_{\Hcal_2} \left(A f\right)
= \left(\left(Af\right) \otimes g\right) h, \quad f\in \Hcal_1, h \in \Hcal_2.
\label{eq.tensorproduct}
\end{align}
That is,
$A \left(f \otimes g\right) = \left(Af\right) \otimes g$. 
Furthermore, if $\{e_i\}_{i \in \Nset}$ is an orthonormal systems (ONS) of $\Hcal_1$ and $\{e'_j\}_{j \in \Nset}$ is an ONS of $\Hcal_2$, then $ \{e_i\otimes e'_j\}_{i,j \in \Nset}$ is an ONS of $\Hcal_1 \otimes \Hcal_2$.

Now consider $f \in \Hcal_1$, $g \in \Hcal_2$ and $A$ is an HS operator mapping from $\Hcal_2$ to $\Hcal_1$. Let $(e_i)_{i\in \Nset}$ be an orthonormal basis of $\Hcal_2$. Then we have the Fourier series expansion of $g \in \Hcal_2$ as $g = \sum_{i\in\Nset} \left\langle g ,e_i \right \rangle_{\Hcal_2}e_i$.
Therefore, using \eqref{eq.appendix.HS.innerproduct}, we have
\begin{align}
\begin{aligned}
\left\langle f \otimes g ,A \right \rangle_{\HS}    
= \sum_{i \in \Nset} \left\langle \left(f \otimes g\right) e_i , Ae_i \right \rangle_{\Hcal_1} 
=& \sum_{i \in \Nset} \left\langle\left\langle g ,e_i \right \rangle_{\Hcal_2}f , Ae_i \right \rangle_{\Hcal_1} \\
=& \sum_{i \in \Nset} \left\langle g ,e_i \right \rangle_{\Hcal_2} \left\langle f , Ae_i \right \rangle_{\Hcal_1} \\
=& \sum_{i \in \Nset} \left\langle g ,e_i \right \rangle_{\Hcal_2} \left\langle A^*f , e_i \right \rangle_{\Hcal_2} \\
=& \left\langle \left\{\left\langle g ,e_i \right \rangle_{\Hcal_2}\right\}_{i \in \Nset}, \left\{\left\langle A^*f , e_i \right \rangle_{\Hcal_2}\right\}_{i \in \Nset} \right \rangle_{l_2(\Nset)}.
\end{aligned}
\end{align}
Since a separable Hilbert space is isomorphic to $l_2(\Nset)$ \cite[Theorem 1.7.2]{aubin2011applied}, let
$T(g) = \left\{\left\langle g ,e_i \right \rangle_{\Hcal_2}\right\}_{i \in \Nset}$ denote such an isomorphism $T : \Hcal_2 \mapsto l_2(\Nset)$ then we have
\begin{align}
\begin{aligned}
\left\langle f \otimes g ,A \right \rangle_{\HS}    
= \left\langle T(g), T(A^*f) \right \rangle_{l_2(\Nset)}
= \left\langle g, A^*f \right \rangle_{\Hcal_2}
= \left\langle f, Ag \right \rangle_{\Hcal_1}.
\end{aligned}
\label{eq.66}
\end{align}



%% file: proof/interpolation.tex
\section{Learning in Intermediate Spaces}
\label{sec.appendix.interpolation}


By the spectral theorem for self-adjoint compact operators \citep[Theorem V.2.10]{kato2013perturbation}, the integral operator $L_\kappa$ defined in \eqref{eq.L_kappa} enjoys the spectral representation \eqref{eq.L_kappa_spectral} which is convergent in $\Lsq$, and $\Lsq =\ker  L_\kappa \oplus \overline{\text{span} \left(\left[e_i\right],i \in \Iset \right)}$.
We show that $\left(\sigma_i^{1/2} e_i \right)_{i\in \Iset}$ is an ONB of $\left(\ker \emd\right)^\perp$.
Define the adjoint of $\emd$ by $\emd^*: \Lsq  \to \Hcal$.
Since $[e_i]$ is an ONS of $\Lsq$, let $e_i:= \sigma_i^{-1} \emd^* [e_i] \in \Hcal$. We then have for all $i \in \Iset$,
\begin{align}
    \sigma_i e_i = \emd^* [e_i]
    \implies
    \sigma_i \sigma_j \langle e_i, e_j \rangle_\Hcal
   = \langle \emd^* [e_i] , \emd^* [e_j] \rangle_\Hcal
   = \langle [e_i] , \emd \emd^* [e_j] \rangle_{\Lsq}.
\end{align}
Recall that $L_\kappa = \emd \emd^*$ and $ L_\kappa [e_j] = \sigma_j [e_j]$, which then implies
\begin{align}
   \sigma_i \sigma_j \langle e_i, e_j \rangle_\Hcal
   = \langle [e_i] , L_\kappa [e_j] \rangle_{\Lsq}
   = \sigma_j \langle [e_i] ,[e_j] \rangle_{\Lsq}.
\end{align}
The right-hand side of the above relation equals $\sigma_i$, when $j=i$, and is zero otherwise. 
Therefore, $\left(\sigma_i^{1/2} e_i \right)_{i\in \Iset}$ is an ONS in $\Hcal$.
Since $\emd^*[f] = 0$, if $[f] \in \ker L_\kappa$, we have
\begin{align}
\overline{\text{range} \left(\emd^*\right)} 
=\overline{\text{span}}\left\{ \sigma_i^{1/2} e_i, i \in \Iset \right\}.
\end{align}
In addition, from \citep[Theorem 12.10]{rudin1991functional}, since $\emd$ is a bounded operator from $\Hcal$ to $\Lsq$, we also have
$\overline{\text{range} \left(\emd^*\right)} = \left(\ker \emd\right)^\perp$.
Thus, we conclude that $\left(\sigma_i^{1/2} e_i \right)_{i\in \Iset}$ is an ONB of $\left(\ker \emd\right)^\perp$. 

In addition, for any $f\in \Lsq$ and $h\in \Hcal$, we have 
\begin{align}
\left \langle \emd^* [f] , h \right \rangle_\Hcal
=\left \langle [f] , \emd h \right \rangle_{L_2(\rho_X)}
= \int_{\Xset} g(x) \ h(x) \dd \rho_X (x), \quad \forall g\in [f].
\end{align}
Taking $h = \phi(x)$ for $x \in X$ yields 
    $\emd^* [f] = \int_{\Xset} \phi(x) g(x) \dd \rho_X$ , $\forall g \in [f]$.
In addition, since $\left(\phi(x) \otimes \phi(x)\right) \nu 
 = \left \langle \nu, \phi(x) \right \rangle_\Hcal \phi(x)$, for all $\nu \in \Hcal$, we have 
\begin{align}
\begin{aligned}
(\emd^* \emd ) \nu
= \emd^* (\emd  \nu )
= \int_{\Xset} \phi(x) \nu(x) \dd \rho_X (x) 
=& \int_{\Xset} \phi(x) \left \langle \nu , \phi(x) \right \rangle_\Hcal \dd \rho_X (x)\\
=& \int_{\Xset} \left(\phi(x) \otimes \phi(x)  \right ) \nu \dd \rho_X (x) .
\end{aligned}
\end{align}
Hence, the covariance operator $C_{XX}$ defined in Section \ref{sec.embedding} can also be written as $C_{XX} = \emd^* \emd$. 
Since we have shown that $\left(\sigma_i^{1/2}e_i\right)_{i \in \Iset}$ is an ONB of $\left(\ker \emd\right)^\perp$, $\left([{e}_i] \right)_{i \in \Iset}$ an ONB of $\overline{\text{range} \emd}$, and we have the spectral representation of $\kov$ with respect to the ONS $\left(\sigma_i^{1/2} e_i \right)_{i \in \Iset}$ in $\Hcal$. 
\begin{align}
    \kov = \sum_{i \in \Iset} \sigma_i \left \langle \cdot , \sigma_i^{1/2}e_i \right \rangle_\Hcal \sigma_i^{1/2} e_i, \quad
    \Hcal = \ker  \kov \oplus \overline{\text{span} \left(e_i,i \in \Iset \right)}.
\label{eq.kov.decomposition}
\end{align}

Finally, as $L_\kappa$ is a strictly positive operator, following \citep[Theorem 4.6]{steinwart2012mercer}, one can define the fractional power $L_\kappa^r: \Lsq \to \Lsq$ for any $r \in [0,\infty)$ as
    $L_\kappa^r [f] := \sum_{i \in \Iset} \sigma_i^r 
    \left\langle [f],[e_i] \right\rangle_{\rho}[e_i]$ for $ [f] \in \Lsq$.
Likewise, let $\left(\tilde{e}_i\right)_{i\in \Jset}$ be an ONB of $\ker \emd$ 
such that $\left(\sigma_i^{1/2} e_i \right)_{i\in \Iset} \cup \left(\tilde{e}_i\right)_{i\in \Jset}$ is an ONB of $\Hcal$. 
Using this notation, we have the following two spectral representations per \cite{fischer2020sobolev},
\begin{align}
\kov^{\frac{1-\gamma}{2}} =& \sum_{i \in \Iset} \sigma_i^{\frac{1-\gamma}{2}} \left \langle \cdot, \sigma_i^{1/2}e_i \right \rangle_\Hcal \sigma_i^{1/2} e_i,
\quad 0\leq \gamma \leq 1,
\label{eq.covariance.spectral.0} \\
\left(\kov+\lambda \id\right)^{-a}
=& \sum_{i \in \Iset} \left(\sigma_i +\lambda\right)^{-a} 
\left \langle \sigma_i^{1/2}e_i ,  \cdot \right\rangle_\Hcal \sigma_i^{1/2}e_i
+ \lambda^{-a} \sum_{j \in \Jset}  \left \langle \tilde{e}_j ,  \cdot \right\rangle_\Hcal \tilde{e}_j,
\quad a>0.
\label{eq.covariance.spectral}
\end{align}

%% file: proof/thm2.tex
\section{Proof of Theorem \ref{thm.Koopma-CME}}
\label{pf.thm2}

Let $K \in \HS(\Hcal,\left[H\right]^\beta)$. 
Recall that $\left(\left[e_i\right]\right)_{i \in \Iset}$ is an ONB of $\overline{\text{ran} \emd}$ in $\Lsq$. Since $\Hcal$ is separable, $K$ admits the decomposition
\begin{align}
K = \sum_{i\in\Iset}\sum_{j\in\Jset} a_{ij} 
\left( [e_i] \otimes d_j \right),
\end{align}
where $(d_j)_{j \in \Jset}$ is any basis of $\Hcal$. Since $([e_i] \otimes d_j )^* = d_j \otimes [e_i]$, its adjoint $K^*$ has the decomposition
\begin{align}
K^* = \sum_{i\in\Iset}\sum_{j\in\Jset} a_{ij}d_j \otimes [e_i],
\end{align}

By Lemma \ref{lemma.isomorphism}, there exists $\mu_K = \iota(K^*) \in [H_V]^\beta$. 
For any $g \in \Hcal$ and $\Aset \in \sigma(X)$,
\begin{align}
\begin{aligned}
\int_\Aset \left \langle g,  \mu_K  \right \rangle_\Hcal \dd \rho_X
=& \int_\Aset \left \langle g,  \iota \left(K^*\right) \right \rangle_{\Hcal} \dd \rho_X\\
\stackrel{(a)}{=}&  \left \langle g,  \int_\Aset \iota \left(\sum_{i\in\Iset}\sum_{j\in\Jset} a_{ij}d_j \otimes [e_i]\right) \right \rangle_{\Hcal} \dd \rho_X\\
\stackrel{(b)}{=}& \left \langle g,  \int_\Aset 
 \sum_{i\in\Iset}\sum_{j\in\Jset} a_{ij} \obar{e_i}\left(X\right) d_j  \right \rangle_{\Hcal} \dd \rho_X\\
=& \int_\Aset \sum_{i\in\Iset}\sum_{j\in\Jset} a_{ij} \int_\Aset  \left \langle g,d_j \right \rangle_{\Hcal}\obar{e_i} \left(X\right) \dd \rho_X \\
=& \int_\Aset \underbrace{\sum_{i\in\Iset}\sum_{j\in\Jset} a_{ij} \left( [e_i] \otimes d_j \right)}_{=K} g \left(X\right) \dd \rho_X.
\end{aligned}
\end{align}
In (a), we can exchange the order of the Bochner integral with a continuous linear operator per \citep[Theorem 36]{dinculeanu2000vector}. In (b), $\obar{e_i} \in [e_i]$ is arbitrary.

On the other hand, recall that the Koopman operator $K$ satisfies
\begin{align}
\begin{aligned}
\int_\Aset K g(X) \dd \rho_X
=\int_\Aset \E[g(X^+) | X] \dd \rho_X
\stackrel{(a)}{=}&\int_\Aset \E[\langle g, \phi(X^+) \rangle_{\Hcal} |X] \dd \rho_X \\
\stackrel{(b)}{=}&\int_\Aset \langle g, \E[\phi(X^+) |X] \rangle_{\Hcal} \dd \rho_X \\
=&\int_\Aset \langle g, \mu(X) \rangle_{\Hcal} \dd \rho_X.
\end{aligned}
\end{align}
Here, (a) follows from the reproducing property, and (b) holds due to the linearity of conditional expectation.

Hence, for any $g \in \Hcal$ and $\Aset \in \sigma(X)$, we have
\begin{align}
\begin{aligned}
\int_\Aset\left \langle g,  \mu_K  \right \rangle_{\Hcal} \dd \rho_X
 =\int_\Aset \langle g, \mu(X) \rangle_{\Hcal} \dd \rho_X . 
\end{aligned}
\end{align}
Therefore, we conclude $\mu = \mu_K \in [\Hcal_V]^\beta$, and $U = K^* \in \HS(\left[H\right]^\beta,\Hcal)$.

%% file: CME_pf.tex
\section{Properties of $R_\lambda$ and Its Gradient}
\label{sec.appendix.gradients}

\input{proof/lemma4-1}

\subsection{Properties of Operator Gradients}
Consider $(x,x^+) \in \Xset \times \Xset$ and define
\begin{align}
\tilde{C}_{XX}(x)= \phi\left(x\right) \otimes \phi\left(x\right), \quad
\tilde{C}_{X^+X}(x^{+},x)= \phi\left(x^{+}\right) \otimes \phi\left(x\right).
\end{align}
Under Assumption \ref{assumption.polk}, we have
\begin{align}
{\small
\begin{aligned}
\vnorm{\tilde{C}_{XX}(x)}^2_\HS 
&= \left \langle \phi\left(x\right) \otimes \phi\left(x\right), \phi\left(x\right) \otimes \phi\left(x\right) \right\rangle_\HS
= \kappa(x,x) \kappa(x,x)
\leq B_\infty^2 \\
\vnorm{\tilde{C}_{X^+X}(x^{+},x)}^2_\HS 
&= \left \langle \phi\left(x^{+}\right) \otimes \phi\left(x\right), \phi\left(x^{+}\right) \otimes \phi\left(x\right) \right\rangle_\HS
= \kappa(x,x) \kappa(x^{+},x^{+})
\leq B_\infty^2 .
\end{aligned}}
\end{align}

Let $B_\kappa=B_\infty + \lambda$. Then, we have
\begin{align}
\begin{aligned}
\max_{x\in\Xset}\vnorm{\tilde{C}_{XX}(x) + \lambda \id}_\op 
\leq& 
\max_{x\in\Xset}\left(\vnorm{\tilde{C}_{XX}(x)}_\op + \lambda \vnorm{\id}_\op\right) \\
\leq& \max_{x\in\Xset}\left(\vnorm{\tilde{C}_{XX}(x)}_\HS\right) + \lambda 
\leq B_\kappa \\
\max_{(x,x^{+} )\in \Xset \times \Xset} \vnorm{\tilde{C}_{X^+X}(x^{+},x)}_\HS 
&\leq B_\infty
\leq B_\kappa.
\end{aligned}
\label{eq.bound.Covariance}
\end{align}

We have the following properties regarding $\nabla R_\lambda(U)$,$\tilde{\nabla} R_\lambda(x,x^+,U)$ which are needed for the convergence analysis. \footnote{The notation $x^+$ here is merely symbolic, and all results hold for any $x^+ \in \Xset$.}

\begin{lemma} (Properties of gradients)
Under Assumptions \ref{assumption.polk} and \ref{assumption.3}, $\nabla R_\lambda(U)$ and its stochastic approximation $\tilde{\nabla} R_\lambda(x,x^{+};U)$ satisfy the following.
\begin{enumerate} [label=(\alph*)]
    \item (Lipschitz gradient)
    $\nabla R_\lambda(U)$ and $\tilde{\nabla} R_\lambda(x,x^{+};U)$ are Lipschitz continuous with respect to $U$ for all $(x,x^{+} )\in \Xset \times \Xset$, i.e., 
    \begin{align}
        \vnorm{\nabla R_\lambda(U_1) - \nabla R_\lambda(U_2)}_\HS &\leq B_\kappa \vnorm{U_1 - U_2}_\HS, \\
        \vnorm{\tilde{\nabla} R_\lambda(x,x^{+};U_1) - \tilde{\nabla} R_\lambda(x,x^{+};U_2)}_\HS &\leq B_\kappa \vnorm{U_1 - U_2}_\HS,
    \end{align}
    for all $U_1,U_2 \in \HS(\Hcal)$.
    \item (Affine scaling) For $U \in \HS(\Hcal)$,
        $\vnorm{\nabla R_\lambda\left(U\right)}_\HS \leq B_\kappa \left(\vnorm{U}_\HS + 1\right)$, and 
        $\vnorm{\tilde{\nabla} R_\lambda (x,x^{+};U)}_\HS \leq B_\kappa \left(\vnorm{U}_\HS + 1\right)$ for all $(x,x^{+} )\in \Xset \times \Xset$.
    \label{lemma.CME.gradients.affine}
\end{enumerate}
\label{lemma.CME.gradients}
\end{lemma}

\begin{proof}

Notice that
\begin{align}
\begin{aligned}
\vnorm{\tilde{\nabla} R_\lambda(x,x^{+};U_1) - \tilde{\nabla} R_\lambda(x,x^{+};U_2)}_\HS
=\vnorm{\left(U_1-U_2\right) \left(\tilde{C}_{XX}(x)+\lambda\id\right)}_\HS,
\end{aligned}
\end{align}
for all $(x, x^+) \in \Xset \times \Xset$.
Since $\vnorm{AB}_\HS \leq \vnorm{A}_\HS \vnorm{B}_\op$ for any HS operator $A$ and bounded linear operator $B$, we infer
\begin{align}
\begin{aligned}
\vnorm{\tilde{\nabla} R_\lambda(x,x^{+};U_1) - \tilde{\nabla} R_\lambda(x,x^{+};U_2)}_\HS
\leq & \vnorm{U_1-U_2}_\HS \vnorm{\tilde{C}_{XX}(x)+\lambda\id}_\op \\
\leq & B_\kappa\vnorm{U_1-U_2}_\HS,
\end{aligned}
\label{eq.CME.grad.prop.1}
\end{align}
which then yields
\begin{align}
\begin{aligned}
\vnorm{\tilde{\nabla} R_\lambda (x,x^{+};U)}_\HS 
\leq& \vnorm{\tilde{\nabla} R_\lambda (x,x^{+};U) - \tilde{\nabla} R_\lambda \left(x,x^+,0\right)}_\HS + \vnorm{\tilde{\nabla} R_\lambda \left(x,y,0\right)}_\HS \\
\leq& B_\kappa \vnorm{U}_\HS + \vnorm{\tilde{C}_{X^+X}(x^{+},x)}_\HS \\
\leq& B_\kappa \left(\vnorm{U}_\HS + 1\right)
\end{aligned}
\label{eq.CME.bddgrad}
\end{align}
for an HS operator $U$. Furthermore, we deduce that
\begin{align}
\int_{\Xset \times \Xset} \vnorm{\tilde{\nabla} R_\lambda (x,x^{+};U)}_\HS \dd\rho \left(x,x^+\right) \leq B_\kappa \left(\vnorm{U}_\HS + 1\right) < \infty,
\end{align}
i.e.,  $\tilde{\nabla} R_\lambda(x,x^{+};U)$ is Bochner-integrable. 
Therefore, using Jensen's inequality, we have 
\begin{align}
\begin{aligned}
\vnorm{\nabla R_\lambda \left(U\right)}_\HS
=\vnorm{\E\left[\tilde{\nabla} R_\lambda(x,x^{+};U) \right]}_\HS
\leq& \E\left[\vnorm{\tilde{\nabla} R_\lambda (x,x^{+};U)}_\HS\right] \\
\leq& B_\kappa \left(\vnorm{U}_\HS + 1\right).
\end{aligned}
\end{align}
Similarly, using \eqref{eq.CME.grad.prop.1}, we get
\begin{align}
\begin{aligned}
\vnorm{ \nabla R_\lambda (U_1) - \nabla R_\lambda (U_2)}_\HS
=& \vnorm{ \E \left[\tilde{\nabla} R_\lambda(x,x^{+};U_1) - \tilde{\nabla} R_\lambda(x,x^{+};U_2) \right] }_\HS \\
\leq& \E \left[\vnorm{\tilde{\nabla}R_\lambda(x,x^{+};U_1) - \tilde{\nabla}R_\lambda(x,x^{+};U_2)}_\HS\right]\\
\leq&  B_\kappa \vnorm{U_1 - U_2}_\HS.
\end{aligned}
\end{align}
\end{proof}

\section{Proof of Lemma \ref{lemma.CME.U}}
\label{lemma.CME.U.pf}

\input{proof/lemma4}

%% file: proof/lemma4-1.tex
\subsection{Proof of Lemma \ref{lemma.CME.gradient}}
\label{lemma.CME.gradient.pf}

We start by showing $R_\lambda: \HS(\Hcal,\Hcal) \to \Rset$ in \eqref{eq.regression} is differentiable. 
For any $U, U'\in \HS(\Hcal,\Hcal)$, we have
\begin{align}
\begin{aligned}
&\lim_{h\to0}\frac{R_\lambda(U+h U') - R_\lambda(U)}{h} \\
=& \underbrace{\lim_{h\to0} \frac{1}{2h}
\E\left[\vnorm{\phi(x^+) - U \phi(x) - h U' \phi(x) }_{\Hcal}^2 - \vnorm{\phi(x^+) - U \phi(x)}_{\Hcal}^2 \right]}_{T_1} \\
&+ \underbrace{\lim_{h\to0}\frac{\lambda \vnorm{U+h U'}^2_{\HS} - \lambda\vnorm{U}_{\HS}^2}{2h}}_{T_2}.
\end{aligned}
\label{eq.OSDG-exp}
\end{align}
To compute $T_1$ , first notice that
\begin{align}
\begin{aligned}
&\frac{1}{2h}
\E\left[\vnorm{\phi(x^+) - U \phi(x) - h U' \phi(x) }_{\Hcal}^2 - \vnorm{\phi(x^+) - U \phi(x)}_{\Hcal}^2 \right] \\
=& \E\left[\frac{
- 2h \left \langle \phi(x^+) - U \phi(x) , U' \phi(x) \right \rangle_{\Hcal}
+ \vnorm{h U' \phi(x) }_{\Hcal}^2}{2h}\right] \\
=& \E\left[- \left \langle \phi(x^+) - U \phi(x) , U' \phi(x) \right \rangle_{\Hcal}
+ \frac{h}{2} \vnorm{U' \phi(x) }_{\Hcal}^2\right].
\end{aligned}
\end{align}
By Assumption \ref{assumption.polk}, the kernel function is bounded. Since $U,U'$ are Hilbert-Schmidt from $\Hcal$ to $\Hcal$, we can apply the dominated convergence theorem to obtain
\begin{align}
\begin{aligned}
T_1
=&\lim_{h\to0} 
\E\left[- \left \langle \phi(x^+) - U \phi(x) , U' \phi(x) \right \rangle_{\Hcal}
+ \frac{h}{2} \vnorm{U' \phi(x) }_{\Hcal}^2 \right] \\
=& \E\left[ - \left \langle \phi(x^+) - U \phi(x) , U' \phi(x) \right \rangle_{\Hcal} \right],
\end{aligned}
\end{align}
where the last line above can be written as
\begin{align}
\begin{aligned}
T_1
=&  - \E\left[ \left \langle \left( \phi(x^+) - U \phi(x) \right) \otimes \left(\phi(x)\right) , U' \right \rangle_{\HS} \right]\\
=&  - \left \langle \E\left[ \left( \phi(x^+) - U \phi(x) \right) \otimes \left(\phi(x)\right)\right] , U' \right \rangle_{\HS}.
\end{aligned}
\end{align}

Likewise for $T_2$, we have
\begin{align}
\begin{aligned}
T_2 = \frac{\lambda}{2} \lim_{h\to0}\frac{\vnorm{U+h U'}^2_{\HS} - \vnorm{U}_{\HS}^2}{h} 
= \lambda \left \langle U, U' \right \rangle _\HS.
\end{aligned}
\end{align}

Putting together, we conclude that $R_\lambda$ is differentiable, and it gradient $\nabla R_\lambda$ satisfies
\begin{align}
\begin{aligned}
\left \langle \nabla R_\lambda\left(U\right), U' \right \rangle_\HS
=&\lim_{h\to0}\frac{R_\lambda(U+h U') - R_\lambda(U)}{h} \\
=&\left \langle - \E\left[ \left( \phi(x^+) - U \phi(x) \right) \otimes \left(\phi(x)\right)\right]
+ \lambda U, U' \right \rangle_\HS.
\end{aligned}
\end{align}
This implies that the operator gradient of $R_\lambda(U)$ is given by 
\begin{align}
\begin{aligned}
\nabla R_\lambda(U)= - \E\left[ \left( \phi(x^+) - U \phi(x) \right) \otimes \left(\phi(x)\right)\right]
+ \lambda U 
= U C_{XX} -C_{X^+X} + \lambda U.
\end{aligned}
\end{align}
In addition, under Assumption \ref{assumption.polk}, $C_{XX}$ (similarly, $C_{X^+X}$) is Hilbert Schmidt since
\begin{align}
\begin{aligned}
\vnorm{C_{XX}}^2_\HS 
&= \left \langle \E \left[\phi\left(x\right) \otimes \phi\left(x\right)\right], \E\left[\phi\left(x\right) \otimes \phi\left(x\right)\right] \right\rangle_\HS
\leq \kappa(x,x) \kappa(x,x)
\leq B_\infty^2.
\end{aligned}
\end{align}
Hence, we also get that $\nabla R_\lambda(U)\in \HS(\Hcal, \Hcal)$.

We next prove that $R_\lambda$ is strongly convex.
Let $g(U) := R_\lambda(U) - \frac{\lambda}{2}\vnorm{U}_{\HS}^2$. Then for $U_1,U_2  \in \HS(\Hcal,\Hcal)$ and $\alpha \in (0,1]$, we have
\begin{align}
\begin{aligned}
g\left(\alpha U_1 + \left(1-\alpha\right) U_2 \right)
=&\frac{1}{2} \E \left[\vnorm{\phi(x^{+}) - \left(\alpha U_1 + \left(1-\alpha\right) U_2 \right) \phi(x)}_{\Hcal}^2 \right] 
\\
&=\frac{1}{2} \E \left[\vnorm{\alpha \underbrace{\left(\phi(x^{+}) -  U_1 \phi(x) \right)}_{:=T_1} + \left(1-\alpha\right) \underbrace{\left(\phi(x^{+}) - U_2  \phi(x)\right)}_{:=T_2}}_{\Hcal}^2 \right]
\\
&= \frac{1}{2} \E \left[\alpha^2 \vnorm{T_1}_{\Hcal}^2 
+ \left(1-\alpha\right)^2 \vnorm{T_2}_{\Hcal}^2 
+ 2\alpha \left(1-\alpha\right)\left\langle T_1, T_2 \right\rangle_{\Hcal}\right].
\end{aligned}
\end{align}
Furthermore, for $\alpha \in (0,1]$,
\begin{align}
\begin{aligned}
&\alpha g\left(U_1\right) + \left(1-\alpha\right) g \left(U_2 \right)
- g\left(\alpha U_1 + \left(1-\alpha\right) U_2 \right)\\
&= \frac{\alpha}{2} \E \vnorm{T_1}_{\Hcal}^2
+ \frac{1-\alpha}{2}\E \vnorm{T_2}_{\Hcal}^2  
- \frac{1}{2} \E \vnorm{\alpha T_1 + \left(1-\alpha\right) T_2}_{\Hcal}^2 \\
&= \frac{1}{2} \E \left[\alpha \left(1-\alpha\right)  \vnorm{T_1}_{\Hcal}^2
+ \alpha \left(1-\alpha\right)\vnorm{T_2}_{\Hcal}^2
- 2 \alpha \left(1-\alpha\right) \left\langle T_1, T_2 \right\rangle_{\Hcal} \right] \\
&= \frac{1}{2}\alpha \left(1-\alpha\right)  
\E \left[ \vnorm{T_1 - T_2}_{\Hcal}^2\right] 
 \geq 0,
\end{aligned}
\label{eq.R.convex}
\end{align}
implying that $g: \HS(\Hcal,\Hcal) \to \Rset$ 
is a convex functional in the sense of \cite[p. 190]{luenberger1997optimization}. 
Rearranging the terms in \eqref{eq.R.convex}, we obtain
\begin{align}
\begin{aligned}
g\left(U_1\right) - g\left(U_2 \right)
\geq \frac{g\left(U_2 + \alpha\left(U_1 - U_2\right)\right) - g\left(U_2 \right) }{\alpha},
\quad \alpha \in (0,1].
\end{aligned}
\end{align}
Taking $\alpha \to 0$ gives
\begin{align}
\begin{aligned}
g\left(U_1\right) - g\left(U_2 \right)
\geq  \lim_{\alpha \to 0}   \frac{g\left(U_2 + \alpha\left(U_1 - U_2\right) \right) - g\left(U_2 \right) }{\alpha} 
= \left \langle \nabla g\left(U_2\right), U_1 - U_2 \right \rangle_\HS,
\end{aligned}
\label{eq.CME.convex}
\end{align}
where limit exists since both $R_\lambda$ and $\vnorm{\cdot}_\HS^2$ is differentiable. 
Using the definition of $g$, the above relation implies that
\begin{align}
\begin{aligned}
\left[R_\lambda\left(U_1\right) - \frac{\lambda}{2} \vnorm{U_1}_\HS^2 \right] - \left[ R_\lambda\left(U_2 \right) - \frac{\lambda}{2} \vnorm{U_2}_\HS^2 \right]
\geq& \left \langle \nabla R_\lambda\left(U_2\right) - \lambda U_2, U_1 - U_2 \right \rangle_\HS
\end{aligned}
\end{align}
for all $U_1,U_2 \in \HS(\Hcal,\Hcal)$.
Rearranging terms gives
\begin{align}
\begin{aligned}
&R_\lambda\left(U_1\right) - R_\lambda\left(U_2 \right) \\
\geq& \left \langle \nabla R_\lambda\left(U_2\right) , U_1 - U_2 \right \rangle_\HS 
- \lambda \left \langle  U_2, U_1 - U_2 \right \rangle_\HS
+  \frac{\lambda}{2} \vnorm{U_1}_\HS^2 
-  \frac{\lambda}{2} \vnorm{U_2}_\HS^2 \\
=& \left \langle \nabla R_\lambda\left(U_2\right) , U_1 - U_2 \right \rangle_\HS 
- \lambda \left \langle  U_2, U_1 \right \rangle_\HS
+  \frac{\lambda}{2} \vnorm{U_1}_\HS^2 
+  \frac{\lambda}{2} \vnorm{U_2}_\HS^2 \\
=& \left \langle \nabla R_\lambda\left(U_2\right), U_1 - U_2 \right \rangle_\HS 
+  \frac{\lambda}{2} \vnorm{U_1-U_2}_\HS^2.
\end{aligned}
\label{eq.CME.strongly-convex}
\end{align}
That is, $R_\lambda$ is $\lambda$-strongly convex.

We now prove $R_\lambda(\cdot)$ is strong l.s.c. It is known that the norm in a normed space is strong l.s.c., and hence, $\frac{\lambda}{2}\vnorm{U}_\HS^2$ is strong l.s.c. To show
$\E \left[\vnorm{\phi(x^+) - U \phi(x)}_{\Hcal}^2 \right]$ is strong l.s.c., consider $\{U_n\}_{n \in \Nset}$ converging to $U$ in strong operator topology, i.e., $\lim_{n \to \infty}\vnorm{U_n f - U f}_\Hcal = 0$.
We have
\begin{align}
\begin{aligned}
&\E \left[\vnorm{\phi(x^+) - U \phi(x)}_{\Hcal}^2 \right]\\
=& \E \left[\vnorm{\phi(x^+) - U_n \phi(x) + U_n \phi(x)
- U \phi(x)}_{\Hcal}^2 \right] \\
\leq& \E \left[\vnorm{\phi(x^+) - U_n \phi(x)}_{\Hcal}^2 \right]
+ 2 \left| \E \vnorm{\phi(x^+) - U_n \phi(x)}_\Hcal
\vnorm{U_n \phi(x)- U \phi(x)}_\Hcal \right|\\
&+ \E \left[\vnorm{U_n \phi(x) - U \phi(x)}_{\Hcal}^2 \right].
\end{aligned}
\end{align}
Taking $\liminf$ on both sides, the last term goes to 0, and we have
\begin{align}
\begin{aligned}
\E \left[\vnorm{\phi(x^+) - U \phi(x)}_{\Hcal}^2 \right]
\leq 
&\liminf_{n \to \infty} \E \left[\vnorm{\phi(x^+) - U_n \phi(x)}_{\Hcal}^2 \right] \\
&+ 2 \liminf_{n \to \infty} \left| \E \vnorm{\phi(x^+) - U_n \phi(x)}_\Hcal
\vnorm{U_n \phi(x)- U \phi(x)}_\Hcal \right|.
\end{aligned}
\end{align}
Note that since $U_n$ is a bounded operator, we have
\begin{align}
\vnorm{\phi(x^+) - U_n \phi(x)}_\Hcal 
\leq 
\vnorm{\phi(x^+)}_\Hcal +  \vnorm{U_n}\vnorm{\phi(x)}_\Hcal \leq B_\infty \left(1+ \vnorm{U_n}\right)
< \infty.
\end{align}
Then we have 
$\liminf_{n \to \infty} \left| \E \vnorm{\phi(x^+) - U_n \phi(x)}_\Hcal
\vnorm{U_n \phi(x)- U \phi(x)}_\Hcal \right| \to 0$ which follows from the dominated convergence theorem. And we conclude
\begin{align}
\begin{aligned}
\E \left[\vnorm{\phi(x^+) - U \phi(x)}_{\Hcal}^2 \right]
\leq 
\liminf_{n \to \infty} \E \left[\vnorm{\phi(x^+) - U_n \phi(x)}_{\Hcal}^2 \right].
\end{aligned}
\end{align}
That is, $R_\lambda$ is strong l.s.c. 

Finally, since $R_\lambda(U) \to +\infty$ if $\vnorm{U}_\HS \to +\infty$, $R_\lambda(U)$ is coercive. Combining the above results, we have that $R_\lambda: \HS(\Hcal,\Hcal) \to \Rset$ is strong l.s.c, convex, coercive functional. Hence, there exists a unique minimizer. In particular, if $U_\lambda$ minimizes $R_\lambda$, it must be a zero of $\nabla R_\lambda$. That is,  $U_\lambda C_{XX} -C_{X^+X} + \lambda U_\lambda = 0$ which implies $U_\lambda = C_{X^+X} (C_{XX} + \lambda \id )^{-1}$, where $C_{XX} + \lambda \id$ is invertible since it is strictly positive. This completes the proof.

%% file: proof/lemma4.tex
We proceed via induction. Let $U_0 = 0$. After receiving $(x_0,x^{+}_0)$, we update the estimate as 
$U_1 = \eta_0 \tilde{C}_{X^+X}(0)
= \eta_1 \phi(x^{+}_0) \otimes \phi(x_0)$,
proving the base case. 
Next, assume that at the $t$-th iteration  
$U_{t} = \sum_{i=1,j=1}^{t-1} W^{ij}_{t-1} \phi(x^{+}_{i}) \otimes \phi(x_j).$
Then, we have
\begin{align}
\begin{aligned}
U_t \tilde{C}_{XX}(t)
=& \Bigl[\sum_{i=1}^{t-1}\sum_{j=1}^{t-1} W^{ij}_{t-1} \phi(x^{+}_{i}) \otimes \phi(x_j) \Bigr] \ \Bigl[\phi(x_t) \otimes \phi(x_t) \Bigr] \\
\stackrel{(a)}{=}& \sum_{i=1,j=1}^{t-1} W^{ij}_{t-1} \Bigl[\left(\phi(x^{+}_{i}) \otimes \phi(x_j) \right) \phi(x_t) \Bigr] \otimes \phi(x_t) \\
\stackrel{(b)}{=}& \sum_{i=1,j=1}^{t-1} W^{ij}_{t-1} \Bigl( \bigl \langle \phi(x_t), \phi(x_j) \bigr \rangle_\Hcal \ \phi(x^{+}_{i}) \Bigr) \otimes \phi(x_t) \\
=&  \sum_{i=1,j=1}^{t-1} W^{ij}_{t-1}  \kappa_X(x_j,x_t) \ \left[\phi(x^{+}_{i}) \otimes \phi(x_t)\right],
\end{aligned}
\end{align}
where (a) follows from $A \left(f \otimes g\right) = \left(Af\right) \otimes g$ for any bounded linear operator $A$, and (b) follows from the definition of tensor products. Substituting the above relation into \eqref{eq.CME.SFGD} for $t+1$ gives
\begin{align}
\begin{aligned}
U_{t+1}
=&(1  -  \lambda \eta_t  ) U_t -   \eta_t  \left(U_t \tilde{C}_{XX}(t) - \tilde{C}_{X^+X}(t)\right) \\
=&(1 -\lambda \eta_t) \sum_{i=1,j=1}^{t-1} W_{t-1}^{ij} \phi(x^{+}_{i}) \otimes \phi(x_j) \\
&- \eta_t \sum_{i=1,j=1}^{t-1} W^{ij}_{t-1}  \kappa_X(x_j,x_t) \phi(x^{+}_{i}) \otimes \phi(x_t) 
+ \eta_t \phi(x^{+}_t) \otimes \phi(x_t) \\
=&\sum_{i=1,j=1}^t W^{ij}_{t} \phi(x^{+}_{i}) \otimes \phi(x_j)\\ 
=& \Psi_{X^+,t} W_t \Phi_{X,t}^\top,
\end{aligned}
\end{align}
where the $(i,j)$-th element of $W_{t}$ is given by \eqref{eq.U.w}.

%% file: appdx_CMP.tex
\section{Implementing Algorithm \ref{algorithm.main}}
\label{appdx.CMP}

Algorithm \ref{algorithm.main} describes updates for infinite-dimensional operators. However, it can be efficiently implemented using finite-dimensional Gram matrices, as we describe next. After receiving new samples $(x_{t+1},x^{+}_{t+1})$, let $\Phi_{X,t+1}$ (similarly, $\Psi_{X^+,t+1}$) be the feature matrices constructed from  $\left\{\phi(x_i)\right\}_{i \in \Ical_t}$ ($\left\{\phi(x_i^+)\right\}_{i \in \Ical_t}$), and
$\tilde{\Phi}_{X,t+1}=\left[\Phi_{X,t+1}, \phi(x_{t+1})\right]$,
$\tilde{\Psi}_{X^+,t+1}=\left[\Psi_{X^+,t+1}, \phi(x^{+}_{t+1}) \right]$.
Define Gram matrices $G_{X,t+1} = \Phi_{X,t+1}^\top \Phi_{X,t+1}$, 
$G_{X^+,t+1} = \Psi_{X^+,t+1}^\top \Psi_{X^+,t+1}$,
$\tilde{G}_{X,t+1} =\tilde{\Phi}_{X,t+1}^\top \tilde{\Phi}_{X,t+1}$,
$\quad \tilde{G}_{Y,t+1} = \tilde{\Psi}_{X^+,t+1}^\top \tilde{\Psi}_{X^+,t+1}$,
$\bar{G}_{X,t+1} = \tilde{\Phi}_{X,t+1}^\top \Phi_{X,t+1}$, and
$\bar{G}_{X^+,t+1} = \tilde{\Psi}_{X^+,t+1}^\top \Psi_{X^+,t+1}$.

In the rest of this derivation, we omit the index $t+1$ in the notation for simplicity.
Then we can write the left-hand side of the condition \eqref{eq.CME.stop} in terms of the decision variable $Z \in \Rset^{|\Ical_t|\times |\Ical_t|}$ as
\begin{align}
\begin{aligned}
\ell(Z):=
&\vnorm{\sum_{i\in \Ical_t} \sum_{j \in \Ical_t} Z^{ij} \phi(x^{+}_i) \otimes \phi(x_j) -\sum_{i \in \tilde{\Ical}_{t+1}}\sum_{j \in \tilde{\Ical}_{t+1}} \tilde{W}^{ij} \phi(x^{+}_i) \otimes \phi(x_j)}_\HS^2 \\
=& \vnorm{\Psi_{X^+} Z \Phi_{X}^\top
- \tilde{\Psi}_Y \tilde{W} \tilde{\Phi}_X^\top }_\HS^2 \\
\stackrel{(a)}{=}& \Tr\left(\Phi_{X} Z^\top \Psi_{X^+}^\top \Psi_{X^+} Z \Phi_{X}^\top \right)
- 2 \Tr \left( \tilde{\Phi}_X \tilde{W}^\top \tilde{\Psi}_Y^\top \Psi_{X^+} Z \Phi_{X}^\top \right)\\
&+ \Tr \left(\tilde{\Phi}_X \tilde{W}^\top \tilde{\Psi}_Y^\top \tilde{\Psi}_Y \tilde{W} \tilde{\Phi}_X^\top
\right) \\
=& \Tr\left(\Phi_{X} Z^\top G_{X^+} Z \Phi_{X}^\top 
- 2 \tilde{\Phi}_X \tilde{W}^\top \bar{G}_{X^+} Z \Phi_{X}^\top 
+ \tilde{\Phi}_X \tilde{W}^\top \tilde{G}_{X^+} \tilde{W} \tilde{\Phi}_X^\top
\right),
\end{aligned}
\end{align}
where line (a) follows from $\left \langle A, B \right \rangle_\HS = \Tr(A^\top B)$ for two HS operators $A,B$, $\vnorm{A}_{\HS}^2 = \Tr(A^* A)$, and $\Tr(AB) = \Tr(BA)$.
Notice $\ell(Z)$ is a convex quadratic function in $Z$ that attains its minimum at $Z_\star = G_{X^+}^{-1} \bar{G}_{X^+}^\top \tilde{W} \bar{G}_X G_X^{-1}$ with
\begin{align}
\ell(Z_\star) 
= \Tr \left[  \tilde{W}^\top \left( \tilde{G}_{X^+} - \bar{G}_{X^+}  G_{X^+}^{-1}  \bar{G}_{X^+}^\top\right) \tilde{W} \tilde{G}_X \right],
\end{align}
where Assumption \ref{assumption.mixing} precludes the possibility of the process being periodic, and thus our dataset has no repeated samples, and $G_{X^+}$ is invertible. 
Hence, the condition \eqref{eq.CME.stop} reduces to check whether $\ell(Z_\star) \leq \ve_t$.
The coefficient matrix can be computed as 
$W = Z_\star 
= G_{X^+}^{-1} \bar{G}_{X^+}^\top \tilde{W} \bar{G}_X G_X^{-1}$.
Moreover, to speed up computation, at each time $t\in\Tset$, the inversion of Gram matrix $G_{X^+,t}^{-1}$ can be recursively computed based on $G_{X^+,t-1}^{-1}$ using the Woodbury matrix identity \citep{horn1994topics}. 

\subsection{Details Regarding the Experiment in Section \ref{sec.4.2}}
\label{sec.4well.details}
We approximate $K$ and its leading eigenfunctions following the procedure introduced in Section \ref{sec.example}. The steaming data consists of samples on $[-2,2]\times[-2,2]$ which are collected from $400$ trajectories with $100$ evolutions along each with sampling interval $\tau = 0.1s$. 
We choose the kernel function $\kappa \left(x_1, x_2\right)
= 0.4 \times \exp(-\|x_1-x_2\|^2_2/(2 \times 0.4^2)) + 0.6 \times \exp(-\|x_1-x_2\|^2_2/(2 \times 0.7^2))$. 
We use a constant stepsize with $\eta = 0.3$, and the budget is set as $\ve = \eta^4$.
After computing the eigenfunction, we leverage k-means clustering techniques to locate metastable sets which are shown in Figure \ref{fig.4well.2000},\ref{fig.4well.20000}, and \ref{fig.4well}. 

%% file: 5-proof.tex
\section{Proof of Results in Section \ref{sec.5.analysis}}
\label{appendix.convergence.proof}

We begin by establishing a few supporting lemmas that will be useful later.  
\begin{lemma} (Uniform boundedness)
Let Assumptions \ref{assumption.polk} and \ref{assumption.3} hold. 
If $\eta_t < 1/\lambda$ for $t \in \Tset$, then $U_\lambda$ and the iterates $\{U_t\}_{t\in \Tset}$ generated by Algorithm \ref{algorithm.main} are uniformly bounded as
\begin{align}
 \vnorm{U_t}_\HS \leq \frac{B_\infty}{\lambda},\quad
 \vnorm{U_\lambda}_\HS \leq \frac{B_\infty}{\lambda},\quad
 \forall t \in \Tset. 
\label{eq.CME.bound}
\end{align}
\label{lemma.CME.bound}
\vspace{-0.2in}
\end{lemma}

\begin{proof}
First, notice that once the dictionary $\Dcal_t$ and the coefficient matrix $W_t$ have been updated, \eqref{eq.CME.U.update} can be written as 
    $U_{t+1} = \Pi_{\Dcal_t}[\tilde{U}_{t+1}]$ for $t \in \Tset$.
We establish \eqref{eq.CME.bound} by induction.  At time $t=1$, we have
\begin{align}
\begin{aligned}
\vnorm{U_1}_\HS
= \vnorm{\Pi_{\Dcal_0} \left[U_1\right]}_\HS  
 \stackrel{(a)}{\leq} \vnorm{U_1}_\HS 
 =\vnorm{\eta_0 \tilde{C}_{X^+X}(0)}_\HS
 \leq \eta_0 B_\infty 
 \stackrel{(b)}{\leq}& B_\infty/\lambda,
\end{aligned}
\end{align}
where (a) follows from the non-expansive property of the projection operator onto the Hilbert space $\HS(\Hcal,\Hcal)$, and (b) follows from the fact that $\eta_t<1/\lambda$. Thus, the base case for induction holds. Now, assume that $\vnorm{U_k}_\HS \leq B_\infty/\lambda$ for $k  =1,\ldots,t$. Then, at time $t+1$, using the non-expansive property of the projection operator again, we have
\begin{align}
\begin{aligned}
\vnorm{U_{t+1}}_\HS 
= \vnorm{\Pi_{\Dcal_t} \left[\tilde{U}_{t+1}\right]}_\HS  
\leq& \vnorm{\tilde{U}_{t+1}}_\HS.
\end{aligned}
\end{align}
We then expand $\tilde{U}_{t+1}$ using \eqref{eq.CME.SFGD} and we have
\begin{align}
\begin{aligned}
\vnorm{U_{t+1}}_\HS 
=&\vnorm{ \left( \id - \lambda \eta_t \right) U_t
-\eta_t U_t \tilde{C}_{XX}(t) 
+ \eta_t \tilde{C}_{X^+X}(t)}_\HS \\
=&\vnorm{U_t \left( \id - \eta_t \left( \lambda \id + \tilde{C}_{XX}(t)\right) \right) + \eta_t \tilde{C}_{X^+X}(t)}_\HS \\
\leq&\vnorm{U_t}_\HS \vnorm{ \id - \eta_t \left( \lambda \id + \tilde{C}_{XX}(t)\right)}_\op + \eta_t \vnorm{\tilde{C}_{X^+X}(t)}_\HS,
\end{aligned}
\end{align}
where the last line holds due to the relation $\vnorm{AB}_\HS \leq \vnorm{A}_\HS \vnorm{B}_\op$. Furthermore,  the operator norm of a self-adjoint operator coincides with its maximum eigenvalue, and hence, with $\tilde{C}_{XX}(t)$ is self-adjoint. Denote $\kappa_{x_t}:=\kappa_X(x_t,\cdot)$ and we have
\begin{align}
\begin{aligned}
\vnorm{\id - \eta_t\left(\kappa_{x_t} \otimes \kappa_{x_t} + \lambda \id \right)}_\op
=& \sigma_{\text{max}} \left(\left( \id - \eta_t \left(\kappa_{x_t} \otimes \kappa_{x_t} + \lambda \id \right) \right) \right) \\
\leq& 1- \eta_t \sigma_{\text{min}}\left(\kappa_{x_t} \otimes \kappa_{x_t} + \lambda I \right) \\
\leq& 1- \eta_t \lambda.
\end{aligned}
\label{eq.107}
\end{align}
Hence, we conclude
\begin{align}
\begin{aligned}
\vnorm{U_{t+1}}_\HS 
\leq \vnorm{U_t}_\HS \left(1 - \eta_t \lambda\right) + \eta_t \vnorm{\tilde{C}_{X^+X}(t)}_\HS 
\leq \frac{B_\infty}{\lambda} \left(1 - \eta_t \lambda\right) + \eta_t B_\infty
= \frac{B_\infty}{\lambda}.
\end{aligned}
\end{align}
In addition, $U_\lambda$  satisfies 
\begin{align}
\begin{aligned}
\vnorm{U_\lambda}_\HS  
= \vnorm{C_{X^+X}\left(C_{XX}+\lambda \id \right)^{-1}}_\HS 
=&\vnorm{\left(C_{XX}+\lambda \id \right)^{-1}C_{X^+X}^*}_\HS \\
\stackrel{(a)}{\leq}&  \vnorm{\left(C_{XX}+\lambda \id \right)^{-1}}_\op \vnorm{C_{X^+X}^*}_\HS  \\
\stackrel{(b)}{\leq}& \frac{\vnorm{C_{X^+X}}_\HS}{\lambda} \\
\stackrel{(c)}{\leq}& \frac{B_\infty}{\lambda},
\end{aligned} 
\end{align}
where (a) follows from the fact that $\vnorm{BA}_\HS \leq \vnorm{B}_\op \vnorm{A}_\HS$ for an HS $A$ and bounded linear operator $B$, (b) holds since $\vnorm{\left(C_{XX}+\lambda \id \right)^{-1}}_\op \leq 1/\lambda$ and (c) follows from 
$\vnorm{C_{X^+X}}_\HS \leq B_\infty$.
\end{proof}

We present the following lemma, which characterizes the difference between two iterates via the sum of stepsizes and the norm of an iterate, and will be useful later. A similar result for stochastic approximation in finite-dimensional Euclidean space appeared in \cite{srikant2019finite} and \cite{chen2022finite}. Here, we consider stochastic recursion in the space of HS operators, which is infinite-dimensional, and make use of properties of operator-valued gradients presented in Lemma \ref{lemma.CME.gradients}. 

\begin{lemma} 
Let Assumptions \ref{assumption.polk} and \ref{assumption.3} hold. For $s<r$, denote $ \eta_{s,r-1} := \sum_{k=s}^{r-1} \eta_k$ and assume $\eta_{s,r-1} \leq 1/4 B$, for some $B>0$. 
Then:
   \begin{enumerate} [label=(\alph*)]
   \item $\vnorm{U_{s} -U_{r}}_\HS \leq 2 B \eta_{s,r-1} \left( \vnorm{U_s}_\HS +1 \right) $,
   \item $\vnorm{U_{s} -U_{r}}_\HS \leq 4 B \eta_{s,r-1} \left( \vnorm{U_r}_\HS +1 \right) $.
   \end{enumerate}
\label{lemma.CME.FT.fdiff}
\end{lemma}
\begin{proof}
By Lemma \ref{lemma.CME.gradients},  the stochastic operator gradient scales affinely with respect to the current iterates. We leverage this property to provide a bound for $\vnorm{U_{t+1}}_\HS$ in terms of $\vnorm{U_{t}}_\HS$, and repeatedly apply this results to bound $U_s - U_r$.
Let $t \in [s,r]$, and we have
\begin{align}
\begin{aligned}
\vnorm{U_{t+1} - U_{t}}_\HS 
=&  \eta_t  \vnorm{-\tilde{\nabla} R_\lambda\left(x_t,x^+_t,U_{t}\right) + \frac{E_{t}}{\eta_t}}_\HS \\
\leq&  \eta_t  \vnorm{-\tilde{\nabla} R_\lambda \left(x_t,x^+_t,U_{t}\right)}_\HS + \vnorm{E_{t}}_\HS.
\end{aligned}
\end{align}

Recall that under Assumption \ref{assumption.3}(b), there exists some $B_\ve>0$ such that for all $t \in \Tset$, the sparsification budget satisfies
$\ve_t \leq \bcmp \eta_t^2 \leq B_\ve \eta_t \left(\vnorm{U_{t}}_\HS+1 \right) $.   
Together with Lemma \ref{lemma.CME.gradients} \ref{lemma.CME.gradients.affine} and condition $\vnorm{E_t}_\HS\leq \ve_t$, we have 
\begin{align}
\begin{aligned}
\vnorm{U_{t+1} - U_{t}}_\HS 
\leq & \eta_t  B_\kappa \left(\vnorm{U_{t}}_\HS+1 \right) + \ve_t \\
\leq & \eta_t  B_\kappa \left(\vnorm{U_{t}}_\HS+1 \right) + B_\ve \eta_t \left(\vnorm{U_{t}}_\HS+1 \right)\\
=& \eta_t B \left(\vnorm{U_{t}}_\HS+1 \right).
\end{aligned}
\end{align}
The triangle inequality gives
\begin{align}
\begin{aligned}
\vnorm{U_{t+1}}_\HS 
\leq \vnorm{U_{t}}_\HS + \vnorm{U_{t+1} - U_{t}}_\HS  
\leq  \left(\eta_t B +1 \right)\vnorm{U_{t}}_\HS + \eta_t B.
\end{aligned}
\end{align}
As a result, the iterates $U_{t+1}$ scales affinely  as
$\vnorm{U_{t+1}}_\HS +1   
\leq  \left(\eta_t B +1 \right)\left(\vnorm{U_{t}}_\HS + 1\right)$.
By recursively applying the above inequality, we have
\begin{align}
\vnorm{U_{t}}_\HS +1 
\leq \Pi_{i=s}^{t-1} \left(\eta_i B +1 \right) \left(\vnorm{U_{s}}_\HS +1\right).
\end{align}
Using $1+x \leq e^x $ for $x \in \Rset$, 
we then obtain
\begin{align}
\begin{aligned}
\vnorm{U_{t}}_\HS +1 
\leq&  \exp(B \eta_{s,t-1}) \left(\vnorm{U_{s}}_\HS+1\right) \\
\leq&  \underbrace{\exp(B \eta_{s,r-1})}_{< 2} \left(\vnorm{U_{s}}_\HS+1\right) 
\leq 2 \left(\vnorm{U_{s}}_\HS+1\right).
\end{aligned}
\end{align}
Thus, we obtain the first claim as
\begin{align}
\begin{aligned}
\vnorm{U_{r} - U_{s}}_\HS 
\leq \sum_{t=s}^{r-1} \vnorm{U_{t+1} - U_{t}}_\HS 
\leq& 2 B \sum_{t=s}^{r-1} \eta_t  \left(\vnorm{U_{s}}_\HS+1\right) \\
=& 2 B \eta_{s,r-1} \left(\vnorm{U_{s}}_\HS+1\right).
\end{aligned}
\end{align}
Since $\vnorm{U_{s}}_\HS \leq \vnorm{U_{r}}_\HS+ \vnorm{U_{r}-U_s}_\HS$, the above relation also yields
\begin{align}
\begin{aligned}
\vnorm{U_{r} - U_{s}}_\HS 
&\leq 2 B \eta_{s,r-1} \left(\vnorm{U_{r}}_\HS+ \vnorm{U_{r}-U_s}_\HS+1\right)
\\
&\leq 2 B \eta_{s,r-1} \left(\vnorm{U_{r}}_\HS + 1\right)
+
\frac{1}{2} \vnorm{U_{r}-U_s}_\HS,
\end{aligned}
\end{align}
rearranging which gives the second claim, completing the proof of the lemma.
\end{proof}

%% file: proof/lemma5-1.tex
\subsection{Proof of Lemma \ref{lemma.Kt}}
\label{pf.lemma.Kt}
We start with the first term (sampling error) in \eqref{eq.K.bias-variance}.
By the isomorphism in Lemma \ref{lemma.isomorphism}, we have 
\begin{align}
\vnorm{\left[K_t - K_\lambda\right]}_{\HS(\Hcal,\left[\Hcal\right]^\beta)}
= \vnorm{\left[U_t - U_\lambda\right]}_{\HS(\left[\Hcal\right]^\beta, \Hcal)}
\label{eq.Kt.mu_t}
\end{align}

We first introduce the following lemma that provides an upper bound for the $\gamma$-norm for elements in $\Hcal_V$ in terms of the $\HS$-norm of an element in $\HS(\Hcal,\Hcal)$. 
Recall that $\iota_\kappa$ is the linear isomorphism from $\HS(\Hcal,\Hcal)$ to $\Hcal_V$ in Lemma \ref{lemma.isomorphism}. 
\begin{lemma} (Bounding the $\beta$-norm)
For any $\beta \in (0,1)$ and $A \in \HS(\Hcal,\Hcal)$, we have
\begin{align}
\begin{aligned}
\vnorm{\left[A \right]}_{\HS(\left[\Hcal\right]^\beta ,\Hcal)}^2 
\leq \vnorm{A}_\HS^2 B_\infty^{1-\beta}.
\end{aligned}
\end{align}
\label{lemma.CME.gamma2HS.2}
\end{lemma}
\begin{proof}
By \cite[Lemma 2]{li2022optimal}, for any $A \in \HS(\Hcal,\Hcal)$,  we have
\begin{align}
    \vnorm{\left[A \right]}_{\HS(\left[\Hcal\right]^\beta ,\Hcal)}^2 \leq \vnorm{A C_{XX}^{\frac{1-\beta}{2}}}_\HS.
\label{eq.li.lemma2}
\end{align}

Recall that for two HS operator $A_1$ and $A_2$, we have $\vnorm{A_1 A_2}_\HS \leq \vnorm{A_1}_\HS \vnorm{A_2}_\op$ and $\vnorm{A_2}_\op \leq \vnorm{A_2}_\HS$, we get
\begin{align}
\begin{aligned}
\vnorm{A C_{XX}^{\frac{1-\beta }{2}}}_\HS^2 
\leq \vnorm{A}_\HS^2 \vnorm{C_{XX}^{\frac{1-\beta }{2}}}_\op^2
\leq \vnorm{A}_\HS^2 \vnorm{C_{XX}^{\frac{1-\beta }{2}}}_\HS^2
\leq \vnorm{A}_\HS^2 B_\infty^{1-\beta }.
\end{aligned}
\end{align}    
where the last line holds since $\vnorm{C_{XX}}_\HS \leq B_\infty$. Plugging this into \eqref{eq.li.lemma2} completes the proof.
\end{proof}

We can now relate the norm of the intermediate space to the $\HS$-norm as
\begin{align}
\vnorm{\left[U_t - U_\lambda\right]}_{\HS(\left[\Hcal\right]^\beta ,\Hcal)}^2 
\leq \vnorm{U_t - U_\lambda}_\HS^2 B_\infty^{1-\beta}.
\label{eq.lemma5.1}
\end{align}
To bound the bias term in \eqref{eq.K.bias-variance}, applying \cite[Lemma 1]{li2022optimal} gives 
$\vnorm{[U_\lambda] - U}_{\HS(\left[\Hcal\right]^\beta , \Hcal)}
\leq \vnorm{U}_{\HS(\left[\Hcal\right]^\beta, \Hcal)}$.
As a consequence, under Assumption \ref{assumption.src}, we have
\begin{align}
\vnorm{[K_\lambda] - K}_{\HS(\Hcal, \left[\Hcal\right]^\beta)}^2
=\vnorm{[U_\lambda] - \CMEO}_{\HS(\left[\Hcal\right]^\beta , \Hcal)}^2
\leq \Bsrc^2.
\label{eq.lemma5.2}
\end{align}
Combining \eqref{eq.Kt.mu_t}, \eqref{eq.lemma5.1} and \eqref{eq.lemma5.2} completes the proof.

%% file: proof/lemma5-3.tex
\subsection{Proof of Theorem \ref{corollary.Koopman.asy}}
\label{appendix.CME.asy}

Recall from Lemma \ref{lemma.Kt}, we have
\begin{align}
\begin{aligned}
\vnorm{\left[K_t\right] - K}_{\HS(\Hcal,{[H]^\beta})}^2
\leq & 2\lambda^{-\left(\gamma+1\right)} B_\kappa^2 \vnorm{U_t - U_\lambda}_\HS^2
+  2 \Bsrc^2.
\end{aligned}
\end{align}
In the sequel, we characterize the convergence behavior of $\vnorm{U_t - U_\lambda}_\HS$. 
To prove the result, we construct an almost super-martingale sequence and leverage the almost supermartingale convergence theorem \citep{robbins1971convergence} to show that the sequence converges to some limit almost surely. Finally, we use the fact that the step-size sequence is nonsummable to prove the claim.

\emph{(Step 1)} Using recursion \eqref{eq.CME.fixedpt.SA}, for $t \in \Nset$, we have
\begin{align}
\begin{aligned}
\vnorm{U_{t+1} - U_\lambda}_\HS^2
=& \vnorm{U_{t}+ \eta_t \left(-\tilde{\nabla} R_\lambda\left(x_t,x^{+}_t;U_{t}\right)+\frac{E_t}{\eta_t} \right)- U_\lambda}_\HS^2\\
=& \vnorm{U_{t}- U_\lambda}_\HS^2
- 2 \eta_t \left \langle U_t - U_\lambda, \tilde{\nabla} R_\lambda\left(x_t,x^{+}_t;U_{t} \right) \right \rangle_\HS \\
&+ 2 \eta_t \left \langle U_t - U_\lambda, \frac{E_t}{\eta_t} \right \rangle_\HS
+ \eta_t^2 \vnorm{-\tilde{\nabla} R_\lambda\left(x_t,x^{+}_t;U_{t}\right)+\frac{E_t}{\eta_t}}_\HS^2\\
\leq & \vnorm{U_{t}- U_\lambda}_\HS^2
- 2 \eta_t \left \langle U_t - U_\lambda, \tilde{\nabla} R_\lambda\left(x_t,x^{+}_t;U_{t} \right) \right \rangle_\HS \\
&+ 2 \eta_t  \vnorm{U_t - U_\lambda}_\HS 
\vnorm{\frac{E_t}{\eta_t}}_\HS
+ 2 \eta_t^2 \vnorm{-\tilde{\nabla} R_\lambda\left(x_t,x^{+}_t;U_{t}\right)}_\HS^2\\
&+2 \eta_t^2 \vnorm{\frac{E_t}{\eta_t}}_\HS^2,
\end{aligned}
\end{align}
where the last line follows from Cauchy-Schwartz and $\vnorm{A+B}_\HS^2 \leq 2 \vnorm{A}_\HS^2 + 2 \vnorm{B}_\HS^2$ for $A,B \in \HS(\Hcal,\Hcal)$.
Since $\vnorm{E_t}\leq \ve_t$, we have
\begin{align}
\begin{aligned}
\vnorm{U_{t+1} - U_\lambda}_\HS^2
\leq & \vnorm{U_{t}- U_\lambda}_\HS^2
- 2 \eta_t \left \langle U_t - U_\lambda, \tilde{\nabla} R_\lambda\left(x_t,x^{+}_t;U_{t} \right) \right \rangle_\HS \\
&+ 2 \ve_t \vnorm{U_t - U_\lambda}_\HS
+ 2\eta_t^2 \vnorm{-\tilde{\nabla} R_\lambda\left(x_t,x^{+}_t;U_{t}\right)}_\HS^2
+ 2  \ve_t^2.
\end{aligned}
\end{align}
Taking conditional expectation with respect to $\Fcal_{t}$ gives
\begin{align}
\begin{aligned}
\E\left[\vnorm{U_{t+1} - U_\lambda}_\HS^2|\Fcal_t \right]
\leq & \vnorm{U_{t}- U_\lambda}_\HS^2
\underbrace{- 2 \eta_t \left \langle U_t - U_\lambda, \E\left[\tilde{\nabla} R_\lambda\left(x_t,x^{+}_t;U_{t} \right)|\Fcal_t \right] \right \rangle_\HS}_{:=T} \\
&+ 2 \ve_t \vnorm{U_t - U_\lambda}_\HS
+ 2\eta_t^2 \underbrace{\E\left[\vnorm{-\tilde{\nabla} R_\lambda\left(x_t,x^{+}_t;U_{t}\right)}_\HS^2|\Fcal_t \right]}_{\leq b_t^2 \text{ from } \eqref{assumption.asy.a.b.c}}\\
&+ 2  \ve_t^2.
\end{aligned}
\label{eq.asy.main}
\end{align}

To further bound the above equation, we next study the term $T$ as follows.
\begin{align}
\begin{aligned}
T
=&- 2 \eta_t \left \langle U_t - U_\lambda, \E\left[\tilde{\nabla} R_\lambda\left(x_t,x^{+}_t;U_{t} \right)|\Fcal_t \right] \right \rangle_\HS \\
=& - 2 \eta_t \left \langle U_t - U_\lambda, 
\nabla R_\lambda\left(U_{t} \right) \right \rangle_\HS \\
&+ 2\eta_t \left \langle U_t - U_\lambda, \nabla R_\lambda\left(U_{t} \right) - \E\left[\tilde{\nabla} R_\lambda\left(x_t,x^{+}_t;U_{t} \right)|\Fcal_t \right] \right \rangle_\HS \\
\leq &- 2 \eta_t \left \langle U_t - U_\lambda, 
\nabla R_\lambda\left(U_{t} \right) \right \rangle_\HS \\
&+ 2\eta_t \vnorm{U_t - U_\lambda}_\HS
\vnorm{\nabla R_\lambda\left(U_{t} \right) - \E\left[\tilde{\nabla} R_\lambda\left(x_t,x^{+}_t;U_{t} \right)|\Fcal_t \right]}_\HS
\\
\leq & - 2 \eta_t \left( R_\lambda\left(U_t\right) - R_\lambda\left(U_\lambda\right) \right) \\
&+ 2\eta_t \vnorm{U_t - U_\lambda}_\HS
\vnorm{\nabla R_\lambda\left(U_{t} \right) - \E\left[\tilde{\nabla} R_\lambda\left(x_t,x^{+}_t;U_{t} \right)|\Fcal_t \right]}_\HS
\\
\leq & -2 \eta_t \left( R_\lambda\left(U_t\right) - R_\lambda\left(U_\lambda\right) \right) 
+ 2\eta_t a_t \vnorm{U_t - U_\lambda}_\HS,
\end{aligned}  
\end{align}
where we have used Cauchy-Schwartz inequality, convexity of $R_\lambda$ from Lemma \ref{lemma.CME.gradients}, and our assumption in \eqref{assumption.asy.a.b.c}.
Substituting the above result into \eqref{eq.asy.main}, we have
\begin{align}
\begin{aligned}
\E\left[\vnorm{U_{t+1} - U_\lambda}_\HS^2|\Fcal_t \right]
\leq & \vnorm{U_{t}- U_\lambda}_\HS^2
-2 \eta_t \left( R_\lambda\left(U_t\right) - R_\lambda\left(U_\lambda\right) \right) \\
&+ 2\eta_t a_t \vnorm{U_t - U_\lambda}_\HS 
+ 2 \ve_t \vnorm{U_t - U_\lambda}_\HS
+ 2\eta_t^2 b_t^2
+ 2  \ve_t^2.
\end{aligned}
\end{align}
Since $\vnorm{U_t - U_\lambda}_\HS \leq \frac{1}{2}\left(1+\vnorm{U_t - U_\lambda}_\HS^2\right)$,  we have
\begin{align}
\begin{aligned}
\E\left[\vnorm{U_{t+1} - U_\lambda}_\HS^2|\Fcal_t \right]
=&\left( 1 + \eta_t a_t + \ve_t \right) \vnorm{U_{t}- U_\lambda}_\HS^2
-2 \eta_t \left( R_\lambda\left(U_t\right) - R_\lambda\left(U_\lambda\right) \right)
\\
&+ 2\eta_t^2 b_t^2
+ 2 \ve_t^2
+ \eta_t a_t + \ve_t.
\end{aligned}
\label{eq.asy.martingale}
\end{align}

\emph{(Step 2)} Notice that \eqref{eq.asy.martingale} suggests that $\vnorm{U_{t}- U_\lambda}_\HS^2$ is an almost supermartignale sequence. Thus, we can use the almost supermartingale convergence result \citep{robbins1971convergence}, which is stated as follows.\\

\textbf{Theorem} (Almost Supermartingales Convergence Theorem \citep{robbins1971convergence})
Let $m_t,p_t,q_t,s_t$ be $\Fcal_t$-measurable finite nonnegative random variable with filtration $\left\{\Fcal\right\}_{t \in \Nset}$. If $\sum_{t \in \Nset} p_t <\infty$, $\sum_{t \in \Nset} q_t <\infty$, and 
\begin{align*}    
\E [m_{t+1}|\Fcal_t ] \leq m_t (1+p_t) + q_t - s_t,
\end{align*}
almost surely. Then $\lim_{t \to \infty} m_t$ exists and is finite and $\sum_{t \in \Nset} s_t < \infty$ almost surely.\\

To apply the Almost Supermartingales Convergence Theorem, 
note that under assumptions in Theorem \ref{corollary.Koopman.asy}, we have 
\begin{align}
\begin{aligned}
&\sum_{t \in \Nset} \left( \eta_t a_t + \ve_t \right)
\leq \sum_{t \in \Nset} \left( \eta_t a_t + \bcmp \eta_t^2 \right)
< \infty,\\
&\sum_{t \in \Nset} \left( 2\eta_t^2 b_t^2
+ 2 \ve_t^2 + \eta_t a_t + \ve_t\right) < \infty.
\end{aligned}
\end{align}

Define
$m_t :=  \vnorm{U_{t}- U_\lambda}_\HS^2$, 
$p_t := \eta_t a_t + \ve_t$, 
$q_t := 2\eta_t^2 b_t^2 + 2 \ve_t^2 + \eta_t a_t + \ve_t$, and
$s_t := 2 \eta_t \left( R_\lambda\left(U_t\right) \right)$.
By the Almost Supermartingales Convergence Theorem, $\vnorm{U_{t}- U_\lambda}^2_\HS$ converges to some nonnegative random variable almost surely and
\begin{align}
\sum_{t \in \Nset} \eta_t \left( R_\lambda\left(U_t\right) - R_\lambda\left(U_\lambda\right) \right) < \infty, \quad 
\rho-\text{a.s.}
\end{align}
As $\sum_{t \in \Nset} \eta_t = \infty$, we have
\begin{align}
\liminf_{t \to \infty} R_\lambda\left(U_t\right) = R_\lambda\left(U_\lambda\right), \quad 
\rho-\text{a.s.}
\end{align}

Since $\left\{\vnorm{U_t - U_\lambda}^2_\HS\right\}$ converges almost surely, let $\vnorm{U_t - U_\lambda}^2_\HS \to \xi,$ for some $\xi \geq 0$. We next show $\xi = 0$. As $\{U_t\}_{t \in \Nset}$ is a bounded sequence, let $\{U_{tl}\}_{l=0}^\infty$ be a bounded subsequence of $\{U_t\}_{t \in \Nset}$ along which the $\liminf$ is reached, i.e.,
\begin{align}
\lim_{l \to \infty} R_\lambda\left(U_{tl}\right) 
= \liminf_{t \to \infty} R_\lambda\left(U_t\right) 
= R_\lambda\left(U_\lambda\right).
\label{eq.asy.subsequence}
\end{align}
By the Banach-Alaoglu theorem, there exists a weakly convergent subsequence of $\{U_{tl}\}_{l=0}^\infty$ converging to some $U^\circ$. By Lemma \ref{lemma.CME.gradient}, $R_\lambda$ is weak l.s.c. Together with \eqref{eq.asy.subsequence}, we have that the value of $R_\lambda$ evaluated at the weak limit $U^\circ$ satisfies $R_\lambda\left(U^\circ\right) = R_\lambda\left(U_\lambda\right)$. Also from Lemma \ref{lemma.CME.gradient},  $U_\lambda$ is the unique minimizer of $R_\lambda$. Thus, we conclude $U^\circ = U_\lambda$ and $\vnorm{U_t - U_\lambda}^2_\HS$ converges to $0$ over said subsequence, implying
\begin{align}
\lim_{t \to \infty} \vnorm{U_t - U_\lambda}_\HS^2 = 0, \quad \rho-\text{a.s.}
\end{align}

The rest of the proof follows from substituting the above result into \eqref{eq.mu_t.gamma}.



%% file: proof/lemma5-4.tex
\subsection{Proof of Lemma \ref{lemma.CME.mixing}}
\label{lemma.CME.mixing.pf}

Let $\rho = \pi P'$ where $P'$ is defined in \eqref{eq:P'.def}.
Consider $Y_i := (X_i, X_i^+)$ for $i \in \Nset$, where such successive pairs are formed from one-step transitions of the process. We define a Markov transition kernel $P_Y$ on $\Xset \times \Xset$ that governs the evolution of $Y_i$ to $Y_{i+1}$. Specifically, given $Y_i = (x_i,x_{i+1})$, the next state $Y_{i+1} = \left(x_i',{x_{i+1}}'\right)$ is generated be first sampling
$x_i' \sim P(\cdot \mid x_{i+1})$, and then ${x_{i+1}}' \sim P(\cdot \mid x_i')$.
Formally, for any measurable set $B \in \Bcal_X \times \Bcal_X$, and any $(x_i,x_{i+1}) \in \Xset \times \Xset$, we have
\begin{align}
P_Y\left(B \mid \left(x_i,{x_{i+1}} \right) \right) 
:= \int_{\Xset} \int_{\Xset} \bone_{(x_i',{x_{i+1}}') \in B} P(\dd {x_{i+1}}' \mid x_i') P (\dd x_i' \mid {x_{i+1}}).
\end{align}
Thus, $P_Y$ defines a Markov process on $\Xset \times \Xset$ with transition from $\left(x_i,x_{i+1}\right)$ to $\left(x_i',{x_{i+1}}'\right)$.

Using $P_Y$ defined above, if $\left\{X_t\right\}_{t \in \Tset}$ is uniformly geometrically ergodic with a unique stationary distribution $\pi$, after $\tau(\delta)$ steps, we have
\begin{align}
\sup_{x\in\Xset \times \Xset}\vnorm{{P}_Y(x,\cdot)-\pi P'}_{TV} \leq
\sup_{x\in \Xset}\vnorm{P^{s}(x,\cdot)-\pi}_{TV}
\leq 
R\gamma^s\quad,
\label{eq.mixing.contractive}
\end{align}
where the first equality follows from the fact that TV distance is contractive with respect to the kernels.
That is, we have $\sup_{x\in\Xset \times \Xset}\vnorm{{P}_Y(x,\cdot)-\pi P'}_{TV} \leq \delta$.

We are now ready to prove the result. For $t \geq \tau(\delta)$, we can write the Bochner conditional expectation as Bochner integral w.r.t $P_Y^{t+s}\left(\cdot|\Fcal_s \right)$, $\rho(\cdot)$ as
\begin{align}
\begin{aligned}
&\vnorm{\E \left[\tilde{\nabla} R_\lambda \left(x_{t+s},x^{+}_{t+s};U\right)|\Fcal_s\right] - \nabla R_\lambda \left(U\right)}_\HS \\
=& \vnorm{\int \tilde{\nabla} R_\lambda \left(x_{t+s},x^{+}_{t+s};U\right) dP_Y^{t+s}\left(x,x^{+}|\Fcal_s \right) - \int \tilde{\nabla} R_\lambda \left(x,x^{+};U\right)  \dd \rho\left( x,x^{+}\right)}_\HS.
\end{aligned}
\end{align}
To bound the above term, for any $A\in \HS(\Hcal,\Hcal)$ with $\vnorm{A}_{\HS} \leq 1$, define the scalar-valued function 
\begin{align}
f_A(x,x^{+}) = \left \langle A, \tilde{\nabla} R_\lambda \left(x,x^{+};U\right)  \right \rangle_{\HS}.   
\label{eq.f_A}
\end{align}
From the affine scaling property in Lemma \ref{lemma.CME.gradients}, we have 
$\left|f_A(x,x^{+})\right| \leq B_\kappa \left(\vnorm{U}_\HS+1\right)$. Hence, we can apply the total variation formula \cite[Proposition 4.5]{levin2017markov} to obtain
\begin{align}
\begin{aligned}
\left|\int f_A \dd P_Y^{t+s}\left(x,x^{+}|\Fcal_s \right)
 - \int f_A \dd \rho \right|
\leq 2 B_\kappa \left(\vnorm{U}_\HS+1\right) \vnorm{P_Y^{t+s}\left(\cdot|\Fcal_s \right) - \rho(\cdot)}_\TV.
\end{aligned}
\end{align}
Plugging back \eqref{eq.f_A} and taking supremum over $\vnorm{A}_\HS \leq 1$, we have
\begin{align}
\begin{aligned}
&\vnorm{\int \tilde{\nabla} R_\lambda \left(x_{t+s},x^{+}_{t+s};U\right) dP_Y^{t+s}\left(x,x^{+}|\Fcal_s \right) - \int \tilde{\nabla} R_\lambda \left(x,x^{+};U\right)  \dd \rho\left( x,x^{+}\right)}_\HS \\
\leq & 2 B_\kappa \left(\vnorm{U}_\HS+1\right) \vnorm{P_Y^{t+s}\left(\cdot|\Fcal_s \right) - \rho(\cdot)}_\TV \\
\leq &2 B_\kappa \delta \left(\vnorm{U}_\HS+1\right),
\end{aligned}
\end{align}
for any $s \in \Tset$ and $t \geq \tau(\delta)$, where the last line holds by \eqref{eq.mixing.contractive}.

%% file: CME_FT_pf.tex
\subsection{Proof of Lemma \ref{thm.CME.FT.one-step}}
\label{thm.CME.FT.one-step.pf}
Since $\vnorm{U(t) - U_\lambda}_{\HS}^2=\left \langle U_t - U_\lambda, U_t - U_\lambda\right \rangle_\HS$, for $t \geq \tau_{t}$,  we have,
\begin{align}
\begin{aligned}
&\E \left[ \vnorm{U_{t+1}-U_\lambda}_\HS^2 | \Fcal_{t - \tau_t} \right]
-\E \left[ \vnorm{U_{t}-U_\lambda}_\HS^2 | \Fcal_{t - \tau_t} \right] \\
&=\E \left[   
\vnorm{\left(U_{t+1} -U_t \right) 
+ \left(U_t  -U_\lambda \right)}_\HS^2
-\vnorm{U_{t}-U_\lambda}_\HS^2
| \Fcal_{t - \tau_t} \right] \\
&=\E \left[   
2 \left \langle U_{t+1} - U_t, U_{t} - U_\lambda\right \rangle_\HS 
+ \vnorm{U_{t+1} - U_{t}}_\HS^2| \Fcal_{t - \tau_t}\right]. 
\end{aligned}
\end{align}

Expanding $U_{t+1}-U_t$ using recursion \eqref{eq.CME.fixedpt.SA}, we have
\begin{align}
\begin{aligned}
&\E \left[ \vnorm{U_{t+1}-U_\lambda}_\HS^2 | \Fcal_{t - \tau_t} \right]
-\E \left[ \vnorm{U_{t}-U_\lambda}_\HS^2 | \Fcal_{t - \tau_t} \right] \\
&= 2\E \left[\left \langle U_{t} - U_\lambda, U_{t+1} - U_{t}\right \rangle_\HS| \Fcal_{t - \tau_t} \right] + \E \left[ \vnorm{U_{t+1} - U_{t}}_\HS^2 | \Fcal_{t - \tau_t} \right] \\
&=2\E \left[ \left \langle U_{t} - U_\lambda,  \eta_t \left(-\tilde{\nabla} R_\lambda\left(x_t,x^{+}_t,U_{t}\right)+\frac{E_t}{\eta_t} \right)\right \rangle_\HS| \Fcal_{t - \tau_t} \right]
\\
& \quad + \E \left[ \vnorm{\eta_t \left(-\tilde{\nabla} R_\lambda\left(x_t,x^{+}_t,U_{t}\right)+\frac{E_t}{\eta_t} \right)}_\HS^2 | \Fcal_{t - \tau_t} \right] \\
&=2 \eta_t \underbrace{\E \left[ \left \langle U_{t} - U_\lambda, -\nabla R_\lambda\left(U_{t}\right) \right \rangle_\HS| \Fcal_{t - \tau_t} \right]}_{:=T_1}
+ 2 \eta_t \underbrace{ \E \left[ \left \langle U_{t} - U_\lambda, \frac{E_t}{\eta_t} \right \rangle_\HS| \Fcal_{t - \tau_t} \right]}_{:=T_2} \\
& \quad + 2 \eta_t \underbrace{\E \left[ \left \langle U_{t} - U_\lambda, -\tilde{\nabla} R_\lambda\left(x_t,x^{+}_t,U_{t}\right) + \nabla R_\lambda\left(U_{t}\right) \right \rangle_\HS| \Fcal_{t - \tau_t} \right]}_{:=T_3}\\
& \quad + \eta_t^2 \underbrace{\E \left[ \vnorm{-\tilde{\nabla} R_\lambda\left(x_t,x^{+}_t,U_{t}\right)+\frac{E_t}{\eta_t}}_\HS^2 | \Fcal_{t - \tau_t} \right]}_{:=T_4}.
\end{aligned}
\label{eq.pf.one-step-main}
\end{align}

In the above decomposition, $T_1$ corresponds to the negative drift. This term can be bounded by the strong convexity established in Lemma \ref{lemma.CME.gradient}.
$T_2$ follows from the error due to compression and depends on a proper choice of sparsification budget $\{\ve_t\}_{t \in \Tset}$. 
$T_3$ is a consequence of Markovian sampling, and if we were to collect IID samples, $T_3$ equals zero. Thanks to Lemma \ref{lemma.CME.mixing}, $T_3$ can be bounded by invoking the mixing property. Lastly, $T_4$ collects the error due to the discretization of ODE and compression. It can be controlled under a proper choice of stepsizes and compression budget. 
The proof will seek to analyze a discretized version of the continuous-time dynamics $\dot{U}(t)  = -\nabla R_\lambda\left(U\left(t\right)\right)$ for $U \in \HS(\Hcal)$.
We next provide an upper bound for each term above in four steps with the final step combining these four results.

\emph{(Step 1)} 
Recall from Lemma \ref{lemma.CME.gradient} that $R_\lambda$ is strongly convex. Continue from \eqref{eq.CME.strongly-convex} in the proof of Lemma \ref{lemma.CME.gradients}, we have for $U_1,U_2 \in \HS(\Hcal,\Hcal)$,
\begin{align}
\begin{aligned}
R_\lambda\left(U_1\right) - R_\lambda\left(U_2 \right) 
\geq& \left \langle \nabla R_\lambda\left(U_2\right), U_1 - U_2 \right \rangle_\HS 
+  \frac{\lambda}{2} \vnorm{U_1-U_2}^2_\HS, \\
R_\lambda\left(U_2\right) - R_\lambda\left(U_1 \right) 
\geq& \left \langle \nabla R_\lambda\left(U_1\right), U_2 - U_1 \right \rangle_\HS 
+  \frac{\lambda}{2} \vnorm{U_1-U_2}^2_\HS.
\end{aligned}
\end{align}
Adding the above two relations, we get
\begin{align}
\begin{aligned}
0 \geq& \left \langle \nabla R_\lambda\left(U_2\right)- \nabla R_\lambda\left(U_1\right), U_1 - U_2 \right \rangle_\HS 
+  \lambda\vnorm{U_1-U_2}_\HS^2.
\end{aligned}
\end{align}
Setting $U_1 = U$, $U_2 = U_\lambda$ for which $\nabla R_\lambda\left(U_\lambda\right)=0$, we have
\begin{align}
\left \langle -\nabla R_\lambda(U), U - U_\lambda \right \rangle_\HS \leq -\lambda \vnorm{U - U_\lambda}_\HS^2.
\label{lemma.CME.gradients.convex}
\end{align}
A bound on $T_1$ then follows as
\begin{align}
T_1 \leq -2  \lambda \E \left[ \vnorm{U_{t} - U_\lambda}_\HS^2| \Fcal_{t - \tau_t} \right].
\label{eq.CME.T1}
\end{align}

\emph{(Step 2)}  To bound $T_2$, recall that $\vnorm{E_t}_\HS \leq \ve_t$, and we deduce
\begin{align}
\begin{aligned}
T_2 =   \E \left[ \left \langle U_{t} - U_\lambda, \frac{E_t}{\eta_t}\right \rangle_\HS| \Fcal_{t - \tau_t} \right] 
\leq& \frac{1}{\eta_t} \E \left[ \vnorm{U_{t} - U_\lambda}_\HS \vnorm{ E_t}_\HS| \Fcal_{t - \tau_t} \right]\\
\leq& \frac{1}{\eta_t} \ve_t \E \left[ \vnorm{U_{t} - U_\lambda}_\HS| \Fcal_{t - \tau_t} \right].
\end{aligned}
\end{align}

To further bound $\vnorm{U_{t} - U_\lambda}_\HS$, we use triangle inequality and Lemma \ref{lemma.CME.bound} to obtain
\begin{align}
\begin{aligned}
\vnorm{U_{t} - U_\lambda}_\HS
\leq \vnorm{U_{t}}_\HS + \vnorm{U_\lambda}_\HS
=  2 B_\infty/\lambda 
\implies T_2 \leq  \frac{2 \ve_t B_\infty}{\eta_t \lambda}.
\end{aligned}
\end{align}

\emph{(Step 3)} To bound $T_3$, we invoke the mixing property, and we rearrange $T_3$ as
\begin{align}
\begin{aligned}
T_3
=&\E \left[ \left \langle U_{t} - U_\lambda, -\tilde{\nabla} R_\lambda\left(x_t,x^{+}_t,U_{t}\right) + \nabla R_\lambda\left(U_{t}\right) \right \rangle_\HS| \Fcal_{t - \tau_t} \right] \\
=& \underbrace{\E \left[ \left \langle U_{t} - U_{t - \tau_t}, -\tilde{\nabla} R_\lambda\left(x_t,x^{+}_t,U_{t}\right) + \nabla R_\lambda\left(U_{t}\right) \right \rangle_\HS| \Fcal_{t - \tau_t} \right]}_{:=T_{3,1}} \\
&+ \underbrace{\E \left[ \left \langle U_{t - \tau_t} - U_\lambda, -\tilde{\nabla} R_\lambda \left(x_t,x^+_t,U_{t - \tau_t}\right) +\nabla R_\lambda\left(U_{t - \tau_t}\right) \right \rangle_\HS| \Fcal_{t - \tau_t} \right]}_{:=T_{3,2}} \\
&+ \E \left[ \left \langle U_{t - \tau_t} - U_\lambda, 
-\tilde{\nabla} R_\lambda \left(x_t,x^+_t,U_{t}\right) +\tilde{\nabla} R_\lambda \left(x_t,x^+_t,U_{t - \tau_t}\right) \right.\right.\\
&\left. \left. \quad \quad -\nabla R_\lambda \left(U_{t - \tau_t}\right) + \nabla R_\lambda \left(U_{t}\right) \right \rangle_\HS| \Fcal_{t - \tau_t} \right]
\end{aligned}
\end{align}
Call the last term $T_{3,3}$. We next bound $T_{3,(1,2,3)}$ separately. In $T_{3,1}$, we apply Lemma \ref{lemma.CME.FT.fdiff} to bound $\vnorm{U_{t} - U_{t - \tau_t}}_\HS$ and Lemma \ref{lemma.CME.gradients} to bound the norm of gradients. Specifically, the Cauchy-Schwartz inequality gives
\begin{align}
\begin{aligned}
T_{3,1} 
\leq& \E \left[ \vnorm{U_{t} - U_{t - \tau_t}}_\HS \vnorm{-\tilde{\nabla} R_\lambda\left(x_t,x^{+}_t,U_{t}\right) + \nabla R_\lambda\left(U_{t}\right)}_\HS | \Fcal_{t - \tau_t} \right] \\
\stackrel{(a)}{\leq}& \E \left[ 4B \eta_{t-\tau_t,t-1} \left(\vnorm{U_{t}}_\HS+1\right) 
\vnorm{-\tilde{\nabla} R_\lambda\left(x_t,x^{+}_t,U_{t}\right) + \nabla R_\lambda\left(U_{t}\right)}_\HS | \Fcal_{t - \tau_t} \right] \\
\stackrel{(b)}{\leq}& \E \left[ 4B \eta_{t-\tau_t,t-1} \left(\vnorm{U_{t}}_\HS+1\right) \right. \\
&\left. \times \left( \vnorm{-\tilde{\nabla} R_\lambda\left(x_t,x^{+}_t;U_{t}\right)}_\HS + \vnorm{-\nabla R_\lambda\left(U_{t}\right)}_\HS \right)  | \Fcal_{t - \tau_t} \right]\\
\stackrel{(c)}{\leq}& 8 B^2 \eta_{t-\tau_t,t-1} \E \left[\left(\vnorm{U_{t}}_\HS+1\right)^2 | \Fcal_{t - \tau_t} \right] \\
\leq & 8 B^2 \eta_{t-\tau_t,t-1} \E \left[\left(\vnorm{U_{t} - U_\lambda}_\HS+ \vnorm{U_\lambda}_\HS+1\right)^2 | \Fcal_{t - \tau_t} \right] \\
\leq & 16 B^2 \eta_{t-\tau_t,t-1} \left(\E \left[\vnorm{U_{t} - U_\lambda}_\HS^2| \Fcal_{t - \tau_t} \right]+ \Xi_\lambda^2\right)..
\end{aligned}
\label{eq.T31}
\end{align}
To obtain (a), we use Lemma \ref{lemma.CME.FT.fdiff} to get
$\vnorm{U_{t} - U_{t - \tau_t}}_\HS \leq  4B \eta_{t-\tau_t,t-1} \left(\vnorm{U_{t}}_\HS+1\right)$.
Step (b) holds due to triangle inequality and step (c) follows from Lemma \ref{lemma.CME.gradients}\ref{lemma.CME.gradients.affine}.

In order to bound $T_{3,2}$, Cauchy-Schwatz inequality gives
\begin{align}
\begin{aligned}
T_{3,2} 
\leq & \vnorm{U_{t - \tau_t} - U_\lambda}_\HS \vnorm{\E \left[-\tilde{\nabla} R_\lambda \left(x_t,x^+_t,U_{t - \tau_t}\right)| \Fcal_{t - \tau_t} \right] +\nabla R_\lambda\left(U_{t - \tau_t}\right)}_\HS \\
\leq & 2 B_\kappa \eta_t \vnorm{U_{t - \tau_t} - U_\lambda}_\HS  \left(\vnorm{U_{t - \tau_t}}_\HS+1 \right),
\end{aligned}
\label{eq.T_3,2}
\end{align}
where we apply Lemma \ref{lemma.CME.mixing} to bound the bias of operator-valued stochastic gradients.
We next attempt to obtain a bound of $\vnorm{U_{t - \tau_t} - U_\lambda}_\HS$ in \eqref{eq.T_3,2} in terms of $\vnorm{U_{t} - U_{\lambda}}_\HS$ as 
\begin{align}
\begin{aligned}
\vnorm{U_{t - \tau_t} - U_\lambda}_\HS
\stackrel{(a)}{\leq}& \vnorm{U_{t} - U_{t - \tau_t} }_\HS 
+ \vnorm{U_{t} - U_\lambda}_\HS \\
\stackrel{(b)}{\leq}&  4 B_\kappa \eta_{t - \tau_t,t-1} \left( \vnorm{U_t}_\HS + 1\right)
+ \vnorm{U_{t} - U_\lambda}_\HS \\
\stackrel{(c)}{\leq}&  \vnorm{U_t}_\HS + 1
+ \vnorm{U_{t} - U_\lambda}_\HS \\
\stackrel{(d)}{\leq}& \vnorm{U_\lambda}_\HS+ \vnorm{U_t - U_\lambda}_\HS + 1
+ \vnorm{U_{t} - U_\lambda}_\HS \\
=& 2\vnorm{U_t - U_\lambda}_\HS +  \vnorm{U_\lambda}_\HS+ 1,
\end{aligned}
\label{eq.153}
\end{align}
where (a) follows from triangle inequality, (b) holds due to Lemma \ref{lemma.CME.FT.fdiff}, (c) follows from assumption $\eta_{t - \tau_t,t-1}\leq 1/4B$, and (d) holds since $\vnorm{U_t}_\HS= \vnorm{U_{t} - U_\lambda+ U_\lambda}_\HS \leq \vnorm{U_t - U_\lambda}_\HS + \vnorm{U_\lambda}_\HS$.

Likewise, we can bound $\vnorm{U_{t - \tau_t}}_\HS+1$ in terms of $\vnorm{U_t - U_\lambda}$ as 
\begin{align}
\begin{aligned}
\vnorm{U_{t - \tau_t}}_\HS+1
\leq&\vnorm{U_{t - \tau_t} - U_t}_\HS
+ \vnorm{U_{t} - U_\lambda}_\HS
+ \vnorm{U_\lambda}_\HS +1,\\
\stackrel{(a)}{\leq}& \vnorm{U_t}_\HS + 1
+ \vnorm{U_{t} - U_\lambda}_\HS
+ \vnorm{U_\lambda}_\HS +1 \\
\leq & \left( \vnorm{U_t-U_\lambda}_\HS + \vnorm{U_\lambda}_\HS + 1\right)
+ \vnorm{U_{t} - U_\lambda}_\HS
+ \vnorm{U_\lambda}_\HS +1 \\
=& 2 \left(\vnorm{U_t - U_\lambda}_\HS +  \vnorm{U_\lambda}_\HS+ 1\right),
\end{aligned}
\label{eq.155}
\end{align}
where (a) follows from Lemma \ref{lemma.CME.FT.fdiff} and the assumption that $\eta_{t - \tau_t,t-1}\leq 1/4B$.
Notice that $U_{t - \tau_t}$ is $\Fcal_{t - \tau_t}$-adapted. 
Substituting \eqref{eq.153} and \eqref{eq.155} into \eqref{eq.T_3,2} yields that
\begin{align}
\begin{aligned}
T_{3,2} 
\leq &2 B_\kappa \eta_t \E\left[ 4\left(\vnorm{U_t - U_\lambda}_\HS +  \vnorm{U_\lambda}_\HS+ 1\right)^2 | \Fcal_{t - \tau_t}\right] \\
\leq & 16 B_\kappa \eta_t 
\left(\E \left[\left( \vnorm{U_{t} - U_\lambda}_\HS^2 \right)| \Fcal_{t - \tau_t} \right]
+\Xi_\lambda^2 ]\right).
\end{aligned}
\label{eq.T32}
\end{align}

We next provide an upper bound for $T_{3,3}$. Analogous reasoning as before, we leverage Lemma \ref{lemma.CME.gradients} to obtain
\begin{align}
\begin{aligned}
T_{3,3}
\leq&\E \left[ \vnorm{U_{t - \tau_t} - U_\lambda}_\HS 
\left(\vnorm{-\tilde{\nabla} R_\lambda\left(x_t,x^{+}_t,U_{t}\right)+\tilde{\nabla} R_\lambda\left(x_t,x^{+}_t,U_{t - \tau_t}\right)}_\HS \right.\right.\\
&\left.\left.+ \vnorm{-\nabla R_\lambda\left(U_{t - \tau_t}\right) + \nabla R_\lambda\left(U_{t}\right)}_\HS \right)| \Fcal_{t - \tau_t} \right] \\
\leq&2 B_\kappa\E \left[ \vnorm{U_{t - \tau_t} - U_\lambda}_\HS \vnorm{U_{t} - U_{t- \tau_{t}}}_\HS| \Fcal_{t - \tau_t} \right] \\
\stackrel{(a)}{\leq}&8 B_\kappa B \eta_{t - \tau_t,t-1} \E \left[ \vnorm{U_{t - \tau_t} - U_\lambda}_\HS 
\left(\vnorm{U_{t}}_\HS +1 \right)| \Fcal_{t - \tau_t} \right]\\
\leq& 8 B^2 \eta_{t - \tau_t,t-1}\\
&\times \E \left[ \left( 2\vnorm{U_{t} - U_\lambda}_\HS + \vnorm{U_\lambda}_\HS+1\right) \times \left(\vnorm{U_{t}- U_\lambda}_\HS+\vnorm{U_\lambda}_\HS +1 \right)| \Fcal_{t - \tau_t} \right] \\
\leq & 16 B^2 \eta_{t - \tau_t,t-1} \E \left[\left(\vnorm{U_{t}- U_\lambda}_\HS+\vnorm{U_\lambda}_\HS +1 \right)^2| \Fcal_{t - \tau_t} \right] \\
\leq& 32 B^2 \eta_{t - \tau_t,t-1} \left( \E \left[\vnorm{U_{t}- U_\lambda}_\HS^2| \Fcal_{t - \tau_t} \right] + \left(\vnorm{U_\lambda}_\HS +1 \right)^2\right),
\end{aligned}
\label{eq.T33}
\end{align}
where we apply Lemma \ref{lemma.CME.FT.fdiff} to bound $\vnorm{U_{t} - U_{t - \tau_t}}_\HS$ in (a).
Combing the bounds on $T_{3,1}, T_{3,2}$ and $T_{3,3}$, we infer
\begin{align}
\begin{aligned}
T_3 \leq & 
\left(48 B^2 \eta_{t-\tau_t,t-1} + 16 B_\kappa \eta_t \right)
\left(
\E \left[\vnorm{U_{t} - U_\lambda}_\HS^2| \Fcal_{t - \tau_t} \right] + \Xi_\lambda^2 \right)
\end{aligned}
\end{align}

\emph{(Step 4)} Finally, Assumption \ref{assumption.3} (b) guarantees that there exists $B_\ve>0$ such that $\ve_t \leq B_\ve \eta_t \left(\vnorm{U_t}_\HS +1\right)$, for all $t \in \Nset$. In other words, $\ve_t$ scales affinely with respect to the current iterates. We can then apply affine scaling of gradients in Lemma \ref{lemma.CME.gradients} to bound $\vnorm{-\tilde{\nabla} R_\lambda\left(x_t,x^{+}_t,U_{t}\right)}_\HS$. Together with the bound on compression error $E_t$, we have
\begin{align}
\begin{aligned}
T_4
=&\E \left[ \vnorm{-\tilde{\nabla} R_\lambda\left(x_t,x^{+}_t,U_{t}\right)+\frac{E_t}{\eta_t}}_\HS^2 | \Fcal_{t - \tau_t} \right] \\
\leq &  \E \left[ \left(\vnorm{-\tilde{\nabla} R_\lambda\left(x_t,x^{+}_t,U_{t}\right)}_\HS+\ve_t/\eta_t \right)^2 | \Fcal_{t - \tau_t} \right] \\
\leq &  \E \left[ \left(B_\kappa\left(\vnorm{U_{t}}_\HS+1 \right)+B_\ve  \left(\vnorm{U_{t}}_\HS +1\right) \right)^2 | \Fcal_{t - \tau_t} \right] \\
=&  \E \left[ B^2\left(\vnorm{U_{t}}_\HS+1 \right)^2 | \Fcal_{t - \tau_t} \right] \\
\leq & \E \left[ B^2\left(\vnorm{U_{t} - U_\lambda}_\HS+\vnorm{U_\lambda}_\HS+1 \right)^2 | \Fcal_{t - \tau_t} \right] \\
\leq&  2 B^2 \left( \E \left[\vnorm{U_{t} - U_\lambda}_\HS^2 | \Fcal_{t - \tau_t} \right] + \Xi_\lambda^2\right),
\end{aligned}   
\end{align}
where the second line follows from \eqref{eq.CME.stop}, and we bound the term $\vnorm{U_{t}}_\HS$ via $\vnorm{U_{t} - U_{\lambda}}_\HS$ in the last line.

\emph{(Step 5)} 
Combing the bounds on $T_1$ to $T_4$, we have
\begin{align}
\begin{aligned}
&\E \left[ \vnorm{U_{t+1}-U_\lambda}_\HS^2 | \Fcal_{t - \tau_t} \right]
-\E \left[ \vnorm{U_{t}-U_\lambda}_\HS^2 | \Fcal_{t - \tau_t} \right]  \\
& \leq  \left(- 2\eta_t \lambda  + \left(98 B^2 + 32 B\right) \eta_t \eta_{t-\tau_t,t-1} \right)
\E \left[ \vnorm{U_{t} - U_\lambda}_\HS^2 | \Fcal_{t - \tau_t} \right]\\
& \quad + \left(98 B^2  + 32 B \right) \eta_t \eta_{t-\tau_t,t-1} \Xi_\lambda^2 
+ 4 \ve_t B_\infty/\lambda \\
&= \left(- 2\eta_t \lambda  + \check{B} \eta_t \eta_{t-\tau_t,t-1} \right)
\E \left[ \vnorm{U_{t} - U_\lambda}_\HS^2 | \Fcal_{t - \tau_t} \right]\\
& \quad + \check{B} \eta_t \eta_{t-\tau_t,t-1} \Xi_\lambda^2 
+ 4 \ve_t B_\infty/\lambda,
\end{aligned}
\end{align}
since $B$ dominates $B_\kappa$ and $\eta_t \leq \eta_{t-1} \leq \eta_{t-\tau_t, t-1}$.
From the above result, the second part of the lemma follows from elementary algebra; the steps are omitted.

%% file: proof/thm10.tex
\subsection{Proof of Theorem \ref{thm.Koopman.FT.general}}
\label{thm.CME.general.pf}

For $t \geq \tau_*$, we have
\begin{align}
\begin{aligned}
\E\left[ \vnorm{U_{t}-U_\lambda}_\HS^2 \right]
\leq & \left(1 -\lambda \eta_{t-1} \right)\E\left[ \vnorm{U_{t-1}-U_\lambda}_\HS^2 \right] 
+ \Theta_1\left(t-1,\bcmp,\lambda \right) \\
\leq & \E\left[ \vnorm{U_{\tau_*}-U_\lambda}_\HS^2 \right] 
\left(\Pi_{j=\tau_*}^{t-1} \left(1 -\lambda \eta_j \right)\right)\\
&+  \sum_{i=\tau_*}^{t-1}
\Theta_1\left(i,\bcmp,\lambda\right)
\left(\Pi_{j=i+1}^{t-1} \left(1 -\lambda \eta_j \right) \right).
\end{aligned}
\label{eq.thm10}
\end{align}

In addition, Lemma \ref{lemma.CME.bound} gives
$\E\left[ \vnorm{U_{\tau_t}-U_\lambda}_\HS^2 \right]
\leq \E\left[ 2 \vnorm{U_{t-\tau_t}}_\HS^2+ 2 \vnorm{U_\lambda}_\HS^2 \right]
\leq  4B_\infty^2/\lambda^2$.
Plugging into \eqref{eq.thm10}, we have
\begin{align}
{\small
\begin{aligned}
\E\left[ \vnorm{U_{t}-U_\lambda}_\HS^2 \right]
\leq&
4\frac{B_\infty^2}{\lambda^2} \Psi\left(t-1,\tau_*\right)
+  \sum_{i=\tau_*}^{t-1} \Psi(t-1,i+1) \Theta_1\left(i,\bcmp,\lambda\right).
\end{aligned}}
\label{eq.U.general}
\end{align}
Substituting this into Lemma \ref{lemma.Kt} proves the claim.

%% file: proof/corollary11.tex
\subsection{Proof of Corollary \ref{corollary.FT.const}}
\label{thm.FT.const.pf}
Since the stepsizes are constant, i.e., $\eta_t = \eta$ for all $t \in \Tset$, we use the notation $\Theta_1'\left(\bcmp,\lambda\right):=\Theta_1\left(t,\bcmp,\lambda\right)$. 
Notice that a direct consequence of \eqref{eq.U.general} is that when $\eta_t = \eta$ and $\ve_t = \ve$, we have that for all $t \geq \tau_\eta$,
\begin{align}
\begin{aligned}
\sum_{i=\tau_\eta }^{t-1} \Psi(t-1,i+1) \Theta_1'\left(\bcmp,\lambda\right) 
=&\sum_{i=\tau_\eta }^{t-1} 
\Pi_{j=i+1}^{t-1}
\left(1 -\lambda \eta \right)
\Theta_1'\left(\bcmp,\lambda\right) \\
=&\sum_{i=\tau_\eta}^{t-1} 
\left(1 -\lambda \eta \right)^{t-i-1}
\Theta_1'\left(\bcmp,\lambda\right) \\
=&\left(\sum_{k=0}^{t-\tau_\eta-1} 
\left(1 -\lambda \eta \right)^{k}\right)
\Theta_1'\left(\bcmp,\lambda\right)\\
\leq& \frac{1}{\lambda \eta} \Theta_1'\left(\bcmp,\lambda\right).
\end{aligned}
\end{align}
Therefore, from \eqref{eq.U.general}, we have
\begin{align}
\begin{aligned}
\E\left[\vnorm{U_{t}-U_\lambda}_\HS^2 \right]
 \leq &
 4\frac{B_\infty^2}{\lambda^2}  \left(1 -\lambda \eta \right)^{t-\tau_\eta}
+ \Theta_1'\left(\bcmp,\lambda\right)/ \left(\lambda \eta \right).\\
\end{aligned}
\label{coro.CME.FT.const}
\end{align}
Substituting the above relation into Lemma \ref{lemma.Kt}, we have
\begin{align}
\begin{aligned}
&\E \left[\vnorm{\left[K_t\right] - K }_{\HS(\Hcal,{[H]^\beta})}^2 \right]\\
\leq&
2 B_\infty^{1-\gamma}
\left(4\frac{B_\infty^2}{\lambda^2}  \left(1 -\lambda \eta \right)^{t-\tau_\eta}
+ \Theta_1'\left(\bcmp,\lambda\right)/ \left(\lambda \eta \right)\right)
+ 2 \Bsrc^2 \\
=& 
\frac{8 B_\infty^{3-\gamma}}{\lambda^2}  \left(1 -\lambda \eta \right)^{t-\tau_\eta} 
+2 B_\infty^{1-\gamma} \left(\check{B} \tau_\eta \Xi_\lambda^2 
+ 4 \bcmp \frac{B_\infty}{\lambda} \right) \eta 
+ 2 \Bsrc^2.
\end{aligned}
\end{align}
This completes the proof.

%% file: proof/corollary12.tex
\subsection{Proof of Corollary \ref{thm.FT.diminish}}
\label{thm.FT.diminish.pf}

First, notice that under Assumption \ref{assumption.mixing}, the mixing time satisfies $\tau(\delta) \leq B_\mix \left(\log(1/\delta)+1\right)$ for all $\delta>0$. In addition, by \eqref{eq.mixing0}, we have 
\begin{align}
\lim_{\delta\to0} \delta \ \tau(\delta) 
\leq \lim_{\delta\to0} \delta \ B_\mix \left(\log \frac{1}{\delta} +1 \right) 
= \ B_\mix \lim_{\delta\to0} \delta \left(\log \frac{1}{\delta} +1 \right) 
=0.
\end{align}
Setting $\delta = \eta_t$, we have
\begin{align}
\begin{aligned}
\eta_{t-\tau_t,t-1}
\leq \tau_t \eta_{t-\tau_{t}}
\leq& B_\mix \left(\log(1/\eta_t)+1\right) \frac{\eta}{(t-\tau_{t}+r)^a} \\
\leq& B_\mix \left(\log(1/\eta_t)+1\right) \frac{\eta}{(t-B_\mix \left(\log(1/\eta_t)+1\right) +r)^a}.
\end{aligned}
\end{align}
We next choose $r$ such that $\eta_{t-\tau_t,t-1} \leq \lambda/\check{B}$ for $t \geq \tau_t$. To this end, notice that
\begin{align}
\begin{aligned}
\frac{\eta_{t-\tau_t,t-1}}{\eta_t B_\mix \left(\log(1/\eta_t)+1\right) }
\leq& \frac{\eta}{\eta_t (t-B_\mix \left(\log(1/\eta_t)+1\right) +r)^a} \\
=& \frac{\eta (t+r)^a}{\eta (t-
B_\mix \left(a\log(t+r)+\log(1/\eta) +1 \right) +r)^a}\\
=&\left(\frac{t+r}{t-
B_\mix \left(a\log(t+r)+\log(1/\eta) +1\right)+r}\right)^a.
\end{aligned}
\end{align}
Since $a \in (0,1)$, taking $t + r \to \infty$ on both side gives
\begin{align}
\begin{aligned}
&\lim_{t + r \to \infty} \frac{\eta_{t-\tau_t,t-1}}{\eta_t B_\mix \left(\log(1/\eta_t)+1\right) }\\
=&\lim_{t + r \to \infty} \left(\frac{t+r}{t-
B_\mix \left(a\log(t+r)+\log(1/\eta) +1\right)+r}\right)^a\\
=&1
\end{aligned}
\end{align}
Hence, there exists  $r_1>0$ such that fix an $\acute{\epsilon}>0$, we have for all $t \geq 0$,
\begin{align}
\eta_{t-\tau_t,t-1} \leq \left(1 + \acute{\epsilon} \right)\eta_t B_\mix \left(\log(1/\eta_t)+1\right).
\end{align}
This also suggests that
\begin{align}
\begin{aligned}
\Theta_1\left(t,\bcmp,\lambda\right)
=&\check{B} \eta_t \eta_{t-\tau_t,t-1} \Xi_\lambda^2 
+ \frac{4 \bcmp \eta_t^2 B_\infty}{\lambda} \\
\leq& \check{B} \left(1 + \acute{\epsilon} \right) B_\mix \eta_t^2 \left(\log(\frac{1}{\eta_t}) +1 \right)
\Xi_\lambda^2 
+ \frac{4 \bcmp \eta_t^2 B_\infty}{\lambda} . 
\end{aligned}
\end{align}

In addition, the stepsize sequence satisfies
\begin{align}
\lim_{t + r \to \infty} \eta_{t-\tau_{t}} 
= \lim_{t + r \to \infty} \frac{\eta}{(t-\tau_{t}+r)^a} 
=0,\quad a \in (0,1).
\end{align}
Therefore, by the fact that $\lim_{x \to 0} x \left(1+\log \frac{1}{x}\right) = 0$, we have
\begin{align}
\lim_{t + r \to \infty}\tau_t \eta_{t-\tau_{t}}
\leq B_\mix \lim_{t + r \to \infty} \left(\log\frac{1}{\eta_t}+1\right) \eta_{t-\tau_{t}}
=0.
\end{align}
That is, there exists $r_2>0$ such that 
$\eta_{t-\tau_t,t-1} \leq \lambda/\check{B}$.
By setting $r = \max \left(r_1,r_2\right)$, we can guarantee that the condition that $\eta_{t-\tau_t,t-1}\leq \lambda/\check{B}$ in Theorem \ref{thm.Koopman.FT.general} holds.

We are now ready to prove Corollary \ref{thm.FT.diminish}. By Theorem \ref{thm.Koopman.FT.general}, for all $t \geq \tau_*$, we have 
\begin{align}
\begin{aligned}
&\E\left[\vnorm{U_{t}-U_\lambda}_\HS^2 \right]\\
\leq &
4\frac{B_\infty^2}{\lambda^2} \Psi\left(t-1,\tau_*\right)
+ \sum_{i=\tau_*}^{t-1} \Psi(t-1,i+1) \Theta_1\left(i,\bcmp,\lambda\right)\\
\leq & 4\frac{B_\infty^2}{\lambda^2} \Psi\left(t-1,\tau_*\right)\\
&+ \left(\check{B} \left(1 + \acute{\epsilon} \right) B_\mix \left(\log(\frac{1}{\eta_t}) +1 \right)
\Xi_\lambda^2 
+\frac{4 \bcmp B_\infty}{\lambda} \right) 
\sum_{i=\tau_*}^{t-1} \Psi(t-1,i+1) \eta_i^2\\
\leq & 4\frac{B_\infty^2}{\lambda^2} \Psi\left(t-1,\tau_*\right)\\
&+ \underbrace{2\left(\check{B} B_\mix \left(\log\frac{t+r}{\eta} +1 \right)
\Xi_\lambda^2 
+ \frac{4 \bcmp B_\infty}{\lambda} \right)}_{:=\Theta_4} 
\sum_{i=\tau_*}^{t-1} \Psi(t-1,i+1) \eta_i^2,
\end{aligned}
\end{align}
where in the last line, we set $\acute{\epsilon}=1$ for simplicity and plugging in $\eta_t = \frac{\eta}{\left(t+r\right)^a}$ to obtain $\log(1/\eta_t)  = \log(\frac{(t+r)^a}{\eta}) \leq \log(\frac{t+r}{\eta})$ for $a \in (0,1)$.

To bound
$\Psi\left(t-1,\tau_*\right) 
=\Pi_{i=\tau_*}^{t-1} \left(1 -\frac{\lambda \eta}{(i+r)^a} \right)$,
using $1+x \leq e^x$ for $x \in \Rset$, we have
\begin{align}
\begin{aligned}
\Psi\left(t-1,\tau_*\right) 
\leq& \exp\left(-\lambda \eta \int_{\tau_*}^{t} \frac{1}{(i+r)^a} \dd x\right) \\
\leq&  \exp\left(-\frac{\lambda \eta}{1-a} \left(\left(t+r\right)^{1-a} -\left(\tau_*+r\right)^{1-a}\right) \right).
\end{aligned}
\label{eq.185}
\end{align}

To bound $\sum_{i=\tau_*}^{t-1} \Psi(t-1,i+1) \eta_i^2$, consider the recursions $z_{t+1} = \left(1-\lambda \eta_t\right) z_{t} + \eta_t^2$, for $t \geq \tau_*$ with $z_{\tau} = 0$.
We then have $z_{t} =\sum_{i=\tau_*}^{t-1} \Psi(t-1,i+1) \eta_i^2 $. We next show $z_{t} \leq \frac{2}{\lambda} \eta_t$ by induction.
At $\tau_*$, $z_{\tau_*}=0\leq \frac{2}{\lambda} \eta_t$, thus the base case trivially hold. Suppose the relation hold for $k= \tau_* ,\tau_*+1,\ldots t$, for $t \geq \tau_*$, then at time $k+1$, we have
\begin{align}
\begin{aligned}
\frac{2}{\lambda} \eta_{k+1} -  z_{k+1}
=& \frac{2}{\lambda} \eta_{k+1} - \left(1-\lambda \eta_k\right) z_{k} - \eta_k^2 
\geq \frac{2}{\lambda} \eta_{k+1} - \left(1-\lambda \eta_k\right) \frac{2}{\lambda} \eta_k -\eta_k^2 \\
=& \frac{2}{\lambda} \left( \eta_{k+1} - \eta_k\right) + \eta_k^2. 
\end{aligned}
\end{align}
Hence, we have
\begin{align}
\begin{aligned}
\frac{2}{\lambda} \eta_{k+1} -  z_{k+1}
=& \frac{\eta^2}{(k+r)^{2a}} -  \frac{2}{\lambda}\left(\frac{\eta}{(k+r)^{a}} - \frac{\eta}{(k+1+r)^a}\right)\\
=& \frac{1}{(k+r)^{2a}}\left(\eta^2 -  \frac{2\eta}{\lambda}(k+r)^{a} \left(1 - \left(\frac{k+r}{k+1+r}\right)^a\right)\right) \\
\stackrel{(a)}\geq&\frac{1}{(k+r)^{2a}}\left(\eta^2 -  \frac{2\eta}{\lambda}(k+r)^{a} \left(\frac{a}{k+r}\right)\right) \\
=&\frac{\eta}{(k+r)^{2a}}\left(\eta -  \frac{2a}{\lambda}\frac{1}{(k+r)^{1-a}} \right)\\
\stackrel{(b)}\geq& 0,
\end{aligned}
\end{align}
where (a) follows from the relation $\left(\frac{x}{1+x}\right)^a \geq 1 - \frac{a}{x}$ for $x>0$ and (b) holds since $k +r \geq \tau_* \geq (\frac{2a}{\lambda \eta})^{\frac{1}{1-a}}$ for $a \in (0,1)$.
Therefore, $z_{k} \leq \frac{2}{\lambda} \eta_k$ for $k \geq \tau_*$. Taken together, we infer
\begin{align}
\begin{aligned}
\E\left[\vnorm{U_{t}-U_\lambda}_\HS^2 \right]
\leq &
4\frac{B_\infty^2}{\lambda^2} \exp\left(-\frac{\lambda \eta}{1-a} \left(\left(t+r\right)^{1-a} -\left(\tau_*+r\right)^{1-a}\right) \right)\\
&+ \Theta_4 \frac{2\eta}{\lambda} \frac{1}{(t+r)^a}.
\end{aligned}
\end{align}

Substituting the above result into Lemma \ref{lemma.Kt} completes the proof.